\documentclass[journal]{IEEEtran}
\IEEEoverridecommandlockouts 
\usepackage{makeidx}  

\usepackage[utf8]{inputenc}
\usepackage[T1]{fontenc}

\usepackage{graphicx}
  \graphicspath{ {figs/} }
  \DeclareGraphicsExtensions{.pdf,.jpeg,.png,.eps,.jpg}

\usepackage[bookmarksnumbered,unicode]{hyperref} 
	
\usepackage{nicefrac} 
\usepackage{amsmath}
	\usepackage{amsfonts}
	\usepackage{amssymb}
  \usepackage{amsthm} 
\usepackage{mathtools} 
\usepackage{bm} 
\usepackage{bbm} 
\usepackage[dvipsnames,prologue,svgnames]{xcolor}
\usepackage[inline,shortlabels]{enumitem} 
\usepackage[group-separator={,}]{siunitx} 
\usepackage{array}
\usepackage{booktabs} 
\usepackage{tikz} 
\usepackage{tikz-3dplot} 
	\usetikzlibrary{positioning,shapes,fit,3d,calc,arrows.meta,perspective}
  \usetikzlibrary{arrows.meta,bending,positioning,3d,calc,perspective,automata}
  \usetikzlibrary{calc}
\usetikzlibrary{svg.path}
\usetikzlibrary{captl} 
\usetikzlibrary{extshapes}
\usetikzlibrary{automata,positioning,shapes,fit}
\usetikzlibrary{decorations.markings}
\usetikzlibrary{backgrounds}
	\pgfdeclarelayer{BoxesLayer}      
	\pgfdeclarelayer{background}     
	\pgfdeclarelayer{ArrowsLayer}    
	\pgfdeclarelayer{foreground}     
	\pgfsetlayers{BoxesLayer,background,ArrowsLayer,main,foreground}
\usepackage{inconsolata} 
\usepackage{pgfplots} 
	\usepgfplotslibrary{units}  
	\usepgfplotslibrary{fillbetween} 
	\usetikzlibrary{patterns} 
	\usetikzlibrary{intersections}
	\pgfplotsset{compat=newest} 
	\pgfplotsset{plot coordinates/math parser=false} 
	\newlength\figureheight 
	\newlength\figurewidth 
\usepackage{tikz-cd}
\usepackage[ruled,vlined,linesnumbered]{pkgs/algorithm2e-far} 
  
\usepackage[caption=false,font=footnotesize]{subfig}
\usepackage{textcomp} 
\usepackage{stfloats} 
\usepackage{url} 
\usepackage{verbatim} 
\usepackage{cite} 
\usepackage[ligature, inference]{semantic} 
\usepackage{xargs} 
\usepackage{marginnote} 
\usepackage{comment}
\usepackage{orcidlink}
\usepackage{bm} 
\usepackage{pgf}
\usepackage{xfp}
\usepackage{xparse}
\usepackage{dsfont} 
\usepackage{nicefrac} 

\usepackage{mparhack}
\usepackage[colorinlistoftodos,prependcaption,textsize=scriptsize]{todonotes} 
\ifdefined\EnableTodoNotes
\else
  \presetkeys{todonotes}{disable}{}
\fi

\newcommandx{\ME}[2][1=]{\todo[linecolor=blue,backgroundcolor=black!10,bordercolor=blue,#1]{ME: #2}}
\newcommandx{\YS}[2][1=]{\todo[linecolor=blue,backgroundcolor=black!10,bordercolor=blue,textcolor=red,#1]{YS: #2}}
\newcommandx{\UL}[2][1=]{\todo[linecolor=blue,backgroundcolor=black!10,bordercolor=blue,#1]{UL: #2}}
\newcommandx{\XX}[2][1=]{\todo[linecolor=blue,backgroundcolor=black!10,bordercolor=blue,#1]{#2}}
\newcommandx{\MEx}[2][1=]{\todo[linecolor=blue,backgroundcolor=black!10,bordercolor=blue,#1]{%
		\setlength{\baselineskip}{1.0\baselineskip}%
		\sffamily\selectfont\scriptsize ME: #2\par}}
\newcommandx{\MEt}[2][1=]{\todo[linecolor=blue,backgroundcolor=red!10,bordercolor=red,#1]{%
		\sffamily\selectfont\scriptsize TODO: #2\par}}
\newcommandx{\MEd}[2][1=]{\todo[linecolor=gray,backgroundcolor=gray!10,bordercolor=gray,#1]{%
		\sffamily\selectfont\scriptsize\textcolor{gray}{DONE: #2}\par}}
\newcommandx{\MEq}[2][1=]{\todo[linecolor=blue,backgroundcolor=yellow!40,bordercolor=blue,#1]{%
		\sffamily\selectfont\scriptsize\textcolor{black}{Q: #2}\par}}
\newcommandx{\MEr}[2][1=]{\todo[linecolor=blue,backgroundcolor=yellow!40,bordercolor=blue,#1]{%
		\sffamily\selectfont\scriptsize\bf\textcolor{red}{#2 READY FOR REVIEW}\par}}

\makeatletter
\define@key{todonotes}{done}[]{%
    \setkeys{todonotes}{color=green!50!black,backgroundcolor=green!10!white,bordercolor=green!50!black,textcolor=black!70,%
		caption={\checkmark}}%
}
\makeatother

\usepackage{xcolor}

\newif\ifshowrevisions
\showrevisionsfalse

\ifshowrevisions
  \definecolor{RevColor}{rgb}{0,0,0.7}
  \definecolor{ChkColor}{rgb}{0.7,0.3,0.7}
  \newcommand{\Revised}{\color{RevColor}}              
  \newenvironment{revised}{\color{RevColor}}{}          
  \newenvironment{checkit}{\color{ChkColor}}{}          
  
\else
  \newcommand{\Revised}{}

\fi

\makeatletter
  \def\underbar#1{\underline{\sbox\tw@{$#1$}\dp\tw@=-0.1pt\box\tw@}}
  \def\overbar#1{\overline{\sbox\tw@{$#1$}\ht\tw@=0.95\ht\tw@\dp\tw@\z@\box\tw@}}
\makeatother

\newcommand{\B}{\mathbb{B}} 
\newcommand{\Z}{\mathbb{Z}} 
\newcommand{\R}{\mathbb{R}} 
\newcommand{\N}{\mathbb{N}} 
\newcommand{\Powerset}{\mathcal{P}}

\newcommand{\dom}[1]{\mathrm{dom}(#1)} 
\newcommand{\Reachset}{\mathcal{R}} 
\newcommand{\Vector}[1]{\overrightarrow{#1}} 
\newcommand{\True}{\top} 
\newcommand{\False}{\bot} 

\DeclarePairedDelimiter\norm{\lVert}{\rVert}%

\DeclareMathOperator*{\sgn}{sgn} 

\DeclareSIUnit{\each}{\relax}

\newcommand{\Pose}{\xi}
\newcommand{\PoseEst}{\hat{\scalebox{.9}{$\Pose$}}}
\newcommand{\PoseMin}{\underbar{\raisebox{.8pt}{\scalebox{.9}{$\Pose$}}}}
\newcommand{\PoseMax}{\overline{\raisebox{-.0pt}{\scalebox{.9}{$\Pose$}}}}
\newcommand{\PoseSpace}{\Xi}
\newcommand{\PoseSpaceEst}{\hat{\Xi}}
\newcommand{\PoseSpaceGrid}{\PoseSpace_d}
\newcommand{\PosX}{x}
\newcommand{\PosY}{y}
\newcommand{\PosZ}{z}
\newcommand{\Roll}{\phi}
\newcommand{\Pitch}{\theta}
\newcommand{\Yaw}{\psi}

\newcommand{\Coord}[1]{\mathtt{#1}}
\newcommand{\TCF}{\ensuremath{\Coord{TCF}}}
\newcommand{\CCF}{\ensuremath{\Coord{CCF}}}
\newcommand{\PCF}{\ensuremath{\Coord{PCF}}}

\newcommand{\NN}{\mathcal{{N\!\!N}}} 

\newcommand{\Enc}{\triangleright}
\newcommand{\Encoder}{\NN_{\!\Enc}}
\newcommand{\Dec}{\triangleleft}
\newcommand{\Decoder}{\NN_{\!\Dec}}
\newcommand{\Det}{\exists} 
\newcommand{\Detector}{\NN_{\!\Det}}
\newcommand{\SpatialFilter}{\mathit{H}} 
\newcommand{\Mask}{\mathit{M}}

\newcommand{\Layer}[1]{\ensuremath{\mathsf{L#1}}}
\newcommand{\Node}[2]{a^{(#1)}_{#2}}
\newcommand{\PartitionCt}{K}

\newcommand{\Lipschitz}{\mathit{L}}
\newcommand{\BoundInput}{\eta}
\newcommand{\NoiseBound}{\overline{n}} 

\newcommand{\TrainingError}{\epsilon}
\newcommand{\ErrorBound}{\epsilon_{\mathit{bound}}}

\newcommand{\DetectionThreshold}{\tau}
\newcommand{\Error}{e} 

\newcommand{\Camera}{\mathcal{C}}

\newcommand{\Target}{\mathcal{T}}

\newcommand{\TargetSize}{\mathit{N}}
\newcommand{\Clutter}{\overline{\Target}}
\newcommand{\Vertex}{\mathrm{v}}

\newcommand{\Polygon}{\mathrm{V}}
\newcommand{\Projection}{\lambda}
\newcommand{\ProjectionSet}{\Lambda}
\newcommand{\Pixel}{\mathit{q}}

\newcommand{\PixelX}{\Pixel_{x}}
\newcommand{\PixelY}{\Pixel_{y}}
\newcommand{\Cross}{C}

\newcommand{\Composition}{\mathcal{E}}

\newcommand{\Image}{\mathit{I}}

\newcommand{\ImageSpace}{\mathcal{I}}

\newcommand{\ImageSet}{\mathcal{I}}

\newcommand{\ImageIdeal}{\Image_{\Target}}
\newcommand{\ImageClutter}{\Image_{\Clutter}}

\newcommand{\AnalogFn}{f_\textit{analog}}
\newcommand{\DigitalFn}{f_\textit{digital}}
\newcommand{\ConvexHull}{\mathrm{Conv}}

\newcommand{\GridStep}{\mathit{h}}

\newcommand{\Dataset}{\mathcal{D}}

\newcommand{\Radius}{r}

\newcommand{\ImWidthPxl}{\mathit{w}}
\newcommand{\ImHeightPxl}{\mathit{h}}
\newcommand{\ImSizePxl}{\ImHeightPxl \times \ImWidthPxl}

\newcommand{\FocalLength}{f}
\newcommand{\Intrinsic}{\mathbf{K}}

\newcommand{\Sin}[1]{\mathop {\mathit{s}}\nolimits_{#1}}
\newcommand{\Cos}[1]{\mathop {\mathit{c}}\nolimits_{#1}}

\newcommand{\RotX}[1]{\mathbf{R}_{x}({#1})}
\newcommand{\RotY}[1]{\mathbf{R}_{y}({#1})}
\newcommand{\RotZ}[1]{\mathbf{R}_{z}({#1})}
\newcommand{\Rot}[1]{\mathbf{R}({#1})}

\newcommand{\Tran}[1]{\mathbf{T}({#1})}

\newcommand{\PolySize}{\mathit{n}}

\newcommand{\GridSampling}{\texttt{FWDReach-GridSampling}}

\newcommand{\DetectionSet}{\mathcal{H}}

\renewcommand{\r}{\textcolor{red}}

	\newcommand{\source}[1]{\relax}

\newcommand{\eat}[1]{\relax}

\newcommand{\tabref}		[1]{Table~\ref{#1}}
\newcommand{\figref}		[1]{Fig.~\ref{#1}}
\newcommand{\thmref}		[1]{Theorem~\ref{#1}}
\newcommand{\lemref}		[1]{Lemma~\ref{#1}}
\newcommand{\corref}		[1]{Corollary~\ref{#1}}
\newcommand{\exref}			[1]{Example~\ref{#1}}
\newcommand{\algref}		[1]{Algorithm~\ref{#1}}
\newcommand{\secref}		[1]{Section~\ref{#1}}
\newcommand{\subsecref}		[1]{Section~\ref{#1}}
\newcommand{\subsubsecref}	[1]{Section~\ref{#1}}
\newcommand{\defref}		[1]{Definition~\ref{#1}}
\newcommand{\remref}		[1]{Remark~\ref{#1}}

\newcommand{\probref}		[1]{Problem~\ref{#1}}
\newcommand{\asmref}		[1]{Assumption~\ref{#1}}

\newcommand{\ie}{i.e., }
\newcommand{\eg}{e.g., }
\newcommand{\suchthat}{\mbox{s.t.}}

\newtheorem{problem}{Problem}
\newtheorem{remark}{Remark}
\newtheorem{definition}{Definition}
\newtheorem{example}{Example}
\newtheorem{lemma}{Lemma}

\newtheorem{corollary}{Corollary}

\newcommand{\PlusOne}{\raisebox{.2\height}{\scalebox{.6}{+}}1}
\newcommand{\MinusOne}{\raisebox{.2\height}{\scalebox{.8}{-}}1}
\newcommand{\Neg}[1]{\raisebox{.2\height}{\scalebox{.8}{--}}#1}

\renewcommand{\emptyset}{\varnothing}
\makeatletter
\renewcommand*\env@matrix[1][*\c@MaxMatrixCols c]{%
  \hskip -\arraycolsep
  \let\@ifnextchar\new@ifnextchar
  \array{#1}}
\makeatother

\SetKw{Break}{break}
\SetKwComment{Comment}{$\triangleright$\ }{}

\SetSideCommentRight
\SetCommentSty{text}

\definecolor{azure(colorwheel)}{rgb}{0.0, 0.5, 1.0}

\definecolor{C0}{HTML}{1F77B4} 
\definecolor{C1}{HTML}{FF7F0E} 
\definecolor{C2}{HTML}{2CA02C} 
\definecolor{C3}{HTML}{D62728} 
\definecolor{C4}{HTML}{9467BD} 
\definecolor{C5}{HTML}{8C564B} 
\definecolor{C6}{HTML}{E377C2} 
\definecolor{C7}{HTML}{7F7F7F} 
\definecolor{C8}{HTML}{BCBD22} 
\definecolor{C9}{HTML}{17BECF} 

\colorlet{OptionA}{DarkBlue}
\colorlet{OptionB}{DarkRed}
\definecolor{TooLong}{HTML}{9467BD} 
\definecolor{Extra}{HTML}{9467BD} 

\newcount\ppnum
\newcommand\numx[1]{%
        \ppnum=#1\relax
        \ifnum\ppnum<0
                $-$%
                \ppnum=-\ppnum
        \fi
        \let\pptemp\empty
        \loop\ifnum\ppnum>999
                \count255=\ppnum
                \divide\ppnum by1000
                \count255=\numexpr \count255 - 1000*\ppnum \relax
                \edef\pptemp{,\ifnum\count255<100 0\ifnum\count255<10 0\fi\fi
                             \the\count255 \pptemp}%
        \repeat
        \the\ppnum
        \pptemp
}
\makeatletter
\DeclareRobustCommand{\rvdots}{%
  \vbox{
    \baselineskip4\p@\lineskiplimit\z@
    \kern-\p@
    \hbox{.}\hbox{.}\hbox{.}
  }}
\makeatother

\tikzset{pAction/.style={font={\fontsize{7}{0}\selectfont}, RoyalBlue, above, pos=.5}}
\tikzset{pGuard/.style={font={\fontsize{7}{0}\selectfont}, Red}}
\tikzset{pLabel/.style={align=center,anchor=north,font={\fontsize{7}{7}\selectfont}, Black!90}}
\tikzset{pDistr/.style={font={\fontsize{7}{0}\selectfont}, Black!90}}
\tikzset{pAssign/.style={font={\fontsize{7}{0}\selectfont}, ForestGreen}}
\tikzset{pArc/.style={-,shorten <=0pt,shorten >=0pt}}
\tikzset{cstate/.style={anchor=center,state,rRect,text width=,minimum size=6mm,inner xsep=2pt,inner ysep=2pt,minimum height=18pt,fill=gray!20,rectangle,rounded corners=1pt}}
\tikzset{vDots/.style={inner sep=0,minimum height=18pt}}
\tikzset{buswidth/.style={%
  line width=1.2,color=black!80,decoration={
  markings, mark= at position 1.0 with {\node[font=\footnotesize,xshift=-15] {$\bm{/}$};\node[below=3pt,xshift=-15] {\tiny #1};}
  }, postaction={decorate}}}
\tikzset{mux/.style={%
  anchor=center,rectangle,outer sep=0.0,draw,minimum height=24pt,minimum width=4pt, inner sep=0, fill=black!90, color=Black!90}}
\tikzset{input/.style={
  state,text width=,align=center,minimum size=16pt,inner sep=0,scale=0.8,font={\fontsize{7}{7}\selectfont}}}
\tikzset{neuron/.style={
  state,text width=,align=center,minimum size=16pt,inner sep=0,scale=0.8,font={\fontsize{7}{7}\selectfont}}}
\tikzset{bias/.style={
  state,text width=,align=center,minimum size=8pt,inner sep=0,scale=0.8,font={\fontsize{7}{7}\selectfont},fill=black!20}}
\tikzset{pBias/.style={
  color=black!70,densely dashed,-{Latex[length=4px,width=2px]},inner sep=0.5px,outer sep=0.0px,pos=0.5,align=center,scale=0.5,font={\fontsize{7}{7}\selectfont}}}
\tikzset{pArrow/.style={
  -{Latex[length=3px,width=2px]},inner sep=1.0px,outer sep=2.0px,pos=0.5,align=center,scale=0.8,font={\fontsize{7}{7}\selectfont}, shorten >=0.5pt}}
\tikzset{qArrow/.style={
  -{Latex[length=2px,width=2px]},inner sep=1.0px,outer sep=2.0px,pos=0.5,align=center,scale=0.8,font={\fontsize{7}{7}\selectfont}, shorten >=0pt}}
\tikzset{sArrow/.style={
  -,inner sep=0.0px,outer sep=0.0px,pos=0.5,align=center,scale=0.8,font={\fontsize{7}{7}\selectfont}, shorten >=0pt}}
\tikzset{pWeight/.style={
  fill=white,inner sep=0.5px,outer sep=0.0px,pos=0.5,align=center,anchor=center,font={\fontsize{7}{7}\selectfont}}}
\tikzset{constant/.style={anchor=center,state,text width=,minimum size=18pt,inner xsep=2pt,inner ysep=2pt}}

\tikzset{pOverlay/.style={text=yellow,font=\tiny\ttfamily, align=left, outer sep=4pt, text width=0.8\FigWidth}}

\theoremstyle{plain}
\newtheorem{theorem}{Theorem}

\newtheorem{assumption}{Assumption}

\setpremisesend{0.50em} 

\newlength{\FigWidth}

\newlength{\FigYGap}

\ifdefined\EnableTodoNotes
  \paperwidth=\dimexpr \paperwidth + 6cm\relax
  \oddsidemargin=\dimexpr\oddsidemargin + 3cm\relax
  \evensidemargin=\dimexpr\evensidemargin + 3cm\relax
  \marginparwidth=\dimexpr\marginparwidth + 3cm\relax
\fi

\ifdefined\EnableTodoNotes
	\makeatletter
	\renewcommand{\todo}[2][]{\leavevmode\tikzexternaldisable\@todo[#1]{#2}\tikzexternalenable}
	\makeatother
\fi

\SetAlgoSkip{4pt}

\SetAlFnt{\footnotesize}
\SetAlCapFnt{\footnotesize}
\SetAlCapNameFnt{\footnotesize}

\definecolor{background}{HTML}{EEEEEE}
\definecolor{textcolor}{HTML}{FFEEDD}

\makeatletter
\def\footnoterule{\kern-3\p@
  \hrule \@width 0.5\columnwidth \kern 2.6\p@} 
\makeatother
\begin{document}
\bstctlcite{IEEEexample:BSTcontrol}
\title{%
  Correct-by-Construction Vision-based Pose Estimation using Geometric Generative Models
}%
\author{
  \IEEEauthorblockN{%
    Ulices Santa Cruz\IEEEauthorrefmark{1},
    Mahmoud Elfar\IEEEauthorrefmark{1}, and
    Yasser Shoukry\IEEEauthorrefmark{1}%
  }%
  \thanks{The authors are with the Department of Electrical Engineering and Computer Science, University of California, Irvine, CA 92697, USA. Email: \{usantacr, elfarm, yshoukry\}@uci.edu}%
}%

\maketitle              
\begin{abstract}

We consider the problem of vision-based pose estimation for autonomous systems.
While deep neural networks have been successfully used for vision-based tasks,
they inherently lack provable guarantees on the correctness of their outputs,
which is crucial for safety-critical applications.
We present a framework for designing certifiable neural networks (NNs) for
perception-based pose estimation that integrates physics-driven modeling with learning-based estimation.
The proposed framework begins by leveraging the known geometry of 
planar objects commonly found in the environment, such as 
traffic signs and runway markings, referred to as target objects. 
At its core, it introduces a geometric generative model (GGM), 
a neural-network-like model whose parameters are derived from 
the image formation process of a target object observed by a camera. 
Once designed, the GGM can be used to train NN-based pose estimators with 
certified guarantees in terms of their estimation errors. 
We first demonstrate this framework in uncluttered environments, 
where the target object is the only object present in the camera's field of view. 
We extend this using ideas from NN reachability analysis to design certified 
object NNs that can detect the presence of the target object in cluttered environments. 
Subsequently, the framework consolidates the certified object detector with 
the certified pose estimator to design a multi-stage perception pipeline 
that generalizes the proposed approach to cluttered environments, 
while maintaining its certified guarantees. 
We evaluate the proposed framework using both synthetic and real images 
of various planar objects commonly encountered by autonomous vehicles. 
{\Revised
Using images captured by an event-based camera,
we show that the trained encoder can effectively estimate the
camera pose relative to a traffic sign in accordance with the
certified bound provided by the framework.
}%

\end{abstract}


\section{Introduction}%
\label{sec:intro}

Autonomous vehicles rely on a variety of sensors to estimate their pose and perceive nearby objects in order to safely navigate their environment.
Imaging sensors are vital for detecting and recognizing objects of interest, such as traffic signs, lane dividers and runway markings.
In recent years, deep neural networks (DNNs) have been successfully deployed in autonomous vehicles to process images captured by cameras and provide high-level semantic information about the environment that contributes to the vehicle's perception of its surroundings and the objects within them.
DNNs outperform traditional computer vision algorithms in tasks such as object detection, recognition and tracking, becoming the most prevalent choice for vision-based perception in autonomous vehicles.

However, DNNs are inherently black boxes, and thus do not provide guarantees
on the correctness of their output~\cite{huang2017safety,mitra2024formal}.
This is particularly problematic for autonomous vehicles where DNNs are used in safety-critical subsystems~\cite{abraham2022industry}.
For example, a NN-based sign recognition subsystem may misclassify a stop sign~\cite{eykholt2018robust},
which can lead to loss of life and property. Similarly, an aircraft approaching a runway for landing may rely on a DNN to interpret threshold stripes or centerline markings.
An undetected misclassification in this safety-critical stage can lead to significant deviations in the estimated approach trajectory~\cite{geng2025deterministic},
endangering both passengers and surrounding operations.
These examples underscore the central challenge:
while end-to-end neural networks excel in perception accuracy under nominal
conditions~\cite{abraham2022industry,gehrig2023recurrent},
their lack of certified guarantees leaves safety-critical systems vulnerable to catastrophic errors.

In this paper, we consider the problem of perception-based pose
estimation for autonomous vehicles.
{\Revised
We consider a calibrated camera observing a planar landmark of
known geometry, referred to as the target object, and seek to
estimate the camera pose relative to it from a binary image.
The guarantees apply over a bounded pose region in which the
target remains fully visible, the GGM decoder is
$\delta$-identifiable, and the discrepancy between observed and
generated images remains within the prescribed pixel-noise
budget.
Target symmetries are admissible only when poses generating the
same image lie within $\delta$; otherwise, the corresponding
orientations are excluded from the operating region.
}
We propose a framework for designing correct-by-construction and
certifiable NNs for this problem that comprises three main
components.
The first component leverages the known geometry of planar target
objects to design a NN-like model, referred to as a
\emph{geometric generative model} (GGM), that emulates the image
formation process of a camera sensor.
Specifically, the framework derives the GGM's architecture and
parameters based on the physical parameters of the camera and
the geometry of the target object.
Once designed, a GGM generates an image of the target object from
the perspective of a given camera pose.

\begin{figure}[t]
	\centering
	\includegraphics[width=\columnwidth]{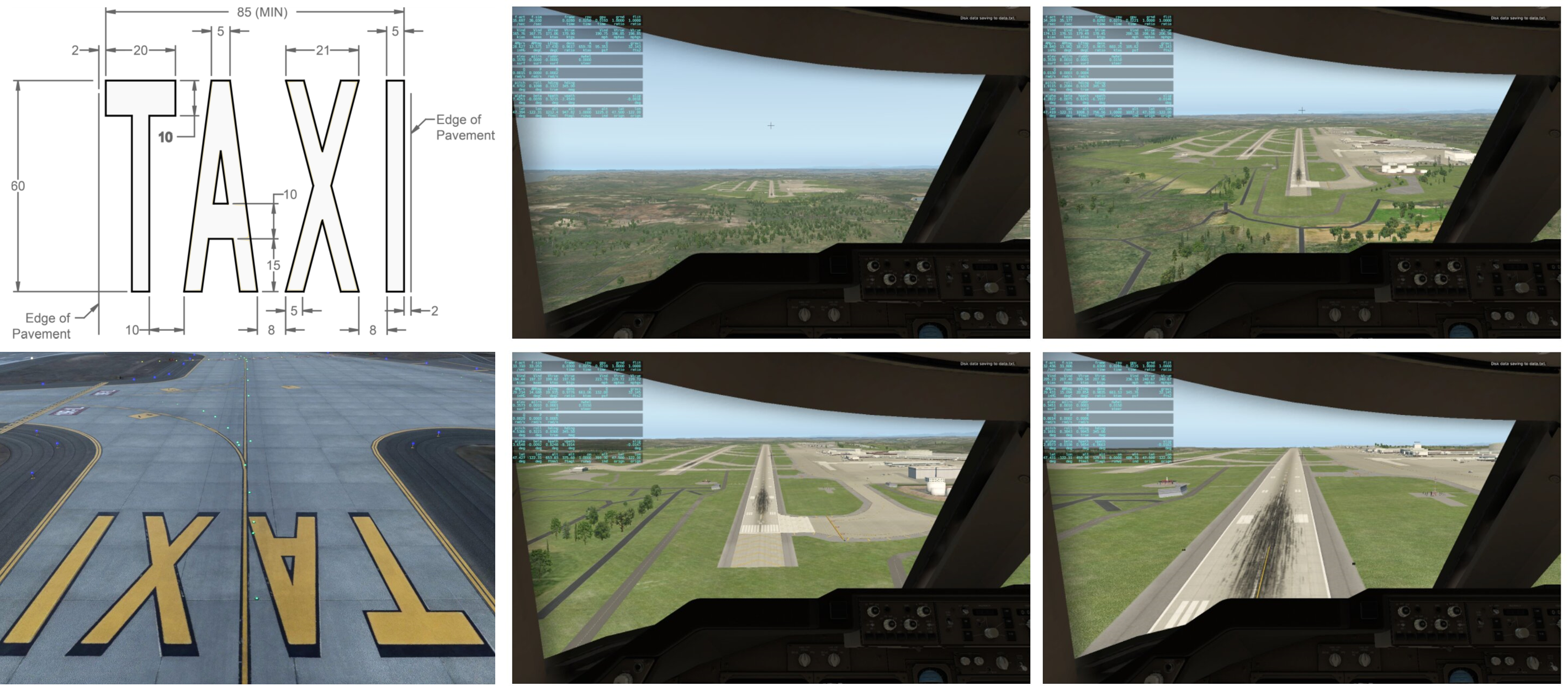}%
	\caption{%
		Examples of planar taxiway markings with known geometry~\cite{faa2020advisory}
    and a simulated runway approach in X-Plane 11~\cite{xplane11}.%
	}%
	\label{fig:planar_markings}%
\end{figure}

While the second component addresses the pose estimation problem under ideal conditions in uncluttered environments,
the third component extends the framework to more realistic scenarios where multiple objects may be present in the scene.
To this end, the framework first tackles the problem of object detection by using light-weight forward reachability algorithm, thanks to the GGM NN like architecture.
Once detection is established, the framework consolidates the certified pose estimator with the object detector
into a multi-stage perception pipeline that incorporates spatial filtering and reachability-based reasoning.
This pipeline enables the framework to generalize certifiable pose estimation to cluttered environments,
ensuring reliable operation even in the presence of distractor objects and noise.

In summary, the main contributions of this paper are as follows:
\begin{itemize}[leftmargin=*]
\item We design GGMs that capture the image formation process of planar objects based on known geometry and camera parameters.
\item We develop a framework for training certifiable NN-based pose estimators that are guaranteed to maintain their certified performance even for data not seen during training.
\item We design a multi-stage perception pipeline that combines certified object detection based on forward reachability analysis with certified pose estimation, enabling operation in cluttered environments.
\item We demonstrate the effectiveness of the framework through numerical simulations and experiments with event-based camera data.
\end{itemize}

The rest of this paper is organized as follows.
\secref{sec:prelims} introduces the camera image-formation
model, polytopic target representation, and pose-estimation
problem.
\secref{sec:ggm} develops the GGM for emulating camera
image formation.
\secref{sec:pose_estimation_uncluttered} presents certified
NN-based pose estimation in uncluttered environments, and
\secref{sec:detection_uncluttered} extends the framework to
target detection.
\secref{sec:pose_estimation_cluttered} integrates detection,
spatial filtering, and pose estimation for cluttered
environments.
Finally, \secref{sec:eval} evaluates the framework using
synthetic and event-camera images of planar targets.

\subsection*{Related Work}

{\Revised
Learned perception methods address object detection,
segmentation, and pose regression~\cite{deng2014deep}, while
classical visual tracking and SLAM estimate camera motion from
geometric correspondences and temporal
optimization~\cite{cadena2017past}.
Recent general-purpose representations further support such
pipelines: DINOv2 provides transferable visual features useful
for correspondence~\cite{oquab2023dinov2}, while SAM provides
promptable masks useful for target
segmentation~\cite{kirillov2023segment}.
Our setting is complementary: single-image camera pose
estimation relative to a known planar target.
These methods can provide correspondences, target masks, or
temporal fusion, while our framework provides a deterministic
pose-error bound when the resulting binary target image
satisfies the prescribed pixel-noise budget.
}

To mitigate those challenges, early formal verification methods cast the verification of feedforward NNs as a Satisfiability-Modulo-Theory (SMT) problem~\cite{huang2017safety}.
These approaches enable exhaustive reasoning about the invariance of classification outputs under bounded input perturbations, formally guaranteeing robustness to small image manipulations.
While precise, SMT-based methods are computationally expensive and scale poorly with network size or input dimensionality, motivating the development of relaxed reachability and convex-optimization formulations~\cite{tran2020verification, tran2021robustness, ostrovsky2022abstraction, khedr2021peregrinn}.
Collectively, these methods provide deterministic safety proofs for limited architectures but remain constrained by computational cost and the complexity of image-based inputs.

Standard NN verification techniques struggle with high-dimensional image inputs, motivating specialized approaches for vision-based control systems.
One direction uses generative models as camera surrogates, training GANs to map system states to plausible images and enabling reachability analysis on the lower-dimensional combined generator-controller network~\cite{katz2022verification, cai2025scalable, geng2025deterministic}.
Alternatively, geometric abstraction methods partition 3D scenes into image-invariant regions or construct interval image representations~\cite{santa2022nnlander, habeeb2023verification, santa2023certified, habeeb2024interval, mitra2024formal}.
While these approaches successfully verify vision-based controllers in specific scenarios, they treat perception model construction and controller design as separate processes, leaving a gap between the perception representation and the downstream control requirements.

Beyond exact verification, several works have sought statistical or contract-based assurance of learned perception systems.
Perception-contract frameworks formalize testable specifications for machine learning components using conformal prediction, systematic abstraction, or counterexample-guided synthesis~\cite{puasuareanu2023closed, li2023refining, sun2024learning, yang2023safe, hsieh2022verifying, ghosh2021counterexample, fremont2020formal}.
Although these methods improve empirical reliability, their guarantees are probabilistic and non-compositional~\cite{abraham2022industry}.
Control-theoretic approaches instead synthesize controllers provably safe despite bounded perception errors by extending control barrier functions to account for measurement uncertainty~\cite{dean2021guaranteeing, dawson2022learning, tong2022enforcing, hsieh2024assuring}.
While effective for runtime safety, these methods assume bounded perception error rather than deriving such bounds from first principles.
Meanwhile, advances in vision architectures have pushed the frontier of perception performance~\cite{gehrig2023recurrent}, but these models remain purely data-driven and unverifiable.

In contrast, our framework unifies physics-based image formation and learning-based estimation within a single certifiable structure.
By encoding camera and object geometry through a GGM with network parameters derived from coordinate transformations rather than learned from data, the perception process itself becomes analyzable via reachability methods.
This enables certified error bounds on pose estimation --- providing a formal layer of correctness absent from existing NN-based perception pipelines.

\section{Preliminaries and Problem Formulation}
\label{sec:prelims}

\subsection{Notation}
We use $\R$, $\N$, $\Z$ and $\B$ to denote the set of real, natural, integer and Boolean numbers, respectively.
We use $v_i$ to denote the $i$-th component of a vector $\bm{v} \in \R^n$.
We also use $\|\bm{v}\|_p$ to denote its $p$-norm.
We use $\mathbf{M}_{ij}$ to denote the $(i,j)$-th entry of a matrix $\mathbf{M} \in \R^{m \times n}$.
For a function $f$, we use $\dom{f}$ to denote its domain.
We use $A \lor B$ to denote the element-wise logical OR operation between two matrices $A$ and $B$, and $A \odot B$ to denote the element-wise (Hadamard) product between them.
We use $\ConvexHull(S)$ to denote the convex hull of a set $S$, and $\Powerset(S)$ to denote its power set.
For an angle $\theta$, we use the nomenclature $\Cos{\theta}$ and $\Sin{\theta}$ to denote $\cos{(\theta)}$ and $\sin{(\theta)}$, respectively.
Given a function $f \colon \mathcal{X} \to \mathcal{Y}$ and a set $S \subseteq \mathcal{X}$, 
we use $\Reachset(f, S)$ to denote the forward reachable set:
\begin{equation*}
	\Reachset(f, S) = \{y \in \mathcal{Y} \mid \exists x \in S .\; y = f(x)\}.
\end{equation*}

\subsection{Target Model}

In this work, a \emph{target} is a planar object whose geometry is known a priori, and
can be represented by the union of a finite number of convex polygons that lie in the same 2D plane.
\begin{definition}[Target Model]\label{def:target_model}
  Given a stationary target in the Euclidean space,
  the target model is composed of a finite set of convex polygons:
  \begin{align*}
    \Target &= \left\{ \Polygon_i \;\Big|\;
		i \! = \! 1, \ldots, \TargetSize,\,
		\Polygon_i = \left( \Vertex_{ij} \right)_{j=1}^{\PolySize_i},\,
		\Vertex_{ij} \in \R^3 \right\},~\text{where:}
  \end{align*}
  \begin{itemize}
    \item $\TargetSize$ is the number of convex polygons in the target;
    \item $\Polygon_i$ is the $i$-th polygon of size $\PolySize_i$, defined as a tuple of the polygon's vertices arranged in a CCW winding order;
    \item $\PolySize_i$ is the number of vertices in the $i$-th polygon; and
    \item $\Vertex_{ij} = (x_{ij},y_{ij},0),\; \forall \Vertex_{ij} \in \Polygon_i$.
  \end{itemize}
\end{definition}

Note that the definition dictates that the polygons are convex and lie in the same 2D plane ($z=0$),
but the target itself does not have to be convex.
Also, all vertices are defined with respect to a local coordinate frame attached to the target.
We will use $\Clutter$ to denote the set of all objects in the environment other than the target,
referred to as the \emph{clutter}.
Since the target is stationary, we can use the target coordinate frame as a global reference frame.

\begin{example}
  \figref{fig:target_example} shows an example of a stop sign as a target.
  The stop sign is composed of 8 convex polygons, \ie $\TargetSize = 8$,
  together forming a non-convex shape $\Target = \left\{ \Polygon_1, \ldots, \Polygon_8 \right\}$.
  Each polygon $\Polygon_i \in \Target$
	 is defined by its vertices, arranged in a counter-clockwise order; \eg $\Polygon_1 = ( \Vertex_{11}, \Vertex_{12}, \Vertex_{13}, \Vertex_{14} )$.
\end{example}

\begin{figure}[t]\vspace{-10pt}
  \centering%
  \subfloat{\includegraphics[width=0.32\linewidth]{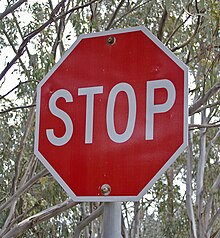}} \hfill%
  \subfloat{\includegraphics[width=0.26\linewidth]{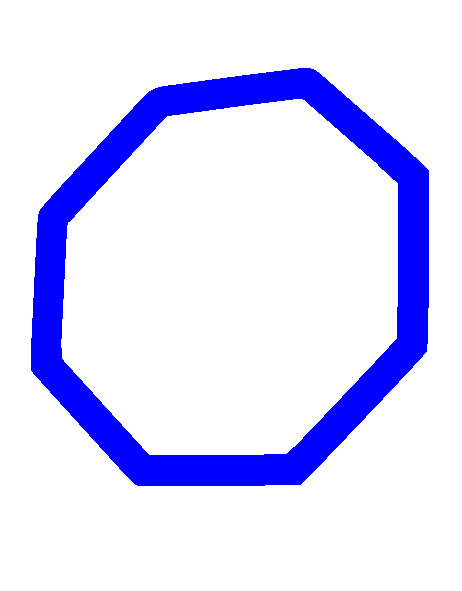}} \hfill%
  \subfloat{\resizebox{0.38\columnwidth}{!}{\begin{tikzpicture}
  \newlength{\octagonradius}
  \setlength{\octagonradius}{40pt}
  \begin{scope}[shift={(0,0)},on grid, scale=1.0, every node/.style={scale=1.0}]
  \draw[thick] (-22.5:\octagonradius) 
      \foreach \a in {22.5,67.5,...,315} { -- (\a:\octagonradius) } -- cycle;
  \draw[thick] (-22.5:1.2\octagonradius) 
      \foreach \a in {22.5,67.5,...,315} { -- (\a:1.2\octagonradius) } -- cycle;
  \foreach \a in {-22.5,22.5,...,315} {
    \draw[thick] (\a:\octagonradius) -- (\a:1.2\octagonradius);
  }
  \foreach \a in {-22.5,22.5} {
      \draw (\a:\octagonradius) node[circle,fill,inner sep=1.2pt]{};
      \draw (\a:1.2\octagonradius) node[circle,fill,inner sep=1.2pt]{};
  }
  \foreach \a in {0,45,...,315} {
    \draw (\a:1.4\octagonradius) node {$\Polygon_{\pgfmathparse{int((\a+22.5)/45+1)}\pgfmathresult}$};
  }
  \draw (22.5:0.9\octagonradius) node[left] {$\Vertex_{12}$};
  \draw (22.5:1.2\octagonradius) node[right] {$\Vertex_{11}$};
  \draw (-22.5:0.9\octagonradius) node[left] {$\Vertex_{13}$};
  \draw (-22.5:1.2\octagonradius) node[right] {$\Vertex_{14}$};
  \draw[thick,-latex] (0,0) -- (0,0.5*\octagonradius) node[above] {$y$};
  \draw[thick,-latex] (0,0) -- (0.5*\octagonradius,0) node[right] {$x$};
  \filldraw[black] (0,0) circle (1.5pt) node[below left] {$(0,0)$};
  \end{scope}
\end{tikzpicture}}}%
  \caption{%
	Stop-sign target, generated image, and corresponding
	eight-polygon non-convex target model.
}
  \label{fig:target_example}
\end{figure}

\subsection{Camera Model}

In this work, we consider a mobile, monochromatic camera that produces digital images.
A camera $\Camera$ is characterized by both its intrinsic and extrinsic parameters.
The camera's \emph{intrinsic parameters} include its focal length, image sensor resolution, and geometric center (principal point).
For simplicity, we assume that the focal length of the camera lens in both the horizontal and vertical directions is the same, denoted by $\FocalLength \in \R_{>0}$.
The image sensor resolution is defined by
the number of pixel rows (height) and columns (width) of the image, denoted by $\ImHeightPxl \in \N$ and $\ImWidthPxl \in \N$, respectively.
We will also assume that the principal point of the camera, where the camera's optical axis intersects the image sensor plane, is at the center of the image, that is, $\nicefrac{\ImWidthPxl}{2}$ and $\nicefrac{\ImHeightPxl}{2}$ in the horizontal and vertical directions, respectively.

The camera's \emph{extrinsic parameters} include its position and orientation, \ie its \emph{pose}, and are captured by the tuple
$\Pose = (\PosX, \PosY, \PosZ, \Roll, \Pitch, \Yaw) \in \PoseSpace \subset \R^6$, where
$\PosX, \PosY, \PosZ$ denote the camera's position with respect to the global coordinate frame;
$\Roll, \Pitch, \Yaw$ denote its roll, pitch, and yaw angles, respectively;
and $\PoseSpace$ is the pose space.
The digital image produced by the camera is denoted by $\Image \in \B^{\ImSizePxl}$.
In ideal cases, this image depends only on the target $\Target$, the camera pose $\Pose$, as well as the intrinsic parameters of the camera as captured by the following definition.

\begin{definition}[Ideal Camera Model]
  {\Revised Let $\Target$ be a target satisfying Definition~1.}
  The ideal camera model with respect to $\Target$ is a function:
  \begin{align*}
    \Camera_{\Target} \colon \PoseSpace \rightarrow \B^{\ImSizePxl}
    \; \text{ given by } \;
    \Camera_{\Target}(\Pose) = \Image_{{\Target}\Pose},
  \end{align*}
  where:
  $\Pose \in \PoseSpace$ is the camera pose, and
  $\Image_{\Target}(\Pose)$ is the ideal image of the target $\Target$ as observed by the camera at pose $\Pose$.
\end{definition}

In practice, camera models are not ideal, since the image captured by the camera also contains other objects in the scene. 
Thus, we identify two types of images based on the presence of other objects in the environment, specifically cluttered and uncluttered images.

\begin{definition}[Image Clutter]\label{def:image_clutter}
  Let $\Target$ be a target object.
  Clutter is defined as the set of all objects in the environment other than the target, denoted by $\Clutter$.
  The image $ \Image(\Pose)$, captured by the camera at pose $\Pose$ in that environment, can be represented as $\Image(\Pose) = \ImageIdeal(\Pose) \lor \ImageClutter(\Pose),$
  where:
  \begin{itemize}
    \item $ \ImageIdeal $ is the ideal image of the target, and
    \item $ \ImageClutter $ is the image of clutter.
  \end{itemize}
  We call an environment cluttered if $ \Clutter \neq \emptyset $, \ie other objects are present in the image, and we call the corresponding image cluttered; otherwise, it is uncluttered.
\end{definition}

This definition is centered around partitioning the objects in the environment into the target object, whose geometry is known a priori, and clutter, whose geometry is unknown.
Based on this partitioning, we can define a system model that captures the target, clutter, camera, and pose space.

\begin{definition}[System Model]
  A system model is a tuple $(\Target, \Clutter, \Camera, \PoseSpace)$, where:
  \begin{itemize}
    \item $\Target$ is the target object;
    \item $\Clutter$ is the clutter in the environment;
    \item $\Camera \colon \PoseSpace \to \B^{\ImSizePxl},\; \Pose \mapsto \Image(\Pose)$ is the camera model; and
    \item $\PoseSpace$ is the pose space of the camera.
  \end{itemize}
\end{definition}

Note that the camera model $\Camera$ depends on both the target $\Target$ and the clutter $\Clutter$.
In an uncluttered environment (\ie $\Clutter = \emptyset$), the camera model reduces to the ideal camera model of the target, \ie $\Camera = \Camera_{\Target}$.

\subsection{Problem Formulation}

{\Revised
Consider a system
$(\Target,\Clutter,\Camera,\PoseSpace)$ comprising a target
$\Target$ with an attached coordinate frame and a camera
$\Camera$ at an unknown relative pose
$\Pose\in\PoseSpace$, producing an image $\Image$ that contains
the target.
Our goal is to estimate the camera pose relative to the target
from the produced image.
}%
To this end, both the target model $\Target$ and the camera model $\Camera$ need to be known a priori.
In practice, however, $\Camera$ is not directly available and, therefore, must be constructed from the camera's intrinsic parameters, which are typically known.
Hence, we define the following problem.

\begin{problem}\label{prob:camera_model}
  Given the intrinsic parameters of a camera $\Camera$ and a target model $\Target$,
  design a target-specific camera model
  $
    \Camera_{\Target} \colon \PoseSpace \rightarrow \B^{\ImSizePxl}$
    such that $\forall \Pose \in \PoseSpace,\;
    \Camera(\Pose) = \Image_{\Target}(\Pose)$,
  where:
  \begin{itemize}
    \item $\PoseSpace$ is the pose space of the camera,
    \item $\Image_{\Target}(\Pose)$ is the ideal image of the target $\Target$ at pose $\Pose$, and
    \item $\ImSizePxl$ is the size of the image produced by the camera.
  \end{itemize}
\end{problem}

In other words, \probref{prob:camera_model} seeks to design a target-specific camera model that reproduces the ideal image of the target at any given pose.
With $\Camera_{\Target}$ in place, we can now define the problem of pose estimation.

\begin{problem}\label{prob:pose_estimation}
  Given a system model $(\Target, \Clutter, \Camera_{\Target}, \PoseSpace)$
  and an image $\Image_{\Pose} = \Image(\Pose),\, \Image_{\Pose} \in \ImageSpace$, produced by the camera at an unknown pose $\Pose \in \PoseSpace$,
  design a pose estimator
	$\NN \colon \ImageSpace_{\Target} \to \PoseSpace$
	such that:
	\begin{align*}
		\forall \Pose \in \PoseSpace:\;
		\PoseEst = \NN(\Image_{\Pose}) \quad \text{and} \quad
		\norm{\Pose - \PoseEst} \leq \ErrorBound,
	\end{align*}
  where:
  \begin{itemize}
    \item $\Image_{\Pose}$ is the image produced by the camera at pose $\Pose$,
    \item $\hat{\Pose}$ is the estimated pose of the camera, and
    \item $\ErrorBound \in \R$ is a small constant.
  \end{itemize}
\end{problem}

Note that the error bound $\ErrorBound$ quantifies the accuracy of the estimated pose $\hat{\Pose}$, and is typically given as a specification.
Thus, examining the factors that affect the error bound $\ErrorBound$ is crucial for designing a robust pose estimator $\NN$.
Also, since the pose $\Pose$ includes both translational and rotational components,
the norm $\norm{\cdot}$ used to quantify the error denotes an appropriate norm on $\R^6$,
\eg a weighted $p$-norm that accounts for the different units of translation and rotation.

\section{NN-based Generative Geometric Model}
\label{sec:ggm}

{\Revised
To address~\probref{prob:camera_model}, we construct a
target-specific \emph{geometric generative model} (GGM) from the
classical pinhole camera model~\cite{hartley2003multiple}.
The GGM is not trained: its architecture and fixed parameters are
derived from the camera intrinsics and target geometry.
For a given pose, it produces the ideal binary image through three
stages:
(i) the analog behavior projects the polygon vertices onto the sensor;
(ii) the digital behavior determines the pixels contained in each
projected polygon; and
(iii) the image composition function combines the polygon images.

The network-like representation is not intended to improve rendering.
Instead, it provides one executable image-formation map used
consistently for synthetic-data generation, identifiability
assessment, grid construction, reachable-image computation, and
spatial-filter construction.
A conventional geometric renderer could produce the same ideal images;
the GGM exposes the operations and fixed parameters in the structured
form used by the subsequent certification procedures.
The decoder is evaluated offline, while only the trained encoder is
evaluated online.
}

\subsection{Pinhole Camera Model for Single Point}

We begin by revisiting the ideal pinhole camera model for a single point~\cite{hartley2003multiple}.
As shown in~\figref{fig:pinhole_camera}, we first introduce the three frames of reference required to describe the image formation process:
\begin{enumerate}
  \item The Target Coordinate Frame (TCF): is attached to the target object, serves as the frame of reference for points in the target.
  \item The Camera Coordinate Frame (CCF): is attached to the camera, provides the frame of reference for defining the camera's position and orientation relative to the target.
  \item The Pixel Coordinate Frame (PCF): is defined over the image plane, serves as the frame of reference for the image produced by the camera.
\end{enumerate}

\begin{figure}[t]
	\centering
  \resizebox{0.9\columnwidth}{!}{\begin{tikzpicture}[isometric view,rotate around x=90, rotate around z=90, rotate around x=70,scale=0.45]
    \tikzset{>=latex} 
    \pgfmathsetmacro{\triangle}{8}
    \pgfmathsetmacro{\triangleDistX}{3}
    \pgfmathsetmacro{\triangleDistY}{-1}
    \pgfmathsetmacro{\triangleDistZ}{12}
    \pgfmathsetmacro{\PCFZ}{-5}
    \pgfmathsetmacro{\gridSizeX}{5}
    \pgfmathsetmacro{\gridSizeY}{6}
    \pgfmathsetmacro{\ratio}{\PCFZ/\triangleDistZ}
  
    \begin{scope}
      \draw[-Latex,thick] (0,0,0) -- (2,0,0) node[pos=1.5] {$x_\texttt{CCF}$};
      \draw[-Latex,thick] (0,0,0) -- (0,2,0) node[pos=1.5] {$y_\texttt{CCF}$};
      \draw[-Latex,thick] (0,0,0) -- (0,0,2) node[pos=1.5,anchor=north,yshift=-5] {$z_\texttt{CCF}$};
      \draw[dash dot] (0,0,\triangleDistZ) -- (0,0,-5) node[pos=0.25,sloped,anchor=south,font=\fontsize{7pt}{1em}\selectfont] {Optical Axis};
      \draw[C0] (0,0,0) circle (0.4);
      
    \end{scope}

    \begin{scope}[canvas is xy plane at z=\PCFZ]
      \fill[lightgray!50, opacity=0.4] (-\gridSizeX,-\gridSizeY) rectangle (\gridSizeX,\gridSizeY);
      \draw[help lines,xstep=0.5,ystep=0.5,gray!80] (-\gridSizeX,-\gridSizeY) grid (\gridSizeX,\gridSizeY);
      \draw[<->,color=black!80] (-\gridSizeX-0.5,-\gridSizeY) -- (-\gridSizeX-0.5,\gridSizeY) node[midway,anchor=north east] {$\ImWidthPxl$};
      \draw[<->,color=black!80] (-\gridSizeX,-\gridSizeY-0.5) -- (\gridSizeX,-\gridSizeY-0.5) node[midway,left] {$\ImHeightPxl$};
      \draw[black!80,fill=black!20] (0,0) circle (7pt);
      \draw[fill=black!80] (0,0) circle (1pt);
      
    \end{scope}
  
    \begin{scope}[canvas is xy plane at z=\triangleDistZ]
      \coordinate (A) at (\triangleDistX,\triangleDistY);
      \coordinate (B) at ($(A) + (30:\triangle)$);
      \coordinate (C) at ($(B) + (150:\triangle)$);
      \draw[C0,thick,fill=C0!20] (A) -- (B) -- (C) -- cycle;
      \node[anchor=south, yshift=2mm, xshift=-2mm, inner sep=1pt] 
        at (A) {$\Vertex_1$};
      \draw[fill=C0] (A) circle (5pt);
      \node[anchor=west, yshift=2mm] 
        at (B) {$\Vertex_2$};
      \draw[fill=C0] (B) circle (7pt);
      \node[anchor=north west, yshift=0mm, xshift=0mm] 
        at (C) {$\Vertex_3$};
      \draw[fill=C0] (C) circle (9pt);
      \coordinate (TCF) at ($(A)!0.5!(B)!0.3333!(C)$);
      \draw[black!80,fill=black!20] (0,0) circle (7pt);
      \draw[fill=black!80] (0,0) circle (1pt);
    \end{scope}
    
    \begin{scope}
      \draw[fill=black!50]
       (TCF) -- ++(0.5,0,0) -- ++(0,0.5,0) -- ++(-0.5,0,0) -- cycle; 
      \draw[-Latex,thick] (TCF) -- ++(0,2) node[pos=1.,anchor=west] {$x_{\texttt{TCF}}$};
      \draw[-Latex,thick] (TCF) -- ++(2,0) node[pos=1.,anchor=south] {$y_{\texttt{TCF}}$};
      \draw[-Latex,thick] (TCF) -- ++(0,0,-2) node[pos=1.,anchor=north] {$z_{\texttt{TCF}}$};
    \end{scope}
  
    \begin{scope}[canvas is xy plane at z=0]
      \coordinate (ZnegA) at ($ (0,0,0)!\ratio!(A) $);
      \draw[densely dotted] (A) -- (ZnegA);
      \fill[C3] (ZnegA) circle (6pt) node[anchor=north west, black] {$\Projection_1$};
      \coordinate (ZnegB) at ($ (0,0,0)!\ratio!(B) $);
      \draw[densely dotted] (B) -- (ZnegB);
      \fill[C3] (ZnegB) circle (5pt) node[anchor=north west, black] {$\Projection_2$};
      \coordinate (ZnegC) at ($ (0,0,0)!\ratio!(C) $);
      \draw[densely dotted] (C) -- (ZnegC);
      \fill[C3] (ZnegC) circle (5pt) node[anchor=east, black] {$\Projection_3$};
      \draw[C3,fill=C3!20,opacity=0.5] (ZnegA) -- (ZnegB) -- (ZnegC) -- cycle;

    \end{scope}
    
    \begin{scope}
      \draw[fill=black!50]
       (\gridSizeX,\gridSizeY,\PCFZ) -- ++(-0.5,0,0) -- ++(0,-0.5,0) -- ++(0.5,0,0) -- cycle;
      \draw[-Latex,thick] (\gridSizeX,\gridSizeY,\PCFZ) -- (\gridSizeX,\gridSizeY-2,\PCFZ) node[pos=1,anchor=south west, xshift=-4] {$x_{\texttt{PCF}}$};
      \draw[-Latex,thick] (\gridSizeX,\gridSizeY,\PCFZ) -- (\gridSizeX-2,\gridSizeY,\PCFZ) node[pos=1,anchor=west] {$y_{\texttt{PCF}}$};
    \end{scope}
  
    \begin{scope}
      \node[scale=1] at (\gridSizeX,-\gridSizeY+1,\PCFZ) [anchor=west,yshift=8] {Camera Sensor};
      \node[scale=1] at (B) [anchor=east,xshift=-5] {Target Polygon};
      
      \draw[thin,scale=1,gray] (0,4,0) -- (0,\gridSizeY*2+2,0);
      \draw[thin,scale=1,gray] (-\gridSizeX,\gridSizeY+0.5,\PCFZ) -- (-\gridSizeX,\gridSizeY+1.5,\PCFZ);
      \draw[<->,color=black!80] (0,\gridSizeY*2+1.5,0) -- (-\gridSizeX,\gridSizeY+1,\PCFZ) node[midway,anchor=north west] {$\FocalLength$};
    \end{scope}

  \end{tikzpicture}}
	\caption{Ideal pinhole camera illustration, showing the target (TCF), camera (CCF), and pixel (PCF) coordinate frames.
  $\Projection_i$ denotes the projection of point $\Vertex_i$ onto the camera sensor.}
  \vspace{-5mm}
	\label{fig:pinhole_camera}
\end{figure}

Using these frames, the mapping of a point on the target to a pixel in the image produced by the camera can be divided into two stages as mentioned earlier.
In the first stage, the analog behavior maps the coordinates of a point on the target
(expressed relative to the camera) to the coordinates of its projection on the camera sensor (relative to a local frame of reference).
Consider a camera pose $\Pose = (\PosX, \PosY, \PosZ, \Roll, \Pitch, \Yaw)$
with respect to the TCF of a target $\Target$.
We then formally define the analog behavior of the camera over a point $\Vertex$ of the target's 2D plane as follows.

\begin{definition}[Analog Behavior]
  The analog behavior of the camera for point $\Vertex$ is a function:
  \begin{align*}
    \AnalogFn^{\Vertex}\colon \PoseSpace \to \R^3, \quad 
		\Pose \mapsto \Projection^\CCF,
  \end{align*}
  where:
	$\Projection$ denotes the projection of point $\Vertex$ by the camera lens onto the camera sensor,
	and
  $\Projection^\CCF \in \R^3$ represents its coordinates in the CCF for a given camera pose $\Pose$.
\end{definition}

Note that an analog function $\AnalogFn^\Vertex$ is defined for a specific point $\Vertex$,
and its exact form is determined by the camera's intrinsic parameters.
To this end, we examine two operations that $\Vertex$ undergoes:
the transformation of its coordinates from the TCF to the CCF, and
the projection of that point by the camera lens onto the camera sensor.
The transformation of the coordinates of point $\Vertex$ in the TCF, denoted by $\Vertex^\TCF$,
to its coordinates in the CCF, denoted by $\Vertex^\CCF$,
is obtained by applying both the rotation and translation matrices~\cite{hartley2003multiple},
denoted $\Rot{\Pose}$ and $\Tran{\Pose}$, respectively, to $p^\TCF$, where:
\begingroup\setlength{\arraycolsep}{0.4em}\begin{align*}
  \Tran{\Pose} 
    & = \begin{bmatrix}
          \PosX & \PosY & \PosZ
        \end{bmatrix}^{\intercal },
    \\[1ex]
  \Rot{\Pose}
    & = \RotZ{\Yaw} \cdot \RotY{\Pitch} \cdot \RotX{\Roll}
    \\
    & = \begin{bmatrix}
        \Cos{\Yaw} \Cos{\Pitch} 
          & \Cos{\Yaw} \Sin{\Pitch} \Sin{\Roll} - \Sin{\Yaw} \Cos{\Roll} 
          & \Cos{\Yaw} \Sin{\Pitch} \Cos{\Roll} + \Sin{\Yaw} \Sin{\Roll} \\
        \Sin{\Yaw} \Cos{\Pitch} 
          & \Sin{\Yaw} \Sin{\Pitch} \Sin{\Roll} + \Cos{\Yaw} \Cos{\Roll} 
          & \Sin{\Yaw} \Sin{\Pitch} \Cos{\Roll} - \Cos{\Yaw} \Sin{\Roll} \\
        -\Sin{\Pitch} 
          & \Cos{\Pitch} \Sin{\Roll} 
          & \Cos{\Pitch} \Cos{\Roll}
      \end{bmatrix},
\end{align*}\endgroup
and $\Pose = (\PosX, \PosY, \PosZ, \Roll, \Pitch, \Yaw)$.
The transformed point coordinates can then be written as:
\begin{align}
	\Vertex^\CCF = \Rot{\Pose} \cdot \Vertex^\TCF + \Tran{\Pose}.
  \label{eq:extrinsic} 
\end{align}

Next, let $\Projection$ be the projection of point $\Vertex$ by the camera lens onto the camera sensor.
Using the camera intrinsic matrix $\Intrinsic$, the homogeneous pixel coordinates of $\Projection$ can be obtained as follows:
\begin{align}
  \Projection^\CCF &= \Intrinsic \cdot \Vertex^\CCF  
  =   
  \begin{bmatrix}
    \FocalLength & 0 & \nicefrac{\ImWidthPxl}{2} \\
    0 & \FocalLength & \nicefrac{\ImHeightPxl}{2} \\
    0 & 0 & 1 
  \end{bmatrix}
  \begin{bmatrix}
    \Vertex^\CCF_x \\
    \Vertex^\CCF_y \\
    \Vertex^\CCF_z
  \end{bmatrix}. \label{eq:intrinsic}
\end{align}
By substituting \eqref{eq:extrinsic} into \eqref{eq:intrinsic},
we can define the exact form of the analog function $\AnalogFn^{\Vertex}$ as follows:
\begin{align}
  \AnalogFn^{\Vertex}(\Pose)
  &= \Projection^\CCF
  = \Intrinsic \cdot \left( \Rot{\Pose} \cdot \Vertex^\TCF + \Tran{\Pose} \right).
  \label{eq:analog_fn}
\end{align}

In the second stage, the digital behavior maps these homogeneous coordinates to discrete pixel values in the image.
We formally define the digital behavior as follows.

\begin{definition}[Digital Behavior]\label{def:digital_behavior}
  The digital behavior of the camera for a point $\Vertex$
  at pixel $\Pixel \in \Z^2$ in the image is a function:
  \begin{align*}
    \DigitalFn^{\Vertex,\Pixel} \colon \Projection^\CCF \mapsto \Image(\Pixel),
  \end{align*}
  where $\Projection^\CCF$ is the homogeneous coordinates of $\Vertex$ in the CCF,
	obtained from the analog stage; and
  $\Image(\Pixel) \in \B$ is the value at pixel $\Pixel$ in the image.
\end{definition}

Note that the digital function $\DigitalFn^{\Vertex,\Pixel}$ is defined for a specific point $\Vertex$ and pixel $\Pixel$.
To realize this function, we compute the corresponding pixel coordinates in the PCF
that correspond to $\Projection^\CCF$, denoted by $\Projection^\PCF$,
by rounding the normalized coordinates of $\Projection^\CCF$ to the nearest integer:
\begin{align*}
  \Projection^\PCF = \left( \left\lfloor \nicefrac{\Projection^\CCF_x}{\Projection^\CCF_z} \right\rfloor, \left\lfloor \nicefrac{\Projection^\CCF_y}{\Projection^\CCF_z} \right\rfloor \right).
\end{align*}
Therefore, the projection of a single point $\Vertex$ results in setting the pixel value at $\Projection^\PCF$ to $1$ in the image produced by the camera provided that $\Projection^\PCF$ lies within the image boundaries.
Consequently, we can define the exact form of the digital function for any pixel
$\Pixel = (\Pixel_x, \Pixel_y)$ as follows:
\begin{align}
  \DigitalFn^{\Vertex,\Pixel}(\Projection^\CCF) = \begin{cases}
    1 & \text{if }  \Projection^\PCF \! = \! \Pixel,\;
                    1 \!\leq\! \Pixel_x \!\leq\! \ImWidthPxl,\;
                    1 \!\leq\! \Pixel_y \!\leq\! \ImHeightPxl,\\
    0 & \text{otherwise}.
  \end{cases}
  \label{eq:digital_fn}
\end{align}

\subsection{Generalized Camera Model for Convex Polygons}

So far we have derived the exact form of the analog and digital
functions~\eqref{eq:analog_fn} and~\eqref{eq:digital_fn} for mapping a single point $\Vertex$ to the image produced by the camera.
However, a target object has virtually an infinite number of points,
making it infeasible to apply the analog and digital functions to each point individually.
To address this issue, we derive a procedure that efficiently maps a target object to the image produced by the camera by leveraging the convex polygon-based representation of target objects that we introduced earlier in~\defref{def:target_model}.

We start by defining the analog behavior for a convex polygon $\Polygon$ as the
collection of analog functions applied to each vertex $\Vertex \in \Polygon$, that is:
\begin{align*}
  \AnalogFn^{\Polygon} : \PoseSpace \to (\R^3)^{\PolySize}, \qquad
  \Pose \mapsto \big(\Projection_1^{\CCF}(\Pose),\ldots,\Projection_{\PolySize}^{\CCF}(\Pose)\big),
\end{align*}
where $\Projection_j^{\CCF}(\Pose) := \AnalogFn^{\Vertex_j}(\Pose) \in \R^3$ denotes the (homogeneous) image-plane coordinates of the projection of $\Vertex_j$ at pose $\Pose$.
We denote the pose-dependent collection of projected vertices by
\begin{align*}
  \ProjectionSet^{\CCF}(\Pose) \;:=\; \big\{ \Projection_1^{\CCF}(\Pose),\ldots,\Projection_{\PolySize}^{\CCF}(\Pose) \big\}.
\end{align*}

The digital behavior for a convex polygon, however, must consider all the points within the polygon's convex hull rather than just its vertices.
Hence, we extend the definition of the digital behavior to convex polygons as follows.
\begin{definition}[Digital Behavior for Convex Polygons]
  The digital behavior of the camera for a convex polygon $\Polygon$
  at pixel $\Pixel \in \{1, \ldots, \ImHeightPxl\} \times \{1, \ldots, \ImWidthPxl\}$ in the image produced by the camera is a function:
  \begin{align*}
    \DigitalFn^{\Polygon,\Pixel} \colon \ProjectionSet^{\CCF} \mapsto \Image(\Pixel)
    \;\;\suchthat\;\;
    \Image(\Pixel) = \begin{cases}
      1 & \text{if } \Pixel \in \ConvexHull\left(\ProjectionSet^{\CCF}\right), \\
      0 & \text{otherwise},
    \end{cases}
  \end{align*}
  where $\ProjectionSet^{\CCF} = \{ \Projection^{\CCF}_1, \ldots, \Projection^{\CCF}_{\PolySize} \}$ is
	the set of homogeneous image coordinates of the projections of the vertices of $\Polygon$,
  and $\ConvexHull(\ProjectionSet^{\CCF})$ is the convex hull of $\ProjectionSet^{\CCF}$.
\end{definition}

Unlike the single-point case, where only one pixel is activated, the polygon digital behavior activates all pixels inside the convex hull of the projected vertices, subject to the field-of-view constraints of the image.
Consequently, instead of applying the digital function to each point in the polygon individually,
we can now obtain the image of the polygon by only applying the polygon digital function at each pixel in the image.

To realize $\DigitalFn^{\Polygon,\Pixel}$, we implement the condition $\Pixel \in \ConvexHull\left(\ProjectionSet^{\CCF}\right)$ by applying an orientation test based on the cross product of the pixel coordinates with each edge of the projected polygon $\ProjectionSet^{\CCF}$.
For a pixel $\Pixel = (\Pixel_x, \Pixel_y)$,
we start by picking the first two projected vertices $\Projection_1, \Projection_2 \in \ProjectionSet^{\CCF}$,
and compute the cross product of the vectors $\overrightarrow{\Projection_1\Projection_2}$ and $\overrightarrow{\Projection_1\Pixel}$ as follows:
\begin{align}
  \Cross_{\Pixel}(1,2)
  &\triangleq \Vector{\Projection_1\Projection_2} \times \Vector{\Projection_1\Pixel} \nonumber \\
  &= \Pixel_y \left( \Projection_{2x} - \Projection_{1x} \right)
     - \Pixel_x \left( \Projection_{2y} - \Projection_{1y} \right) \nonumber \\
     & \phantom{=\mkern\medmuskip} + \left( \Projection_{2x}\Projection_{1y} - \Projection_{2y}\Projection_{1x} \right). \label{eq:cross_product}
\end{align}
We repeat this computation for each consecutive pair of vertices
$\Cross_{\Pixel}(i,n-1)$, $i=1 \ldots n$,
and for the closing edge $\Cross_{\Pixel}({\PolySize},1)$.
If all cross products share the same sign, then the condition $\Pixel \in \ConvexHull\left(\ProjectionSet^{\CCF}\right)$ holds, and the pixel is activated.
That is:
\begin{align}
  \DigitalFn^{\Polygon,\Pixel}(\ProjectionSet^{\CCF}) \! = \! \begin{cases}
    1 & \text{if } \forall i \in \{1, \ldots, \PolySize\!-\!1\}\colon \\
      & \sgn\left( \Cross_{\Pixel}(i,{i\!\!+\!\!1}) \right) \! = \!
        \sgn\left( \Cross_{\Pixel}(\PolySize,1) \right) \\
    0 & \text{otherwise}.
  \end{cases} \label{eq:digital_fn_sign_test}
\end{align}
Finally, by applying~\eqref{eq:cross_product} and~\eqref{eq:digital_fn_sign_test} at every pixel location, we can construct the image of the polygon as follows:
\begin{align}
  \Image_{\Polygon} (\Pose) = \begin{bmatrix}
    \DigitalFn^{\Polygon,(\PixelX,\PixelY)}(\ProjectionSet^{\CCF})
  \end{bmatrix}_{\PixelX=1,\PixelY=1}^{\ImHeightPxl,\ImWidthPxl}.
  \label{eq:polygon_image}
\end{align}
This procedure allows for efficient generation of the polygon image without
evaluating the analog and digital behaviors for infinitely many interior points.

\begin{example}
	\figref{fig:pixel_test} illustrates the orientation test for determining whether a pixel $\Pixel$ lies within the convex hull of the projected vertices $\Projection_1, \Projection_2, \Projection_3$ of the convex polygon from~\figref{fig:pinhole_camera}.
	In this example, the signs of the cross products $\Cross_{\Pixel}(1,2)$, $\Cross_{\Pixel}(2,3)$, and $\Cross_{\Pixel}(3,1)$ is depicted for each pixel, where the color of each pixel indicates the sign of the corresponding cross product at that pixel.
	For example, for $\Pixel = (12,10)$:
	$$ \Cross_{(12,10)}(1,2) < 0,\; \Cross_{(12,10)}(2,3) < 0,\; \Cross_{(12,10)}(3,1) > 0, $$
	so $\DigitalFn^{\Polygon,(12,10)}(\ProjectionSet^{\CCF}) = 0$.
	However, for $\Pixel = (15,14)$, we have:
	$$ \Cross_{(15,14)}(1,2),\; \Cross_{(15,14)}(2,3),\; \Cross_{(15,14)}(3,1)\; < 0, $$
	so $\DigitalFn^{\Polygon,(15,14)}(\ProjectionSet^{\CCF}) = 1$ since all cross products share the same sign.
\end{example}

\begin{figure}[t]
	\centering
  \resizebox{1\columnwidth}{!}{\begin{tikzpicture}[x=1cm,y=1cm, every node/.style={font=\footnotesize}]
  \def\rows{20}               
  \def\cols{24}               
  \def\cell{0.22}             
  \def\gx{gray!40}            
  \def\bg{gray!8}             
	
	\newcommand{\PxA}{11.6}
	\newcommand{\PyA}{13.2}
	\newcommand{\PxB}{14.9}
	\newcommand{\PyB}{18.8}
	\newcommand{\PxC}{18.6}
	\newcommand{\PyC}{13.2}

  \pgfmathsetmacro{\W}{\cols*\cell}
  \pgfmathsetmacro{\H}{\rows*\cell}
	
	\newcommand{\DrawPixelSquare}[2]{%
		\draw[C0, thick] %
			({#1*\cell - \cell},{\H - #2*\cell}) %
				rectangle %
			({#1*\cell},{\H - (#2-1)*\cell});%
	}
	\newcommand{\DrawPixelPositive}[2]{%
		\fill[C0!40, opacity=0.3] %
			({#1*\cell - \cell},{\H - #2*\cell}) %
				rectangle %
			({#1*\cell},{\H - (#2-1)*\cell});%
	}
	\newcommand{\DrawPixelNegative}[2]{%
		\fill[C1!90, opacity=0.3] %
			({#1*\cell - \cell},{\H - #2*\cell}) %
				rectangle %
			({#1*\cell},{\H - (#2-1)*\cell});%
	}
	\newcommand{\DrawPixelPolygon}[2]{%
		\fill[C3!90, fill opacity=0.3, draw=C3!80] %
			({#1*\cell - \cell},{\H - #2*\cell}) %
				rectangle %
			({#1*\cell},{\H - (#2-1)*\cell});%
	}

	\begin{scope}[shift={(0.32*\W+0.1,0)}, scale=1.0, transform shape] 

		\foreach \i in {1,...,\numexpr\cols-1\relax} {\draw[\gx, line width=0.25pt] (\i*\cell,0) -- (\i*\cell,\H);}
		\foreach \j in {1,...,\numexpr\rows-1\relax} {\draw[\gx, line width=0.25pt] (0,\j*\cell) -- (\W,\j*\cell);}
		
		\foreach \x in {1,...,\cols} {
		\foreach \y in {1,...,\rows} {
			\pgfmathsetmacro{\dxA}{\x-\PxA+0.5}
			\pgfmathsetmacro{\dyA}{\y-\PyA+0.5}
			\pgfmathsetmacro{\uxA}{\PxB-\PxA-0.55}
			\pgfmathsetmacro{\uyA}{\PyB-\PyA-0.5}
			\pgfmathsetmacro{\crosspA}{\uxA*\dyA - \uyA*\dxA  }
			\pgfmathsetmacro{\crossnA}{-\uxA*\dyA + \uyA*\dxA }
			\pgfmathsetmacro{\dxB}{\x-\PxB-1.0}
			\pgfmathsetmacro{\dyB}{\y-\PyB-1.0}
			\pgfmathsetmacro{\uxB}{\PxC-\PxB+0.1}
			\pgfmathsetmacro{\uyB}{\PyC-\PyB+0.1}
			\pgfmathsetmacro{\crosspB}{\uxB*\dyB - \uyB*\dxB  +0}
			\pgfmathsetmacro{\crossnB}{-\uxB*\dyB + \uyB*\dxB +0}
			\pgfmathsetmacro{\dxC}{\x-\PxC-0.5}
			\pgfmathsetmacro{\dyC}{\y-\PyC-0.5}
			\pgfmathsetmacro{\uxC}{\PxA-\PxC}
			\pgfmathsetmacro{\uyC}{\PyA-\PyC}
			\pgfmathsetmacro{\crosspC}{\uxC*\dyC - \uyC*\dxC  }
			\pgfmathsetmacro{\crossnC}{-\uxC*\dyC + \uyC*\dxC }
			\ifdim \crossnA pt>0pt
			\ifdim \crossnB pt>0pt
			\ifdim \crossnC pt>0pt
				\DrawPixelPolygon{\x}{\y}
			\fi
			\fi
			\fi
		}}

		\coordinate (O) at (0,\H);
		\draw[-latex, thick] (O) -- ++(0.9,0) node[below right=-2pt] {$x$};
		\draw[-latex, thick] (O) -- ++(0,-0.9) node[right=-2pt] {$y$};

		\path[fill=\bg, draw=black, line width=0.4pt, fill opacity=0] (0,0) rectangle (\W,\H);
	
		\coordinate (PA) at ({\PxA*\cell},{\H - \PyA*\cell});
		\coordinate (PB) at ({\PxB*\cell},{\H - \PyB*\cell});
		\coordinate (PC) at ({\PxC*\cell},{\H - \PyC*\cell});
	
		\path[fill=none, draw=C3!70!black, line width=2.0pt, opacity=0.5] (PA) -- (PB) -- (PC) -- cycle;
		
		\fill[C3!80!black] (PA) circle[radius=0.06cm] node[above left=-2pt,black] {$\Projection_{1}$};
		\fill[C3!80!black] (PB) circle[radius=0.06cm] node[left=2pt       ,black] {$\Projection_{2}$};
		\fill[C3!80!black] (PC) circle[radius=0.06cm] node[above right=-2pt,black] {$\Projection_{3}$};

	\end{scope}

	\begin{scope}[shift={(0,0.68*\H)}, scale=0.32] 
		
		\path[fill=\bg, draw=black, line width=0.4pt] (0,0) rectangle (\W,\H);
	
		\foreach \i in {1,...,\numexpr\cols-1\relax} {\draw[\gx, line width=0.25pt] (\i*\cell,0) -- (\i*\cell,\H);}
		\foreach \j in {1,...,\numexpr\rows-1\relax} {\draw[\gx, line width=0.25pt] (0,\j*\cell) -- (\W,\j*\cell);}
		\coordinate (PA) at ({\PxA*\cell},{\H - \PyA*\cell});
		\coordinate (PB) at ({\PxB*\cell},{\H - \PyB*\cell});
		\coordinate (PC) at ({\PxC*\cell},{\H - \PyC*\cell});
	
		\path[fill=C3!30, draw=C3!70!black, line width=0.4pt, fill opacity=0.0] (PA) -- (PB) -- (PC) -- cycle;
		
		\fill[C3!80!black] (PA) circle[radius=0.06cm]; 
		\fill[C3!80!black] (PB) circle[radius=0.06cm]; 
		\fill[C3!80!black] (PC) circle[radius=0.06cm]; 
		
		\draw[-{Latex[scale=0.6]}, thick, black] (PA) -- (PB) node[midway,left,xshift=-2pt,yshift=0pt,scale=0.5] {$\Vector{\Projection_1\Projection_2}$};

		\foreach \x in {1,...,\cols} {
		\foreach \y in {1,...,\rows} {
			\pgfmathsetmacro{\dx}{\x-\PxA-0.0}
			\pgfmathsetmacro{\dy}{\y-\PyA-1.0}
			\pgfmathsetmacro{\ux}{\PxB-\PxA}
			\pgfmathsetmacro{\uy}{\PyB-\PyA}
			\pgfmathsetmacro{\crossp}{\ux*\dy - \uy*\dx  }
			\pgfmathsetmacro{\crossn}{-\ux*\dy + \uy*\dx }
			\ifdim \crossp pt>0pt
				\DrawPixelPositive{\x}{\y}
			\fi
			\ifdim \crossn pt>0pt
				\DrawPixelNegative{\x}{\y}
			\fi
		}}
	\end{scope}

	\begin{scope}[shift={(0,0.34*\H)}, scale=0.32] 

		\path[fill=\bg, draw=black, line width=0.4pt] (0,0) rectangle (\W,\H);
	
		\foreach \i in {1,...,\numexpr\cols-1\relax} {\draw[\gx, line width=0.25pt] (\i*\cell,0) -- (\i*\cell,\H);}
		\foreach \j in {1,...,\numexpr\rows-1\relax} {\draw[\gx, line width=0.25pt] (0,\j*\cell) -- (\W,\j*\cell);}
		\coordinate (PA) at ({\PxA*\cell},{\H - \PyA*\cell});
		\coordinate (PB) at ({\PxB*\cell},{\H - \PyB*\cell});
		\coordinate (PC) at ({\PxC*\cell},{\H - \PyC*\cell});
	
		\path[fill=C3!30, draw=C3!70!black, line width=0.4pt, fill opacity=0.0] (PA) -- (PB) -- (PC) -- cycle;
		
		\fill[C3!80!black] (PA) circle[radius=0.06cm]; 
		\fill[C3!80!black] (PB) circle[radius=0.06cm]; 
		\fill[C3!80!black] (PC) circle[radius=0.06cm]; 
		
		\draw[-{Latex[scale=0.6]}, thick, black] (PB) -- (PC) node[midway,right,xshift=1pt,yshift=-2pt,scale=0.5] {$\Vector{\Projection_2\Projection_3}$};

		\foreach \x in {1,...,\cols} {
		\foreach \y in {1,...,\rows} {
			\pgfmathsetmacro{\dx}{\x-\PxB-1.1}
			\pgfmathsetmacro{\dy}{\y-\PyB-1.1}
			\pgfmathsetmacro{\ux}{\PxC-\PxB}
			\pgfmathsetmacro{\uy}{\PyC-\PyB}
			\pgfmathsetmacro{\crossp}{\ux*\dy - \uy*\dx  }
			\pgfmathsetmacro{\crossn}{-\ux*\dy + \uy*\dx }
	
			\ifdim \crossp pt>0pt
				\DrawPixelPositive{\x}{\y}
			\fi
			
			\ifdim \crossn pt>0pt
				\DrawPixelNegative{\x}{\y}
			\fi
		}}
	\end{scope}

	\begin{scope}[shift={(0,0)}, scale=0.32] 

		\path[fill=\bg, draw=black, line width=0.4pt] (0,0) rectangle (\W,\H);
	
		\foreach \i in {1,...,\numexpr\cols-1\relax} {\draw[\gx, line width=0.25pt] (\i*\cell,0) -- (\i*\cell,\H);}
		\foreach \j in {1,...,\numexpr\rows-1\relax} {\draw[\gx, line width=0.25pt] (0,\j*\cell) -- (\W,\j*\cell);}
		\coordinate (PA) at ({\PxA*\cell},{\H - \PyA*\cell});
		\coordinate (PB) at ({\PxB*\cell},{\H - \PyB*\cell});
		\coordinate (PC) at ({\PxC*\cell},{\H - \PyC*\cell});
	
		\path[fill=C3!30, draw=C3!70!black, line width=0.4pt, fill opacity=0.0] (PA) -- (PB) -- (PC) -- cycle;
		
		\fill[C3!80!black] (PA) circle[radius=0.06cm]; 
		\fill[C3!80!black] (PB) circle[radius=0.06cm]; 
		\fill[C3!80!black] (PC) circle[radius=0.06cm]; 
		
		\draw[-{Latex[scale=0.6]}, thick, black] (PC) -- (PA) node[midway,above,xshift=0pt,yshift=0pt,scale=0.5] {$\Vector{\Projection_3\Projection_1}$};

		\foreach \x in {1,...,\cols} {
		\foreach \y in {1,...,\rows} {
			\pgfmathsetmacro{\dx}{\x-\PxC-0.5}
			\pgfmathsetmacro{\dy}{\y-\PyC-0.5}
			\pgfmathsetmacro{\ux}{\PxA-\PxC}
			\pgfmathsetmacro{\uy}{\PyA-\PyC}
			\pgfmathsetmacro{\crossp}{\ux*\dy - \uy*\dx  }
			\pgfmathsetmacro{\crossn}{-\ux*\dy + \uy*\dx }
	
			\ifdim \crossp pt>0pt
				\DrawPixelPositive{\x}{\y}
			\fi
			
			\ifdim \crossn pt>0pt
				\DrawPixelNegative{\x}{\y}
			\fi
		}}
	\end{scope}

	\begin{scope}[shift={(0.32*\W+0.1,0)}]
		\def\pix{0.25}          
		\def\gap{0.2}          
		\def\shifty{-0.12}
		\def\shiftx{-2.3}
		\coordinate (LEGNE) at (\W-0.1,\H-0.1);   
		\pgfmathsetmacro{\rowstep}{\pix+\gap}

		\path[fill=white, fill opacity=0.9, draw=black!40, rounded corners=1pt]
			($(LEGNE)+(\shiftx-0.2, -4*\rowstep-0.1)$) rectangle ($(LEGNE)+(0,0)$);

		\path[fill=C0!90, draw=black!20, fill opacity=0.7]
			($(LEGNE)+(\shiftx,-0*\rowstep-\pix+\shifty)$) rectangle ++(\pix,\pix);
		\node[anchor=west,scale=0.8] at ($(LEGNE)+(\shiftx+\pix+0.15,-0.5*\rowstep+0.5*\shifty)$)
			{$\Cross_{\Pixel}(i,i\!+\!1)\geq 0$};

		\path[fill=C1!90, draw=black!20, fill opacity=0.7]
			($(LEGNE)+(\shiftx,-1*\rowstep-\pix+\shifty)$) rectangle ++(\pix,\pix);
		\node[anchor=west,scale=0.8] at ($(LEGNE)+(\shiftx+\pix+0.15,-1.5*\rowstep+0.5*\shifty)$)
			{$\Cross_{\Pixel}(i,i\!+\!1)<0$};

		\path[fill=C3!90, draw=C3!80, fill opacity=0.5]
			($(LEGNE)+(\shiftx,-2*\rowstep-\pix+\shifty)$) rectangle ++(\pix,\pix);
		\node[anchor=west,scale=0.8] at ($(LEGNE)+(\shiftx+\pix+0.15,-2.5*\rowstep+0.5*\shifty)$)
			{$\DigitalFn^{\Polygon,\Pixel}(\ProjectionSet^{\CCF})=1$};

		\path[fill=white, draw=black!50, fill opacity=0.5]
			($(LEGNE)+(\shiftx,-3*\rowstep-\pix+\shifty)$) rectangle ++(\pix,\pix);
		\node[anchor=west,scale=0.8] at ($(LEGNE)+(\shiftx+\pix+0.15,-3.5*\rowstep+0.5*\shifty)$)
			{$\DigitalFn^{\Polygon,\Pixel}(\ProjectionSet^{\CCF})=0$};
	\end{scope}

\end{tikzpicture}}
	\caption{Orientation test for determining whether each pixel $\Pixel$ lies within the convex hull of the projected vertices $\Projection_1, \Projection_2, \Projection_3$ of the convex polygon from~\figref{fig:pinhole_camera}.
	The sign of the cross products are shown for $\Cross_{\Pixel}(1,2)$, $\Cross_{\Pixel}(2,3)$, and $\Cross_{\Pixel}(3,1)$.
	}
    \vspace{-5mm}
	\label{fig:pixel_test}
\end{figure}

\subsection{Compositional Model for $\TargetSize$-Polygon Targets}

The analog and digital behaviors defined so far describe how to generate the image of a single convex polygon $\Polygon_i$ at a given camera pose~$\Pose$.
We now extend this formulation to entire target objects, which are represented as collections of polygons $\Target = \{\Polygon_1, \ldots, \Polygon_\TargetSize\}$.
For each polygon $\Polygon_i \in \Target$, the analog and digital behaviors yield a binary image $\Image_{\Polygon_i}(\Pose)$.
The image of the target $\Target$ is then obtained by applying a composition expression $\Composition$ over the polygon images, generated by the grammar:
\begin{align*}
	\Composition ::=\
		\Image_{\Polygon_i}(\Pose)\ \mid\
		\neg \Composition\ \mid\
		\Composition_1 \vee \Composition_2\ \mid\
		\Composition_1 \wedge \Composition_2\ \mid\
		\Composition_1 \oplus \Composition_2 ,
\end{align*}
where $\vee$, $\wedge$, and $\oplus$ denote element-wise Boolean operations OR, AND, and XOR, respectively. The semantics of these operators are defined in~\tabref{tab:composition_operators}.
Here, all operations are applied pixelwise on the binary images produced by the analog and digital behaviors. At the geometric level, these correspond respectively to the union, intersection, and symmetric difference of the polygons' projections.
Thus, the compositional model provides a general mechanism to assemble the target image from its constituent polygons.

\begin{table}[t]
	\centering
	\caption{Image Composition Operator Semantics}\label{tab:composition_operators}
	\renewcommand{\arraystretch}{1.4}
	\setlength{\tabcolsep}{0.2em}
	\begin{tabular}{rll}
		\toprule
		$\Image_{\Target}(\Pose)$
			& $:= \Composition_{\Target}(\Image_{\Polygon_1}(\Pose),\ldots,\Image_{\Polygon_\TargetSize}(\Pose))$ \\
		$\neg \Image$
			& $:= \mathbf{1} - \Image$
			& $\text{(Negation)}$ \\
		$(\Image_1 \vee \Image_2)$
			& $:= \max\{\Image_1, \Image_2\} = \min\{1, \Image_1 \!+\! \Image_2\}$
			& $\text{(Union)}$ \\
		$(\Image_1 \wedge \Image_2)$
			& $:= \min\{\Image_1, \Image_2\} = \max\{0, \Image_1 \!+\! \Image_2 \!-\! 1\}$
			& $\text{(Intersection)}$ \\
		$(\Image_1 \oplus \Image_2)$
			& $:= (\Image_1 \vee \Image_2)\ \wedge\ \neg(\Image_1 \wedge \Image_2)$
			& $\text{(Difference)}$ \\
		\bottomrule
	\end{tabular}
\end{table}

\begin{figure*}[t]
  \centering
  \resizebox{1.0\linewidth}{!}{\input{tikz/TKZ_GGM_Process_v05.tex}}
  \vspace{-5mm}
  \caption{Computational graph (CG) of (a) the GGM for generating the image $\Image_{\Target}$ of a target object $\Target$ at camera pose $\Pose = (\PosX, \PosY, \PosZ, \Roll, \Pitch, \Yaw)$;
	(b) the analog behavior CG for vertex $\Vertex$;
	(c) the digital behavior CG for polygon $\Polygon_i$ (of size $\PolySize_i = 4$) at pixel $\Pixel = (\PixelX, \PixelY)$; and
	(d) the composition function CG for combining the images of $\TargetSize$ polygons,
	where $\Image_{\Target}(\Pose) = \Image_{\Polygon_1}(\Pose) \vee \Image_{\Polygon_2}(\Pose) \vee \Image_{\Polygon_3}(\Pose) $.
	Some nodes and edges are consolidated, simplified or omitted for clarity.}
    \vspace{-5mm}
  \label{fig:ggm_process}
\end{figure*}

\subsection{Geometric Generative Model (GGM) for N-Polygon Target}

As shown in~\figref{fig:ggm_process}(a), the (GGM) for a target object $\Target$ involves three stages: the analog behavior, the digital behavior, and the image composition function.
We now discuss the implementation of each stage as a NN-based computational graph (CG).
\figref{fig:ggm_process}(b) illustrates the CG for the analog behavior of a single vertex $\Vertex$,
which directly implements the analog function in~\eqref{eq:analog_fn}.
The CG takes as input the camera pose $\Pose$ and returns the homogeneous image coordinates $\Projection^\CCF$ of the projection of $\Vertex$ in the CCF.
The hidden layers of the CG implements the non-linear sinusoidal and multiplicative operations in~\eqref{eq:extrinsic}, followed by the normalization and linear operations in~\eqref{eq:intrinsic}.
Those computations are characterized by two sets of fixed parameters:
the intrinsic parameters of the camera, including $\FocalLength$, $\ImWidthPxl$, $\ImHeightPxl$; and
the 3D coordinates of the vertex $\Vertex^\TCF = (\Vertex_{x}, \Vertex_{y}, 0)$ in the TCF.
Subsequently, each vertex requires a separate CG with identical structure but different parameters.

\figref{fig:ggm_process}(c) illustrates the CG for the digital behavior of a single polygon $\Polygon_i$ (of size $\PolySize_i = 4$) at a single pixel $\Pixel = (\Pixel_x, \Pixel_y)$, implementing the sign test in~\eqref{eq:digital_fn_sign_test}.
The input layer consists of the projected 2D coordinates of the polygon vertices in the CCF, \ie
$\Projection^{\CCF}_1, \ldots, \Projection^{\CCF}_{\PolySize_i}$,
where $\Projection^{\CCF}_j = (\Projection_{x_j}^{\CCF}, \Projection_{y_j}^{\CCF})$,
resulting in an input size of $2 \PolySize_i$.
The output of the NN is a binary value indicating whether the pixel $\Pixel$ lies inside the projection of the polygon in the image plane.
As shown in~\tabref{tab:ggm_layers}, the GGM subnetwork $\NN^{\Polygon_i,\Pixel}$ includes three hidden layers, namely \Layer1, \Layer2, and \Layer3, that model the geometric relations between the polygon vertices and the image pixels.
The layer \Layer1 computes the cross product in~\eqref{eq:cross_product} for each pair of consecutive input projections:
\begin{align*}
  a^{(1)}_1 &= \Pixel_y \left( \Projection_{2x} \!-\! \Projection_{1x} \right)
  \!-\! \Pixel_x \left( \Projection_{2y} \!-\! \Projection_{1y} \right)
  \!+\! \left( \Projection_{2x}\Projection_{1y} \!-\! \Projection_{2y}\Projection_{1x} \right),
\end{align*}
where $a^{(1)}_1$ is the output of the first node in \Layer1.
To determine whether the pixel $\Pixel$ lies within the polygon, and hence its value,
we implement the orientation test in~\eqref{eq:digital_fn_sign_test}:
\begin{align*}
  \sgn\left( a^{(1)}_i \right) = \sgn\left( a^{(1)}_{i+1} \right) \, \iff \,
  \sum_{i=1}^{\PolySize} \left\lvert a^{(1)}_i \right\rvert - \left\lvert \sum_{i=1}^{\PolySize} a^{(1)}_i \right\rvert = 0.
\end{align*}
That is, if the sum of the absolute values of the cross products equals the absolute value of their sum, then all cross products share the same sign and, consequently, the pixel $\Pixel$ lies within the polygon.
This test is implemented in the subsequent layers, \Layer2, \Layer3, and \Layer4.

\begin{table}[!tb] 
  \centering
  \caption{GGM layers for an $n$-polygon.}%
  \label{tab:ggm_layers}
  \setlength{\tabcolsep}{6pt}
  \centering
  \begin{tabular}{llll}
  \toprule
  Layer & Description & Activation & Size \\
  \midrule
  \Layer0 & Input           & None       & $2\PolySize$ (shared) \\
  \Layer1 & Cross product   & Linear     & $\ImSizePxl \times \PolySize$ \\
  \Layer2 & Absolute value  & ReLU (Paired)  & $\ImSizePxl \times (2\PolySize+2)$\\
  \Layer3 & Summation       & Linear     & $\ImSizePxl \times 2 $ \\
  \Layer4 & Output (Binary)  & ReLU (Threshold)  & $\ImSizePxl $ \\
  \bottomrule
  \end{tabular}
\end{table} 

For each output of \Layer1, $a^{(1)}_i$, both the value and its negation are fed into two separate ReLU nodes in \Layer2:
\begin{align*}
	a^{(2)}_{ia} = \max\{a^{(1)}_i, 0\}, \quad
	a^{(2)}_{ib} = \max\{-a^{(1)}_i, 0\}.
\end{align*}
Either $a^{(2)}_{ia}$ or $a^{(2)}_{ib}$ yields the absolute value of $a^{(1)}_i$, while the other yields zero.
Effectively, their sum equals the absolute value of $a^{(1)}_i$:
\begin{align*}
	a^{(2)}_{ia} + a^{(2)}_{ib}
	= \max\{a^{(1)}_i, 0\} + \max\{-a^{(1)}_i, 0\}
	= | a^{(1)}_i |.
\end{align*}
This summation is performed for all outputs of \Layer2 using the first neuron of \Layer3
to compute the sum of the absolute values of the cross products:
\begin{align*}
  a^{(3)}_1 &= \sum_{i=1}^{\PolySize} \left( a^{(2)}_{ia} + a^{(2)}_{ib} \right)
            = \sum_{i=1}^{\PolySize} | a^{(1)}_i |.
\end{align*}
Similarly, the last two ReLU nodes in \Layer2 compute the positive and negative
values of the sum of the cross products:
\begin{align*}
	a^{(2)}_{(\PolySize+1) a} = \sum_{i=1}^{\PolySize} a^{(1)}_i, \quad
	a^{(2)}_{(\PolySize+1) b} = -\sum_{i=1}^{\PolySize}a^{(1)}_i.
\end{align*}
The second neuron in \Layer3 is then used to compute the absolute value of this sum:
\begin{align*}
	a^{(3)}_{2} = a^{(2)}_{(\PolySize+1) a} + a^{(2)}_{(\PolySize+1) b}
	= \left\lvert \sum_{i=1}^{\PolySize} a^{(1)}_i \right\rvert.
\end{align*}
Finally, the output layer \Layer4 comprises a single neuron that outputs:

\begin{align*}
  a^{(4)} = \left\{ \begin{array}{ll}
    1 & \text{if } a^{(3)}_1 = a^{(3)}_2, \\
    0 & \text{otherwise}.
  \end{array} \right.
\end{align*}
That is, the output of \Layer4 is used to set the pixel value $\Image(\Pixel)$
to $1$ if the pixel $\Pixel$ is inside the polygon, and $0$ otherwise.
Each pair of a polygon $\Polygon_i$ and a pixel $\Pixel$ requires a separate CG with identical structure but different parameters;
thus, a total of $\TargetSize \times \ImSizePxl$ digital behavior CGs are needed to generate the image of $\Target$, where each group of $\ImSizePxl$ CGs generates the image of the corresponding polygon $\Polygon_i$.

\figref{fig:ggm_process}(d) illustrates the CG for the composition function that combines the images of $\TargetSize$ polygons.
In this example, the target $\Target$ consists of three polygons, where its geometry is defined by the union of the images of the constituent polygons.
Thus, the CG implements the expression $\Image_{\Target}(\Pose) = \Image_{\Polygon_1}(\Pose) \vee \Image_{\Polygon_2}(\Pose) \vee \Image_{\Polygon_3}(\Pose)$ using the $\min\{1, \cdot\}$ semantics listed in~\tabref{tab:composition_operators}
by summing the outputs of the three polygon subnetworks
and capping the pixel values using $\min\{1, \cdot\}$:
\begin{align*}
	\Image_{\Target}(\Pose)
		&= \Image_{\Polygon_1}(\Pose) \vee \Image_{\Polygon_2}(\Pose) \vee \Image_{\Polygon_3}(\Pose) \\
		&= \min\left\{1,\; \Image_{\Polygon_1}(\Pose) + \min\left\{1, \Image_{\Polygon_2}(\Pose) + \Image_{\Polygon_3}(\Pose)\right\}\right\}\\
		&= \min\left\{1,\; \Image_{\Polygon_1}(\Pose) + \Image_{\Polygon_2}(\Pose) + \Image_{\Polygon_3}(\Pose)\right\}.
\end{align*}
Note that the architecture of the image composition CG depends on the exact expression $\Composition$ used to combine the polygon images.

\section{Certified Vision-Based Pose Estimation for Uncluttered Environments}
\label{sec:pose_estimation_uncluttered}

We now take the first step towards addressing the general problem of vision-based pose estimation described in~\probref{prob:pose_estimation}.
Specifically, we consider the case where the camera is capturing images of a target object in an uncluttered environment (see~\defref{def:image_clutter}).

\subsection{Problem Overview}

Consider a system $(\Target, \Clutter, \Camera_{\Target}, \PoseSpace)$
in an uncluttered environment where the target object $\Target$ is the only one
present in the scene captured by the camera, \ie $\Clutter = \emptyset$ and $\Image(\Pose) = \ImageIdeal(\Pose) $ by definition.
This assumption relaxes the problem of pose estimation by eliminating the need to detect and identify multiple objects within the image.
We can now formalize the relaxed version of~\probref{prob:pose_estimation} as follows.%

\begin{problem}\label{prob:pose_estimation_uncluttered}
  Given a system $(\Target, \Clutter, \Camera_{\Target}, \PoseSpace_{\Target})$,
	where $\Clutter = \emptyset$, and $\PoseSpace_{\Target} \subseteq \PoseSpace$ is the pose space region where $\Target$ is fully visible,	
	design a pose estimator $\NN \colon \ImageSpace_{\Target} \to \PoseSpace $
	such that:
  \begin{align*}
		\max_{\Pose \in \PoseSpace_\Target} \|\Pose - \NN(\Image_\Pose)\| \leq \ErrorBound,
  \end{align*}
  where $\Image_\Pose = \Image(\Pose)$ is the image produced by the camera at pose $\Pose$,
	and $\ErrorBound \in \R_{\geq 0}$ is a user-defined error bound.
\end{problem}

In other words, the goal is to design a pose estimator that, given an image where the target object is fully visible, can estimate the camera's pose with a certified bound on the estimation error.

\subsection{GGM-based Decoder-Encoder Framework}

We address~\probref{prob:pose_estimation_uncluttered} by proposing a two-stage framework for designing a certifiable image-to-pose estimator.
As shown in~\figref{fig:pose_estimation_framework}, the proposed framework for designing a certifiable image-to-pose estimator consists of two main components.
The first component, the \emph{decoder},
is a GGM $\Decoder \colon \PoseSpace_{\Target} \to \ImageSpace$
for generating synthetic images of the target object for any given pose $\Pose \in \PoseSpace_{\Target}$, where $\PoseSpace_{\Target}$ is the bounded pose space of the camera.
The second component, the \emph{encoder} $\Encoder \colon \ImageSpace \to \PoseSpace_{\Target}$,
is a neural network that takes an ideal image of the target object as input and outputs the estimated pose of the camera.

The process of designing the encoder is outlined in~\algref{alg:dec_enc_framework}.
	First, a GGM-based decoder $\Decoder$ is designed for the geometric properties of the target object and the intrinsic parameters of the camera.
    Next, the GGM-based decoder is utilized to train an encoder $\Encoder$ to learn the mapping between the ideal image of the target object and the pose of the camera at which the image was captured. Akin to autoencoders---albeit the order of the encoder and decoder is flipped as shown in~\figref{fig:pose_estimation_framework}---we use reconstruction loss between the input to the decoder and the output of the encoder:
		\begin{align*}
			\mathcal{L}_{\text{Encoder}} = \norm{\Pose - \PoseEst} = \norm{\Pose - \Encoder(\Decoder(\Pose))}.
		\end{align*}
	Next, the GGM-based decoder is utilized to generate a dataset of pose-image pairs $\Dataset_\text{test} = \{(\Pose_i, \Image_i)\}_{i=1}^N \subset \PoseSpace_{\Target} \times \ImageSpace$.
    For our design, it is crucial that the dataset $Dataset_\text{test}$ is constructed by sampling poses $\Pose$ over a grid of size $\eta$.
	This dataset is then used to test the encoder $\Encoder$ and evaluate the empirical error $\epsilon$ which will then be used to establish the worst case upper bound $\ErrorBound$ on the trained encoder.

\begin{figure}[t]
\centering
\includegraphics[width=1.0\columnwidth]{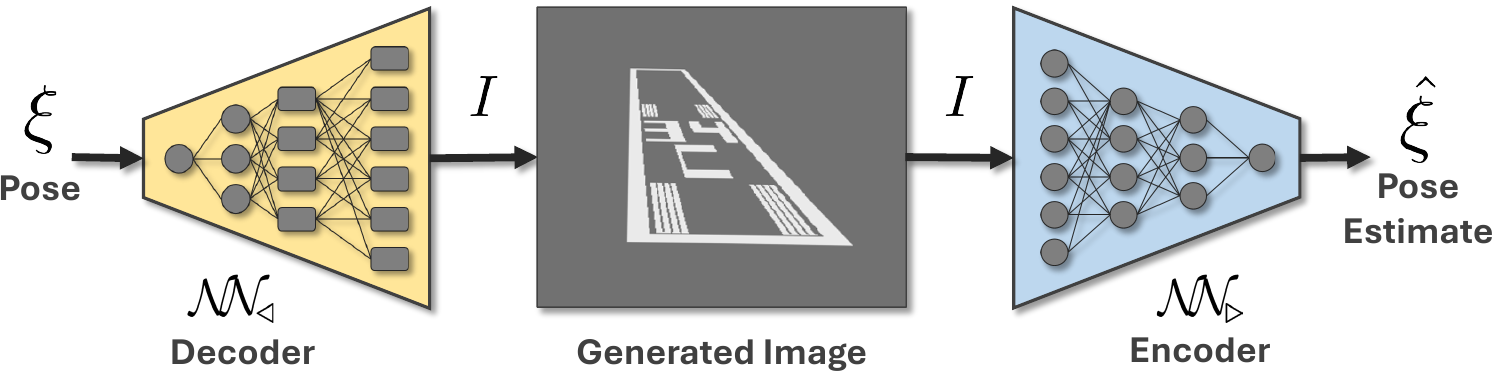}
\vspace{-5mm}
\caption{
	Decoder-encoder framework for designing a certifiable image-to-pose estimator.
}
    \vspace{-2mm}
\label{fig:pose_estimation_framework}
\end{figure}

\begin{algorithm}[t]
	\caption{Decoder--Encoder Framework for Certified Pose Estimation}
	\label{alg:dec_enc_framework}
	\DontPrintSemicolon
	\KwIn{%
		$\Target$, target geometry;
		$\PoseSpace_{\Target}$, bounded pose space;
	}
	\KwOut{%
		$\Encoder$, certified image-to-pose estimator;
	}
	\vspace{0.2em}
	\textbf{\underline{Stage 1: Decoder}}\;
	Construct GGM $\Decoder \colon \PoseSpace_{\Target} \to \ImageSet_{\Target}$ for $\Target$ and $\Camera$\;
	Compute Lipschitz constant $L_{\Decoder}$ of $\Decoder$\;
	Generate testing dataset $\Dataset_\text{test} \gets \emptyset$\;
	Select grid size $\eta$\;
	\For{$i=1,\ldots, \eta^6$}{%
		Sample $\Pose_i$ from a grid over $\PoseSpace_{\Target}$ with resolution $\eta$\;	
	Obtain image from GGM $\Image_i \gets \Decoder(\Pose_i)$\;
		$\Dataset_\text{test} \gets \Dataset_\text{test} \cup \{(\Pose_i, \Image_i)\}$\;
	}
	\vspace{0.2em}
	\textbf{\underline{Stage 2: Encoder}}\;
	Compute Lipschitz constant $L_{\Encoder}$ of $\Encoder$\;
	Sample randomly from $\PoseSpace_{\Target}$ and train the encoder $\Encoder : \ImageSet_{\Target} \to \PoseSpace_{\Target}$ using the reconstruction loss:
    $$ \mathcal{L}_{\text{Encoder}} = \norm{\Pose - \PoseEst} = \norm{\Pose - \Encoder(\Decoder(\Pose))}$$
	
    Evaluate $\Encoder$ on test dataset $\Dataset_\text{test}$ \;
	Compute empirical error $\epsilon \gets \max_{(\Pose_i, \Image_i) \in \Dataset_\text{test}} \norm{\Pose_i - \Encoder(\Image_i)}$\;
	Compute certified bound $\ErrorBound \gets (\eta + \delta) (L_{\Decoder} L_{\Encoder} + 1)  + \epsilon$\;
	\vspace{0.2em}
	\Return $\Encoder$, $\ErrorBound$\;

\end{algorithm}

\subsection{Injectivity Condition and $\delta$-Identifiability}\label{subsec:assessing_delta}

Before we discuss the formal guarantees of the proposed decoder-encoder, we present a necessary condition for the well-posedness of Problem~\ref{prob:pose_estimation_uncluttered}. In particular, for the existence of an encoder that provides unique pose estimates from images, the mapping between the image space and the pose space must be injective.
This means that for each image generated by the decoder, there should be a unique corresponding pose in the pose space.
Conversely, if two different poses produce the same image, the encoder cannot distinguish between them, leading to ambiguity in the pose estimation.

\begin{definition}[Injectivity Condition]\label{def:injectivity}
	The decoder $\Decoder \colon \PoseSpace_{\Target} \to \ImageSpace$ is said to be \emph{injective} if:
	\begin{align}\label{eq:injectivity}
		\forall \Pose_1, \Pose_2 \in \PoseSpace_{\Target},\;
			\Pose_1 \neq \Pose_2 \implies \Decoder(\Pose_1) \neq \Decoder(\Pose_2).
	\end{align}
\end{definition}

Unfortunately, Definition~\ref{def:injectivity} is very restrictive and can not be applied to this problem due to the topological mismatch between the domain and codomain of $\Decoder$. While the pose space $\PoseSpace_{\Target}$ is a continuous space, the image space $\ImageSpace$ is intrinsically discrete, constrained by finite pixel quantization. Consequently, the GGM-based decoder $\Decoder \colon \PoseSpace_{\Target} \to \ImageSpace$ functions as a quantization operator, rendering strict injectivity mathematically impossible; distinct poses that differ by sub-pixel perturbations will inevitably map to the identical images. As a result, the pre-image $\Decoder^{-1}(I)$ of an image $I$ is never a singleton, but rather a set of indistinguishable poses. Therefore, the traditional condition of injectivity in Definition~\ref{def:injectivity} is ill-posed for this setting. To establish rigorous guarantees for the inverse map $\Encoder$, we must relax this requirement to $\delta$-identifiability, replacing the notion of unique recovery with that of recovery within a strictly bounded diameter of uncertainty.

\begin{definition}[$\delta$-Identifiability] \label{def:injectivity2}
The decoder $\Decoder \colon \PoseSpace_{\Target} \to \ImageSpace$ is said to be $\delta$-identifiable if, for any pair of poses $\Pose_1, \Pose_2 \in \PoseSpace_{\Target}$, the following implication holds:
\begin{equation}
    \Decoder(\Pose_1) = \Decoder(\Pose_2) \implies \norm{\Pose_1 - \Pose_2} \le \delta.
\end{equation}
\end{definition}

\begin{assumption}[$\delta$-Identifiable GGM Decoder]\label{asm:injective_decoder}
	We assume that the GGM-based decoder $\Decoder$ is $\delta$-identifiable (i.e., according to \defref{def:injectivity2}) over $\PoseSpace_{\Target}$.%
\end{assumption}

A violation of~\asmref{asm:injective_decoder} occurs when distinct camera poses produce identical or indistinguishable images.
Examples include symmetric target objects viewed from different angles, or poses outside the camera's field of view where the target object is not visible.
{\Revised
Encoder design can be formulated as an inverse problem, with
$\Encoder$ approximating the inverse of the decoder $\Decoder$.
Another source of ambiguity stems from the fact the the decoder maps a continuous pose space to a discrete image space. Hence, strict injectivity cannot generally hold. Instead,  $\delta$-identifiability requires all poses producing the same image to lie within distance $\delta$.
This condition therefore determines
whether the inverse problem admits the bounded-error solution
required by~\probref{prob:pose_estimation_uncluttered},
independently of the approximation accuracy of a particular
encoder.
}

{\Revised

\asmref{asm:injective_decoder} imposes a condition on the target geometry and pose region. The guarantees developed in this section therefore depend on whether this condition holds. We next describe how $\delta$ can be assessed for a given target and pose region.
The stronger condition in~\defref{def:injectivity} is not suitable for this purpose. In addition to being unattainable for the reasons discussed above, strict injectivity is difficult to verify over a continuous pose space. For targets with non-trivial geometry, checking whether every pair of distinct poses produces distinct images is computationally intractable.

In contrast, $\delta$-identifiability provides a quantitative
measure of ambiguity that can be estimated from the
image-generation process, for which we propose two approaches.
In the first, an analytical estimate from the camera model is derived.
Since the analog behavior of the camera is differentiable with respect to $\Pose$,
a lower bound on the smallest singular value of its Jacobian over
$\PoseSpace_{\Target}$ yields a lower bound on the change in
projected vertex locations induced by a change in pose.
Combined with the pixel pitch,
this provides a sufficient condition under which well-separated poses cannot
generate identical digital images.

The second approach estimates $\delta$ empirically and is
adopted in this work. Because the images generated by
$\Decoder$ are binary, two sampled poses are indistinguishable
if and only if their generated images are identical; therefore,
no image-similarity threshold is required. Given a finite sample
$\hat{\PoseSpace}_{\Target} \subset \PoseSpace_{\Target}$, we
evaluate $\Decoder(\Pose)$ for every
$\Pose \in \hat{\PoseSpace}_{\Target}$ and partition the sampled
poses into equivalence classes according to their generated
images. Any equivalence class containing more than one pose
corresponds to a set of sampled poses that cannot be
distinguished from the image alone. We define
\begin{align}\label{eq:delta_empirical}
	\hat{\delta}
	=
	\max_{\Pose_1,\Pose_2\in\hat{\PoseSpace}_{\Target}}
	\left\{
	\norm{\Pose_1-\Pose_2}
	\;:\;
	\Decoder(\Pose_1)=\Decoder(\Pose_2)
	\right\}.
\end{align}
Thus, $\hat{\delta}$ is the smallest value for which
\asmref{asm:injective_decoder} holds over the sampled pose set.
Computing $\hat{\delta}$ requires one decoder evaluation per
sampled pose, while exact grouping of the generated images
eliminates explicit pairwise image comparisons. The resulting
value also depends on the image resolution, since lower
resolutions cause more poses to map to identical images and
therefore increase $\delta$.

}

\subsection{Formal Guarantees on Encoder Performance}	

We now establish formal guarantees on the accuracy of the NN-based encoder $\Encoder$ as follows.

\begin{theorem}[Estimation Error Bounds]%
  \label{thm:estimation_error_bounds}
    Consider a system $(\Target, \Clutter, \Camera_{\Target}, \PoseSpace)$
	where $\Clutter = \emptyset$, i.e., an uncluttered environment.
  Given a neural network encoder $\Encoder$ trained using Algorithm~\ref{alg:dec_enc_framework}
  and assuming that Assumption~\ref{asm:injective_decoder} holds over $\PoseSpace_\Target$, then
	the following bound on the pose estimation error holds:
	\begin{align*}
		\max_{\Pose \in \PoseSpace_\Target} \norm{\Pose - \Encoder(\Image_{\Pose})} \leq (\delta  + \BoundInput) \left(\Lipschitz_{\Dec} \Lipschitz_{\Enc} + 1 \right) + \TrainingError,
	\end{align*}
	where $\Lipschitz_{\Dec}$ and $\Lipschitz_{\Enc}$ are the Lipschitz constants of the decoder and encoder networks, respectively,
    and $\BoundInput$ is the grid size used to generate the test dataset $\Dataset_\text{test}$ in Algorithm~\ref{alg:dec_enc_framework}. Finally, $\TrainingError$ is the maximum empirical error over the test dataset $\Dataset_\text{test}$, i.e.,
    $$\epsilon = \max_{(\Pose_i, \Image_i) \in \Dataset_\text{test}} \norm{\Pose_i - \Encoder(\Image_i)}.$$
\end{theorem}%

\begin{remark}
	\thmref{thm:estimation_error_bounds} provides a worst-case upper bound on the pose estimation error over the entire bounded pose space $\PoseSpace_\Target$.
	This is in contrast to classical machine learning error bounds, which typically provide probabilistic guarantees over only the poses present in the training dataset.
	The bound in~\thmref{thm:estimation_error_bounds} is deterministic and holds for all poses within the specified pose space (even for poses that were not present in the training data),
	offering a stronger assurance of performance for safety-critical applications,
	compared to probabilistic bounds provided in standard machine learning literature, which is limited to the training samples.%
\end{remark}

\begin{proof}

The proof of Theorem \ref{thm:estimation_error_bounds} relies on the properties of Lipschitz continuous functions and the composition of such functions.
We start by noting that the encoder $\Encoder$ and decoder $\Decoder$ are both Lipschitz continuous, which implies that small changes in the input pose will result in bounded changes in the output pose.

First, let $\Pose, \PoseEst \in \PoseSpace_{\Target}$ be the true pose and the estimated pose, respectively.
Also, let $\Image_{\Target}(\Pose)$ be the ideal image of the target object at pose $\Pose$,
and $\Image_{\Pose}$ be the image captured by the camera at pose $\Pose$.
Since an uncluttered environment is assumed, from~\defref{def:image_clutter}, we have:
\begin{align*}
	\Image_{\Pose} &= \Image(\Pose) = \Image_{\Target}(\Pose).
\end{align*}
That is, the image captured by the camera is equal to the image generated by the decoder at the same pose $\Pose$.
We can then write the encoder's output as:
\begin{align}\nonumber
	\PoseEst &= \Encoder(\Image_{\Pose}) = \Encoder(\Image_{\Target}(\Pose)).
\end{align}
We can then express the estimation error as:
\begin{align}\nonumber
	\norm{\PoseEst - \Pose} &= \norm{\Encoder(\Image_{\Pose}) - \Pose} \\
	&= \norm{\Encoder(\Image_{\Target}(\Pose)) - \Pose} \label{eq:proof:estimation_error_a}.
\end{align}
By~\asmref{asm:injective_decoder}, any pose $\Pose'$ such that $\norm{\Pose - \Pose_{\delta}} \le \delta$ can create the same image, i.e., $\Image_{\Target}(\Pose) = \Decoder(\Pose) = \Decoder(\Pose_{\delta})$. Hence, by substituting into~\eqref{eq:proof:estimation_error_a}, we can rewrite the estimation error as:
\begin{align}\label{eq:proof:estimation_error_b}
	\norm{\PoseEst - \Pose} &\le \max_{\substack{\Pose_{\delta} \in \PoseSpace_{\Target} \\ \norm{\Pose - \Pose_{\delta}} \le \delta}  } \norm{\Encoder(\Decoder(\Pose_{\delta})) - \Pose}.
\end{align}

Given a test dataset $\Dataset_\text{test}$ of pose-image pairs with poses sampled over a grid with sample step $\BoundInput$ and maximum test error $\TrainingError$,
there exists a test sample $(\Pose', \Image_{\Pose'}) \in \Dataset_\text{test}$ such that:
\begin{align}\label{eq:proof:constraints}
	\norm{\Pose - \Pose'} &\leq \BoundInput,\quad
	\norm{\Encoder(\Image_{\Pose'}) - \Pose'} \leq \TrainingError, 
\end{align}
where $\Image_{\Pose'} = \Decoder(\Pose')$.
We can now decompose the estimation error in~\eqref{eq:proof:estimation_error_b} using the triangle inequality and the constraints in~\eqref{eq:proof:constraints}:
\begin{align}
	\norm{\PoseEst - \Pose} &\le \max_{\substack{\Pose_{\delta} \in \PoseSpace_{\Target} \\ \norm{\Pose - \Pose_{\delta}} \le \delta}  } \norm{\Encoder(\Decoder(\Pose_\delta)) - \Pose} \nonumber \\
	&\leq \max_{\substack{\Pose_{\delta} \in \PoseSpace_{\Target} \\ \norm{\Pose - \Pose_{\delta}} \le \delta} } \norm{\Encoder(\Decoder(\Pose_\delta)) - \phantom{\lVert}\Encoder(\Image_{\Pose'})} \nonumber \\[-0.9em]
	&\phantom{\leq~\lVert\Encoder(\Decoder(\Pose))} + \norm{\Encoder(\Image_{\Pose'}) -\phantom{\lVert} \Pose'} + \norm{\Pose' - \Pose} \nonumber \\
	&\leq \max_{\substack{\Pose_{\delta} \in \PoseSpace_{\Target} \\ \norm{\Pose - \Pose_{\delta}} \le \delta}  } \norm{\Encoder(\Decoder(\Pose_\delta)) - \Encoder(\Image_{\Pose'})} + \TrainingError + \BoundInput. 
    \label{eq:proof:d}
\end{align}
Since $\Decoder$ and $\Encoder$ are Lipschitz continuous, we have:
\resizebox{0.98\columnwidth}{!}{\parbox{.5\linewidth}{%
\begin{align*}
	\norm{\Decoder(\Pose_\delta) - \Decoder(\Pose')} &\leq \Lipschitz_{\Dec} \norm{\Pose_\delta - \Pose'},\\
	\norm{\Encoder(\Decoder(\Pose_\delta)) - \Encoder(\Decoder(\Pose'))}
	&\leq \Lipschitz_{\Enc} \norm{\Decoder(\Pose_\delta) - \Decoder(\Pose')}.
\end{align*}
}}
\begin{flalign}
  \text{Therefore:}\quad
	\max_{\substack{\Pose_{\delta} \in \PoseSpace_{\Target} \\ \norm{\Pose - \Pose_{\delta}} \le \delta}} 
    \norm{\Encoder(\Decoder(\Pose_\delta)) - \Encoder(\Decoder(\Pose'))}& \nonumber \\[-1em]
    &\hspace{-12em}\leq \Lipschitz_{\Enc} \Lipschitz_{\Dec} \norm{\Pose_\delta - \Pose'}&& \nonumber  \\
    &\hspace{-12em}\leq \Lipschitz_{\Enc} \Lipschitz_{\Dec} (\norm{\Pose_\delta - \Pose} + \norm{\Pose - \Pose'})\nonumber \\
    &\hspace{-12em}\leq \Lipschitz_{\Dec} \Lipschitz_{\Enc} (\delta + \BoundInput). \label{eq:proof:e}
\end{flalign}
Substituting~\eqref{eq:proof:e} into~\eqref{eq:proof:d}, we obtain:
\begin{align*}
	\norm{\PoseEst - \Pose} &\leq \Lipschitz_{\Dec} \Lipschitz_{\Enc} (\delta + \BoundInput) + \TrainingError + \BoundInput
    \\ &\le (\delta + \BoundInput)\left(\Lipschitz_{\Dec} \Lipschitz_{\Enc} + 1\right) + \TrainingError.
\end{align*}
Since we considered an arbitrary pose $\Pose \in \PoseSpace_{\Target}$,
the bound holds for all poses in the space.
This establishes the desired bound on the estimation error, completing the proof.
\end{proof}

Since the constants $\delta$ and $\eta$ are typically small, the bound established in~\thmref{thm:estimation_error_bounds} indicates that the estimation error is dominated by two factors:
\begin{enumerate}
	\item The Lipschitz constants $\Lipschitz_{\Dec}$ and $\Lipschitz_{\Enc}$, which capture the sensitivity of the mapping to perturbations in the input pose and image, respectively. {\Revised The constant $\Lipschitz_{\Dec}$ is determined by the target geometry and the camera model, and is fixed once these are specified. The constant $\Lipschitz_{\Enc}$, by contrast, can be controlled during the training of the encoder using regularization and smoothing techniques.}
	\item The maximum test error $\TrainingError$ which reflects the quality of the encoder's training.
\end{enumerate}

\subsection{Encoder Robustness to Image Noise}

We now extend the results of~\thmref{thm:estimation_error_bounds} to account for images with bounded pixel-wise noise.

\begin{definition}[Bounded Pixel Noise]\label{def:noise}
	An image $\tilde{\Image}_{\Pose}$ is said to have \emph{bounded pixel noise} if:
	\begin{align*}
		\norm{\tilde{\Image}_{\Pose} - \Image_{\Pose}} \leq \NoiseBound,
	\end{align*}
	where $\Image_{\Pose}$ is the ideal image of the target object at pose $\Pose$,
	and $\NoiseBound > 0$ is the maximum noise level.
\end{definition}

The definition above captures various types of noise that may affect at most
$n = \NoiseBound^2$ pixels in the image.
{\Revised
\defref{def:noise} is a worst-case budget rather than a statistical model of the sensor. It does not require that the affected pixels be independent, identically distributed, uniformly spread over the image, or uncorrelated in time. A perturbation that is spatially clustered along the boundary of the target, concentrated in one region of the sensor, or produced by a systematic artifact of the front-end satisfies the definition exactly as a scattered one does, provided the total number of affected pixels remains within the budget. This is a deliberate choice: event-based sensing noise is typically structured, and a model assuming independent pixel corruptions would yield guarantees that hold in expectation but are violated by precisely the correlated perturbations encountered in practice.

Two boundaries of this model are worth stating explicitly. First, the guarantee is conditioned on the budget being respected, and a perturbation affecting more than $n$ pixels falls outside it; this is the case for a large occluder or a global change in illumination. Second, background clutter is not treated as noise in this framework. Clutter is addressed by the spatial filtering stage introduced in~\secref{sec:pose_estimation_cluttered}, which isolates the target before the encoder is applied, rather than by the pixel budget of~\defref{def:noise}.
}
\figref{fig:noisy_image_example} shows an example of a noisy image with bounded pixel noise.%
\begin{figure}[t]
	\centering
	\includegraphics[width=1.0\columnwidth]{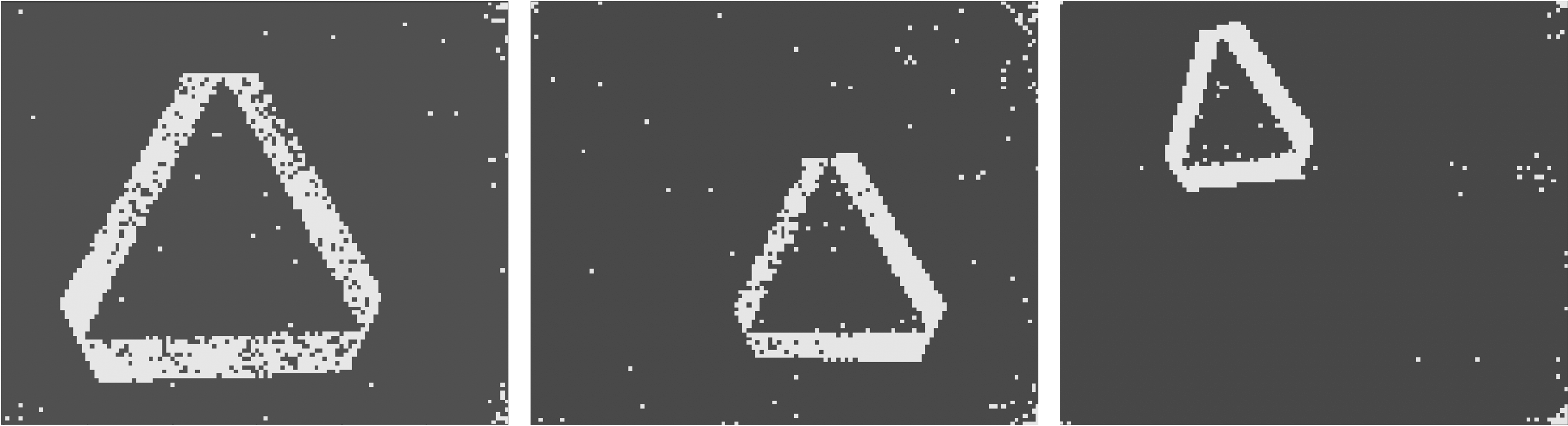}
	\caption{%
		Examples of noisy images of slow-vehicle ahead sign at three different poses in an uncluttered environment.
	}
	\label{fig:noisy_image_example}
\end{figure}
We now establish the certifiable performance of the encoder $\Encoder$ when applied to images with bounded pixel noise.

\begin{corollary}[Estimation Error Bounds under Bounded Noise]
	\label{cor:estimation_error_bounds_noisy}
    Consider a system $(\Target, \Clutter, \Camera_{\Target}, \PoseSpace)$
	where $\Clutter = \emptyset$, i.e., an uncluttered environment.
	Assume that the encoder $\Encoder$ is trained using Algorithm~\ref{alg:dec_enc_framework},
	then, for any noisy image $\tilde{\Image}_{\Pose}$ with bounded pixel noise $\NoiseBound > 0$,
	the following bound on the pose estimation error holds:
	\begin{align*}
		\max_{\Pose \in \PoseSpace_\Target} \norm{\Pose - \Encoder(\tilde{\Image}_{\Pose})} \leq (\delta + \BoundInput) \left(\Lipschitz_{\Dec} \Lipschitz_{\Enc} + 1 \right) + \Lipschitz_{\Enc} \NoiseBound + \TrainingError.
	\end{align*}
\end{corollary}
The proof of~\corref{cor:estimation_error_bounds_noisy} follows the same structure as that of~\thmref{thm:estimation_error_bounds},
with the addition of a term accounting for the noise in the input image.

\section{Certified Vision-Based Object Detection for Uncluttered Environments}
\label{sec:detection_uncluttered}
\subsection{Problem Overview}
In the previous section, we addressed the pose estimation problem under the assumption that the target object is present in the captured image.
In this section, we relax this assumption and consider the object detection problem:
given an image, the goal is to determine whether the target object is present in the scene.
Effectively, the object detection problem serves as a prerequisite to the pose estimation framework.

\begin{problem}[Object Detection]
\label{prob:object_detection}
  Consider a system $(\Target, \Clutter, \Camera_{\Target}, \PoseSpace_{\Target})$
	where $\Clutter = \emptyset$.
	Given an image $\Image \in \ImageSpace$ captured by the camera $\Camera_{\Target}$ at unknown pose $\Pose \in \PoseSpace_{\Target}$,
	design a detector $\Detector \colon \ImageSpace \to \{\True, \False\}$ such that:
	\begin{align*}
		\Detector(\Image) = \left\{
			\begin{array}{ll}
				\True & \text{if } \exists \Pose \in \PoseSpace_{\Target} \text{ s.t. } \Image = \Camera_{\Target}(\Pose),\\
				\False & \text{otherwise.}
			\end{array}
			\right.
	\end{align*}
\end{problem}

A straight forward approach to solve~\probref{prob:object_detection} is to check if the image $I$ belongs to the forward reachable set $\Reachset(\Decoder, \PoseSpace)$ of the GGM-based decoder where $\Reachset(\Decoder, \PoseSpace)$ is defined as:
$$\Reachset(\Decoder, \PoseSpace) = \{I \in \mathcal{I} \mid \exists \Pose \in \PoseSpace .\; I = \Decoder(\Pose)\}.$$
Unfortunately, computing the exact forward reachable set of deep neural networks is computationally
intractable~\cite{huang2017safety,ferlez2021bounding,katz2017reluplex}
and current techniques rely on coarse over-approximation of such sets~\cite{hsieh2022verifying,fatnassi2023bern,fatnassi2024bern,xu2020automatic}.
To that end, we first present a lightweight and accurate technique that exploits the properties of the GGM-based decoder network followed by our approach to solve~\probref{prob:object_detection}.

\subsection{Computing Reachable Set of GGM-Based Detectors Using Grid Sampling}
Given a compact set $\PoseSpace \subset \R^6$, we
construct a sampling grid with step size $\GridStep$ that satisfies $\GridStep < \Lipschitz_{\Dec}^{-1}$, where $\Lipschitz_{\Dec}^{-1}$ is the Lipschitz constant of the GGM decoder network. The output of this step is a discrete set of candidate poses $\PoseSpaceGrid$.
For each candidate pose $\Pose_i \in \PoseSpaceGrid$, we apply the decoder $\Decoder$ to generate the corresponding image:
\begin{align*}
	\Image_i = \Decoder(\Pose_i).
\end{align*}
The reachable set $\Reachset(\Decoder, \PoseSpace)$ is then computed by taking the union over all the generated images $\Image_i$. This process is summarized in Algorithm~\ref{alg:grid_sampling}.

The condition $\GridStep < \Lipschitz_{\Dec}^{-1}$ is critical for correctness of Algorithm~\ref{alg:grid_sampling}. As shown later in the proof of Theorem~\ref{thm:detection_correctness}, this condition 
combined with the grid resolution, ensures to find all poses that produce images that differ strictly by less than 1 pixel. In other words, this approach exploits the fact that the GGM networks imposes a partitioning of the $\PoseSpace$ set into a set of equivalent classes (due to the digital behavior of the GGM), where each equivalent class of poses generate the same exact image. The grid sampling approach above is then guaranteed to cover all possible equivalent classes and hence transforms the problem of computing a reachable set over the continuous space $\PoseSpace$ into an equivalent one over the discrete set $\PoseSpaceGrid$.

\begin{algorithm}[t] 
\caption{Forward reachability via grid sampling}
\label{alg:grid_sampling}
\KwIn{%
	$\PoseSpace \subset \R^6$, compact pose space; \\
    \hspace{3em}$\Decoder$, a GGM-based decoder;
	}
\KwOut{%
	$\PoseSpaceGrid$, discrete set of sampled poses;\\
    \hspace{4em}$\Reachset$, the forward reachable set of $\Decoder$;
}
\SetKwProg{KwFnGridSampling}{function}{\,:}{end}
\KwFnGridSampling{%
	$\GridSampling\,(\PoseSpace, \GridStep)$%
}{
	$\PoseSpaceGrid \leftarrow \emptyset$;\quad $\Reachset \leftarrow \emptyset$; \quad
	$\PoseMin \leftarrow \inf \PoseSpace; \quad \PoseMax \leftarrow \sup \PoseSpace$\;
		$\GridStep \leftarrow \Lipschitz_{\Dec}^{-1} - \textbf{c}$ \tcp*[r]{for any small constant $\textbf{c} > 0$.}
	\ForEach{%
		$k \in \Z^6$ s.t. $\PoseMin_i + k_i \GridStep \le \PoseMax_i$, $i=1,\ldots,6$%
	}{
		$\Pose \leftarrow \PoseMin+k\odot\GridStep$\;
		$\PoseSpaceGrid \leftarrow \PoseSpaceGrid \cup \{\Pose\}$\;
        $\Reachset \leftarrow \Reachset \cup \{\Decoder(\Pose)\}$
	}
	\Return ($\PoseSpaceGrid, \Reachset$)
}
\end{algorithm}

\subsection{GGM-Based Object Detection Framework}

Now, we address~\probref{prob:object_detection} by leveraging the pose estimation framework developed in~\secref{sec:pose_estimation_uncluttered} along with the forward reachability approach in Algorithm~\ref{alg:grid_sampling}.
The proposed detection pipeline is illustrated in~\figref{fig:detection_framework}.
Consider a pose space $\PoseSpace_{\Target} \subset \R^6$
and suppose an image $\Image$ is captured by the camera at some unknown pose $\Pose \in \PoseSpace_{\Target}$.
Given that both the camera parameters and the target object geometry are known a priori,
we can use the pose estimation framework to design a certifiable pose estimator $\Encoder$ for $\Target$ using Algorithm~\ref{alg:dec_enc_framework}.

\begin{figure*}[t]
\centering
\includegraphics[width=1.0\textwidth]{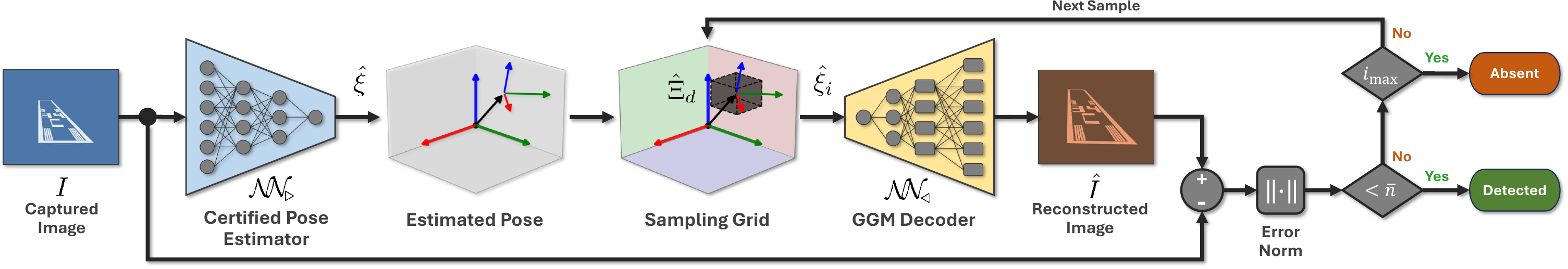}
\vspace{-5mm}
\caption{%
  GGM-based object detection framework.
	Given an input image $\Image$ captured at unknown pose, a certified pose estimator $\Encoder$ is used first to estimate the pose $\PoseEst$, which is then used to construct a sampling grid over the certified error bound set $\PoseSpaceEst$.
	The GGM decoder $\Decoder$ is then applied to each sampled pose in the grid to generate the corresponding images, which are compared with $\Image$ to determine whether the target object is present in the scene.
	}
    \vspace{-3mm}
\label{fig:detection_framework}
\end{figure*}

As detailed in~\algref{alg:object_detection}, the detection procedure operates as follows.
First, we apply the certified pose estimator $\Encoder$ to the input image $\Image$ to obtain an initial pose estimate:
\begin{align*}
	\PoseEst = \Encoder(\Image).
\end{align*}
Recall from~\thmref{thm:estimation_error_bounds} that the error bounds certified for $\Encoder$ guarantee that if $\Pose \in \PoseSpace_{\Target}$, then $\norm{\Pose - \PoseEst} \leq \Radius$, where $\Radius \coloneq (\delta + \BoundInput) \left(\Lipschitz_{\Dec} \Lipschitz_{\Enc} + 1 \right) + \TrainingError$ is the certified error radius.

To verify whether the target is present in the image, we construct the forward reachable set $\Reachset(\Decoder, \PoseSpaceEst)$ where:
$\PoseSpaceEst = \left\{ \Pose \in \PoseSpace_{\Target} \mid \norm{\Pose - \PoseEst} \leq \Radius \right\}.$
We then compute the distance between the original image $\Image$ and the images in the forward reachable set $\Image_i$:
\begin{align*}
	d_i = \norm{\Image - \Image_{\Pose_i}} \qquad \forall \Image_{\Pose_i} \in \Reachset(\Decoder, \PoseSpaceEst).
\end{align*}
If there exists any candidate pose $\Pose_i$ such that $d_i = 0$ (i.e., the generated image matches the input image exactly), we conclude that the target object is present in the scene. Otherwise, the target is absent.
The complete procedure is summarized in~\algref{alg:object_detection}.

\begin{algorithm}[t]
\caption{GGM-Based Object Detection}%
\label{alg:object_detection}
\KwIn{%
	$\Image$, captured image; \\
	\hspace{3em}$\PoseSpace_{\Target}$, target pose space region; \\
}
\KwOut{%
	$\texttt{isDetected}$, detection status;\\
    \hspace{4em}$\PoseEst$, estimated pose;
}
\SetKwProg{KwFnDetect}{function}{\,:}{end}
\KwFnDetect{%
    $\Detector\,(\Image, \PoseSpace_{\Target})$%
}{
	$\PoseEst \gets \Encoder(\Image)$ \tcp*[r]{Apply certified pose estimator to full pose space}
    $\Radius \gets (\delta + \BoundInput) \left(\Lipschitz_{\Dec} \Lipschitz_{\Enc} + 1 \right) + \TrainingError$
    \tcp*[r]{Compute certified error radius}
    $\PoseSpaceEst = \left\{ \Pose \in \PoseSpace_{\Target} \mid \norm{\Pose - \PoseEst} \leq \Radius \right\}$\; 
	$ (\PoseSpaceGrid,\Reachset) \leftarrow \GridSampling(\PoseSpaceEst, \Decoder)$\;
	\ForEach{$(\Pose_i,\Image_{\Pose_i}) \in (\PoseSpaceGrid,\Reachset)$}{
		\If{$\norm{\Image_{\Pose_i} - I} \leq \DetectionThreshold$}{%
			\Return{$(\True, \PoseEst)$} \tcp*[r]{Target present} \label{line:detection_true}%
      \vspace{-0.2em}%
		}\vspace{-0.2em}%
	}
	\Return{\!$(\False, \emptyset)$\!} \tcp*[r]{Samples exhausted, target absent} \label{line:detection_false}%
}
\end{algorithm}

\subsection{Correctness Analysis}

We now provide a formal proof of the correctness of our GGM-based detection framework as follows. 

\begin{theorem}[Correctness of Object Detection]
\label{thm:detection_correctness}
Let $\Decoder : \PoseSpace_{\Target} \to \ImageSpace$ be a GGM-based decoder that satisfies~\asmref{asm:injective_decoder}
with Lipschitz constant $\Lipschitz_{\Dec}$, and let $\Encoder : \ImageSpace \to \PoseSpace_{\Target}$ be a certified pose estimator satisfying the error bound of \thmref{thm:estimation_error_bounds}. If the grid sampling step size used to compute the forward reachable set $\Reachset(\Decoder, \PoseSpaceEst)$ satisfies
$\GridStep 
< \Lipschitz_{\Dec}^{-1},$
then \algref{alg:object_detection} (with $\DetectionThreshold = 0$) is correct in the following sense:
\begin{enumerate}[label=\alph*)]
\item Soundness: If $\Detector(\Image, \PoseSpace_{\Target}) = \True$, then $\Pose \in \PoseSpace_{\Target}$ (the target is present). 
\item Completeness: If $\Pose \in \PoseSpace_{\Target}$ and $\Image = \Decoder(\Pose)$, then $\Detector(\Image, \PoseSpace_{\Target}) = \True$ (the target will be detected).
\end{enumerate}
\end{theorem}

\begin{proof}
We prove each property separately.

\paragraph{Soundness} This follows directly by construction.
If the algorithm returns $\True$ at Line \ref{line:detection_true},
then there exists some $\Pose_i \in \PoseSpaceGrid$
such that $\norm{\Decoder(\Pose_i) - \Image} \le \DetectionThreshold$.
Since $\PoseSpaceGrid \subseteq \PoseSpace_{\Target}$ by construction,
we conclude that the target is present. 

\paragraph{Completeness}
Assume for the sake of contradiction that
the target \emph{is} present, i.e., $\Image = \Decoder(\Pose)$ for some true (unknown) pose $\Pose \in \PoseSpace_{\Target}$,
yet the algorithm returns $\False$ (\ie claims the target is absent).
Since $\False$ was returned, then all grid points in $\PoseSpaceGrid$ are exhausted without finding a match.
That is, $\forall \PoseEst_i \in \PoseSpaceGrid$ we have:
\begin{equation}
\label{eq:no_match}
\norm{\Decoder(\PoseEst_i) - \Image} > \DetectionThreshold.
\end{equation}
Now, by \thmref{thm:estimation_error_bounds}, and since $\Image = \Decoder(\Pose)$ is a clean image of the target at pose $\Pose$,
the encoder estimate $\PoseEst = \Encoder(\Image)$ satisfies the following bound:
\begin{align*}
\norm{\Pose - \PoseEst} \leq \Radius = (\delta + \eta) (\Lipschitz_{\Dec} \Lipschitz_{\Enc} + 1) + \epsilon.
\end{align*}
This means that $\Pose$ lies within the certified error ball of radius $\Radius$ centered at $\PoseEst$.
By the construction of the grid sampling in \algref{alg:grid_sampling},
which discretizes this ball with step size $\GridStep$,
there exists some grid point $\Pose_j \in \PoseSpaceGrid$ such that:
\begin{equation}
\label{eq:grid_proximity}
\norm{\Pose - \Pose_j} \leq \GridStep.
\end{equation}
Now we apply the Lipschitz continuity of the decoder $\Decoder$. Since $\Decoder$ has Lipschitz constant $\Lipschitz_{\Dec}$, we have:
\begin{align}
\norm{\Decoder(\Pose) - \Decoder(\Pose_j)} 
&\leq \Lipschitz_{\Dec} \norm{\Pose - \Pose_j} \label{eq:lipschitz_decoder}  \\
&\leq \Lipschitz_{\Dec} \cdot \GridStep \quad \text{(by \eqref{eq:grid_proximity})} \nonumber \\
&< 1 \quad \text{(by definition of $\GridStep$)}. \label{eq:less_than_one}
\end{align}
Let $\Image_{\Pose_j} = \Decoder(\Pose_j)$ denote the image generated at the grid point $\Pose_j$. From \eqref{eq:less_than_one}, we have:
\begin{align*}
\norm{\Image - \Image_{\Pose_j}} < 1.
\end{align*}
Since fewer than 1 pixel can differ, the two images are identical.
Therefore, at $\Pose_j$, we have:
\begin{align*} 
\norm{\Decoder(\Pose_j) - \Image} = 0 \leq \DetectionThreshold,
\end{align*}
which means Line~\ref{line:detection_true} of \algref{alg:object_detection} should have returned $\True$.
This contradicts our assumption in \eqref{eq:no_match} that all grid points failed the detection test.
Therefore, if the target is present (\ie $\Image = \Decoder(\Pose)$ for some $\Pose \in \PoseSpace_{\Target}$), the algorithm must return $\True$, establishing its completeness.
\end{proof}

{\Revised
For an observed binary image $\tilde{\Image}_{\Pose}$ satisfying
$\norm{\tilde{\Image}_{\Pose}-\Decoder(\Pose)}_2
\leq\NoiseBound$, the completeness claim extends by setting
$\DetectionThreshold=\NoiseBound$.
Indeed, under the same grid condition used above, the proof
establishes that
$\Decoder(\Pose)=\Decoder(\Pose_j)$ for the nearest grid pose
$\Pose_j$, and hence
\begin{align}
\norm{\tilde{\Image}_{\Pose}-\Decoder(\Pose_j)}_2
\leq\NoiseBound.
\end{align}
This threshold alone does not extend soundness to noisy inputs;
soundness additionally requires every non-target image to lie
more than $\NoiseBound$ from the sampled reachable set.
}

\section{Certified Vision-Based Filtering and Pose Estimation for Cluttered Environments}
\label{sec:pose_estimation_cluttered}

This section addresses the problem of pose estimation in cluttered environments,
where multiple objects may be present in the captured image
and the presence of the target object is not known a priori.
The objective is to design an algorithm that determines whether
the target object is present and, if so, estimates the camera pose relative to it.
This extends the earlier assumption of uncluttered environments,
in which the image may only contain the target object.
In that case, a GGM-based detector was used to verify target presence (see \secref{sec:detection_uncluttered}) and, upon detection, to estimate the camera pose relative to the target.

\begin{problem}[Pose Estimation in Cluttered Environments]\label{prob:pose_estimation_cluttered}
	Consider a system $(\Target, \Clutter, \Camera, \PoseSpace_{\Target})$, where
	$\Clutter$ denotes a possibly-nonempty set of clutter objects.
	Let $\PoseSpace_{\Target} \subset \PoseSpace$ denote the subset of poses
	for which the target object $\Target$ is present in the scene.
	Given an image $\Image \in \ImageSpace$ captured by the camera $\Camera_{\Target}$
	at unknown pose $\Pose \in \PoseSpace$, design a detector:
	\begin{align*}
		\Detector^c \colon \ImageSpace \to \{\True, \False\} \times (\PoseSpace_{\Target} \cup \{\emptyset\}),
	\end{align*}
	such that:
	\begin{align*}
		\Detector^c(\Image) &= \left\{
			\begin{array}{ll}
				(\True, \PoseEst) & \text{if } \Pose \in \PoseSpace_{\Target},\\
				(\False, \emptyset) & \text{otherwise,}
			\end{array}
		\right.\\
		\max_{\Pose \in \PoseSpace_\Target} \|\Pose - \PoseEst \| & \leq \;\ErrorBound
		\qquad \text{if } \Pose \in \PoseSpace_{\Target},
	\end{align*}
	where $\PoseEst \in \PoseSpace_{\Target}$ is the estimated pose of the camera.
\end{problem}

\subsection{Spatial Filtering for Clutter Removal}

The key challenge in cluttered environments is that the object detector (see~\secref{sec:detection_uncluttered}) assumes that the target object is the only object present in the image.
In cluttered scenes, the object is surrounded by clutter, which interferes with the detection process.
We address~\probref{prob:pose_estimation_cluttered} by proposing a spatial filtering approach where we partition the pose space into regions, design a specialized detector for each region, and combine their outputs.

Consider a bounded pose space $\PoseSpace_{\Target} \subset \PoseSpace$.
We partition $\PoseSpace_{\Target}$ into a $\PartitionCt$ non-overlapping partitions
$\{ \PoseSpace_1, \ldots, \PoseSpace_{\PartitionCt} \}$ such that:
\begin{align}
    &\cup_{i=1}^{\PartitionCt} \PoseSpace_i = \PoseSpace_{\Target} \label{asm:partition_1}\\
    &I' \odot \bigvee_{I \in \Reachset(\PoseSpace_i, \Decoder)} I \not \in \Reachset(\PoseSpace_i, \Decoder) \nonumber\\
    &\hspace{1em} \text{ for all }
    I' \in
    \Reachset(\PoseSpace_j,\Decoder)
    \text{ with } j \in \{1,...,K\}\setminus \{i\}, \label{asm:partition_2}
\end{align}
where $\bigvee$ denotes the pixel-wise logical OR operation and $\odot$ denotes the pixel-wise logical AND operation.
For each hypercube $\PoseSpace_i$, we construct a spatial filter $\SpatialFilter_i$ that isolates the portion of an image corresponding to all poses $\Pose \in \PoseSpace_i$.
Through the forward reachability analysis discussed in Algorithm~\ref{alg:grid_sampling}, the spatial filter $\SpatialFilter_i$ is realized as the binary mask:
\begin{align} \label{eq:spatial_filter_mask}
	\Mask_i 
    = \bigvee_{\Pose \in \PoseSpace_{i}} \Decoder(\Pose) \quad = \bigvee_{I \in \Reachset(\PoseSpace_{i},\Decoder)} I. 
\end{align}
That is, the mask $\Mask_i$ represents the union of all images that can be generated by $\Decoder$ for poses in $\PoseSpace_i$.
Applying the spatial filter $\SpatialFilter_i$ to an image $\Image$ results in the filtered image:
\begin{align*}
	\widetilde{\Image}_i = \SpatialFilter_i(\Image) = \Image \odot \Mask_i.
\end{align*}

\algref{alg:spatial_filter_mask} outlines the procedure for constructing the spatial filter mask $\Mask_i$
for a given pose space partition $\PoseSpace_i$.
Since $\PoseSpace_i$ is a continuous set, it cannot be directly used to construct the mask.
Therefore, we start by constructing the forward reachable set $\Reachset(\PoseSpace_{i},\Decoder)$ to find all possible images that can be generated within $\PoseSpace_i$.
The mask $\Mask_i$ is then constructed by accumulating the logical OR of all generated images.
Due to the choice of the grid sampling step size $h$ in Algorithm~\ref{alg:grid_sampling}, the algorithm ensures that all possible images generated by poses in $\PoseSpace_i$ are captured in the mask $\Mask_i$.

\begin{algorithm}[t]
	\caption{Spatial Filter Mask Construction}
	\label{alg:spatial_filter_mask}
	\newcommand{\GridStepSub}{\Delta_s}
	\KwIn{$\PoseSpace_i$, pose space region; \\
    \hspace{3em} $\Decoder$, GGM decoder; 
    }
	\KwOut{$\Mask_i$, spatial filter mask;}
	\BlankLine
	
	Initialize $\Mask_i \leftarrow \mathbf{0}^{\ImHeightPxl \times \ImWidthPxl}$\;
	Construct the forward reachable set  
    $(-, \Reachset_i) \leftarrow \GridSampling(\PoseSpace_i, \Decoder)$\;
	\ForEach{$I \in \Reachset_i$}{
		$\Mask_i \leftarrow \Mask_i \vee I$ \tcp*{Accumulate via logical OR}
		}
		\Return $\Mask_i$\;
\end{algorithm}

\subsection{Correctness of Spatial Filtering}

We now establish the correctness properties of the spatial filtering mechanism described earlier.
The correctness relies on two fundamental properties: preservation and exclusion.
The preservation property ensures that if the true pose $\Pose$ lies within a partition $\PoseSpace_i$,
the spatial filter retains the target object in the filtered image.
The exclusion property ensures that if the true pose lies outside all partitions, the spatial filter removes the target object from all filtered images.

\begin{lemma}[Spatial Filter Preservation]\label{lem:spatial_filter_preservation}
  Let $\PoseSpace_i$ be a pose space partition---that satisfies Assumption~\ref{asm:injective_decoder}---and $\Mask_i$ be the corresponding spatial filter mask constructed by~\algref{alg:spatial_filter_mask}.
  For any pose $\Pose \in \PoseSpace_i$, the filtered image satisfies:
  \begin{align*}
    \Detector(\SpatialFilter_i(\Image_\Pose), \PoseSpace_i) = 
    \Detector(\Image_\Pose \odot \Mask_i, \PoseSpace_i) =
    \True.
  \end{align*}
  where $\Image_\Pose = \Decoder(\Pose)$ is the ideal image of the target at pose $\Pose$.
\end{lemma}

\begin{proof}
  By construction, the mask $\Mask_i$ is the logical OR of all images in $\ImageSpace_i = \Reachset(\Decoder, \PoseSpace_i)$.
  Since $\Pose \in \PoseSpace_i$, it follows that $\Image_\Pose \in \ImageSpace_i$.
  Therefore, $\Image_\Pose$ is included in the logical OR operation that defines $\Mask_i$, which implies that for every pixel where $\Image_\Pose$ has value $1$, the mask $\Mask_i$ also has value $1$.
  Consequently, $\Image_\Pose \odot \Mask_i = \Image_\Pose$. The result then follows directly from Theorem~\ref{thm:detection_correctness} since $\Image_\Pose \in \PoseSpace_i$.
\end{proof}

\begin{lemma}[Spatial Filter Exclusion]\label{lem:spatial_filter_exclusion}
  Let $\PoseSpace_i$ be a pose space partition---that satisfies Assumption~\ref{asm:injective_decoder}---and a corresponding spatial filter mask $\Mask_i$.
  For any pose 
  $\Pose \notin \PoseSpace_i$, 
  the filtered image satisfies:
  \begin{align*}
    \Detector(\SpatialFilter_i(\Image_\Pose), \PoseSpace_i) = 
    \Detector(\Image_\Pose \odot \Mask_i, \PoseSpace_i) =
    \False.
  \end{align*}
\end{lemma}

\begin{proof}
    We prove this by contradiction. Assume that 
    $\Pose \notin \PoseSpace_i$
    but $\Detector(\SpatialFilter_i(\Image_\Pose), \PoseSpace_i) = \True$. It follows from Theorem~\ref{thm:detection_correctness} that this equality implies that $\SpatialFilter_i(\Image_\Pose) \in \Reachset(\PoseSpace_i, \Decoder)$ which contradicts the condition in~\eqref{asm:partition_2}.

\end{proof}

\subsection{Certified Pose Estimation in Cluttered Environments}

As illustrated in~\figref{fig:pipeline_cluttered}, our approach for pose estimation in cluttered environments integrates the spatial filtering mechanism with certified object detection to produce reliable detection and pose estimates.
The pipeline comprises three sequential stages: spatial filtering, certified detection, and selection.
In the first stage, the input image $\Image$ is processed by $\PartitionCt$ spatial filters $\{\SpatialFilter_1, \ldots, \SpatialFilter_{\PartitionCt}\}$, each corresponding to a partition $\PoseSpace_i$ of the target pose space.
Each spatial filter applies its mask $\Mask_i$ to the input image, yielding the filtered image $\widetilde{\Image}_i = \Image \odot \Mask_i$.
This operation isolates the portion of the image that could correspond to the target object at poses within $\PoseSpace_i$, while suppressing clutter from other regions.
In the second stage, each filtered image $\widetilde{\Image}_i$ is analyzed by its certified object detector ${\Detector}_i$ that corresponds to the pose space partition $\PoseSpace_i$ (see~\algref{alg:object_detection}).

\begin{figure}[t]
\centering
\includegraphics[width=1.0\columnwidth]{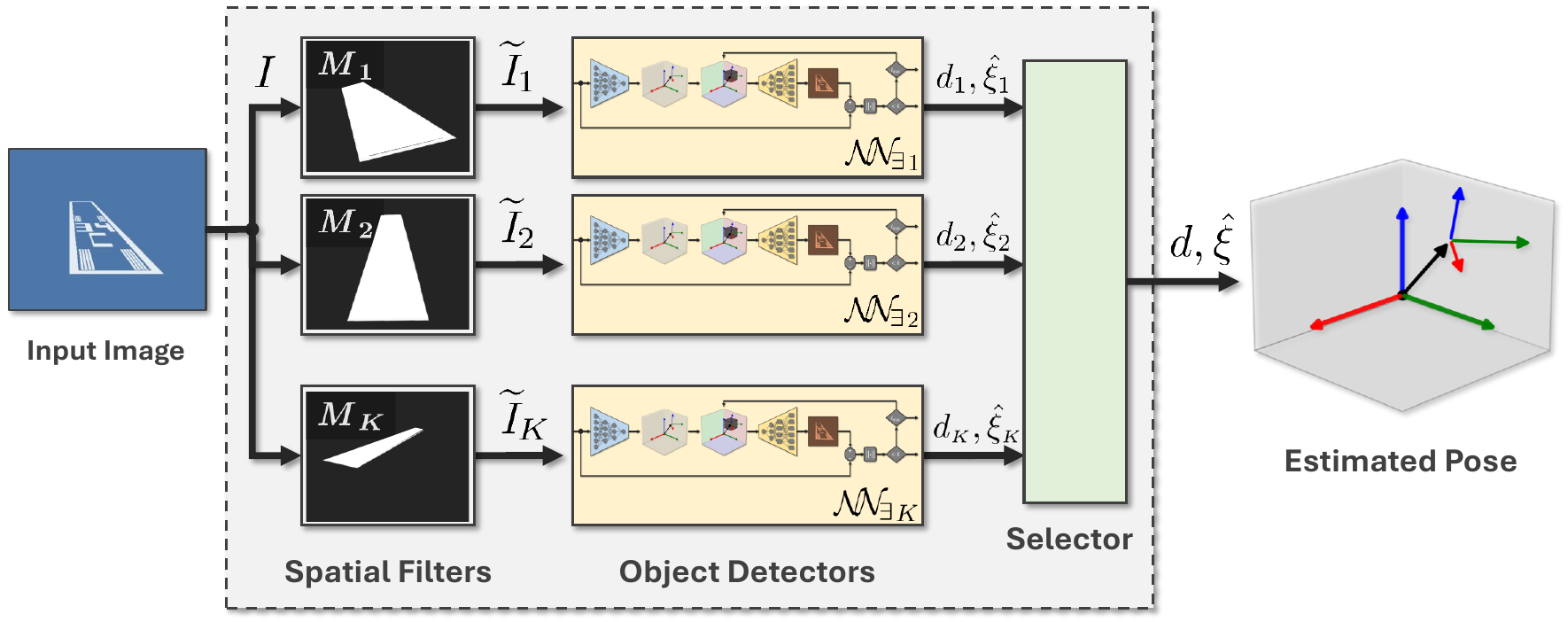}
\caption{%
GGM spatial filtering pipeline for certified pose
estimation in cluttered environments. An input image is
processed by spatial filters associated with pose-space
partitions and analyzed by object detectors. Their outputs
are combined by minimizing residual clutter, yielding the
pose estimate.
}%
\label{fig:pipeline_cluttered}
\end{figure}

For each filtered image, the detector either signals the absence of the target object ($d_i = \False$) or provides a pose estimate $\PoseEst_i$ with detection status $d_i = \True$.
The set of all detections is collected in $\DetectionSet$.  
Since multiple objects might exist in the scene, the algorithm then returns $(\False, \emptyset)$ if no detections were made, or the entire set $\DetectionSet$ otherwise. A third (optional) stage, is a selection mechanism that picks all of a subset of the set $\DetectionSet$ depending on user specific criteria (e.g., closest detected object).
The complete pipeline is formalized in~\algref{alg:pose_estimation_cluttered}.

\begin{algorithm}[t]
\caption{Certified Pose Estimation in Cluttered Environments}
\label{alg:pose_estimation_cluttered}
\KwIn{$\Image$, input image; $\{\Mask_1, \ldots, \Mask_{\PartitionCt}\}$, spatial filter masks; $\Detector$, certified detector; $\Decoder$, GGM decoder; $\ErrorBound$, certified error bound;}
\KwOut{$(d, \PoseEst)$ (detection status and estimated pose)}
\BlankLine

$\DetectionSet \leftarrow \emptyset$ \tcp*{Detection set}

\For{$i = 1$ \KwTo $\PartitionCt$}{
  $\widetilde{\Image}_i \leftarrow \Image \odot \Mask_i$ \tcp*{(1) Spatial filtering}
  $(d_i, \PoseEst_i) \leftarrow \Detector\left(\widetilde{\Image}_i\right)$ \tcp*{(2) Certified detection}
	$\DetectionSet \leftarrow \DetectionSet \cup \{ \PoseEst_i\}$\;
  
}

$\overline{\DetectionSet} \leftarrow \texttt{User-Specified-Criteria}(\DetectionSet)$ \tcp*{(3) Selection}
\Return $\overline{\DetectionSet}$\;
\vspace*{0.2em}
\end{algorithm}

However, the correctness of this localized detection pipeline relies on the purity of the filtered image $\widetilde{\Image}_{i^*} = \Image \odot \Mask_{i^*}$ for the partition $i^*$ where the unknown pose belongs to. If the spatial filter $\Mask_{i^*}$ captures significant portions of the background clutter alongside the object, the resulting signal may breach the detector $\Detector$ decision threshold $\tau$.

Therefore, to guarantee valid pose estimation and object detection, we must constrain the spatial relationship between the target object $\Target$ and environmental clutter $\Clutter$. We introduce a Clutter Intrusion Bound. This condition posits that either the scene layout provides sufficient physical clearance to spatially isolate the object, or, if overlap occurs, the magnitude of the clutter's contribution to the filtered image $\Image \odot \Mask_{i^*}$ remains below the noise level $\overline{n}$ used to tune the detection threshold $\tau$. This duality allows us to trade off spatial constraints (requiring clear backgrounds) with detector robustness (tolerating minor occlusions or background noise), providing a flexible theoretical basis for deployment.

\begin{assumption}[Bounded Clutter Intrusion]
\label{asm:bounded_clutter}
Consider a bounded pose space $\PoseSpace_{\Target} \subset \PoseSpace$ and a non-overlapping partitioning $\{ \PoseSpace_1, \ldots, \PoseSpace_{\PartitionCt} \}$ that satisfies~\eqref{asm:partition_1}-\eqref{asm:partition_2}. Let $\PoseSpace_{i^*}$ denote the partition for which the unknown pose $\Pose$ belongs, i.e., $\Pose \in \PoseSpace_{i^*}$. Let $\overline{n}$ be the noise threshold used to tune the detection threshold $\tau$ used by $\Detector$. We assume that $\overline{n}$ is selected such that $\norm{\Image_{\Clutter} \odot \Mask_{i^*}} \le \overline{n}.$
\end{assumption}

This inequality consolidates two distinct operating regimes into a single constraint:
\begin{enumerate}
    \item \textbf{Spatial Isolation ($\overline{n} = 0$):} The target $\Target$ is sufficiently separated from clutter such that the filter $M_{i^*}$ effectively excludes all clutter pixels. This corresponds to a minimum spatial clearance that depends on $\PoseSpace_{i^*}$ volume.
    \item \textbf{Robust Detection ($\overline{n} > 0$):} Clutter partially intrudes into the filter window, but its magnitude is suppressed by the detector's inherent noise rejection threshold $\tau$.
\end{enumerate}

\begin{theorem}[Pipeline Correctness]\label{thm:pipeline_correctness}
Consider a pose space $\PoseSpace_{\Target}$ that satisfies Assumption~\ref{asm:injective_decoder} and a partitioning $\{\PoseSpace_1, \ldots, \PoseSpace_{\PartitionCt}\}$ that satisfies~\eqref{asm:partition_1}-\eqref{asm:partition_2}.  Let $\PoseSpace_{i^*}$ denote the partition for which the unknown pose $\Pose$ belongs, i.e., $\Pose \in \PoseSpace_{i^*}$.
  For a cluttered image $\Image_\Pose$ captured at unknown pose $\Pose \in \PoseSpace_{\Target}$, and under the Bounded Clutter Intrusion assumption (Assumption~\ref{asm:bounded_clutter}), the pipeline in~\algref{alg:pose_estimation_cluttered} satisfies:

  \begin{enumerate}
    \item \textbf{Uniqueness under Spatial Isolation:} If the spatial filter $M_{i^*}$ achieves perfect spatial isolation (i.e., $\overline{n} = 0$), the algorithm returns a unique valid pose, i.e.:
    \begin{equation}
        \DetectionSet = \{\PoseEst^*\} \; \text{and} \; \norm{\Pose - \PoseEst^*} \le (\delta + \BoundInput) \left(\Lipschitz_{\Dec} \Lipschitz_{\Enc} + 1 \right) + \TrainingError. \label{eq:uniqueness}
    \end{equation}
    
    \item \textbf{Validity under Robust Detection:} If the clutter intrusion in $M_{i^*}$ is non-zero but bounded ($\overline{n} > 0$), the algorithm returns a candidate set $\DetectionSet$ that is guaranteed to include the ground truth, i.e.:
    \begin{align}
        \vert \DetectionSet &\vert \ge 1 \quad \text{and} \quad \exists \PoseEst^* \in \DetectionSet \text{ such that: } \nonumber\\ & \norm{\Pose - \PoseEst^*} \le (\delta + \BoundInput) \left(\Lipschitz_{\Dec} \Lipschitz_{\Enc} + 1 \right) + \Lipschitz_{\Enc}  \overline{n} + \TrainingError. \label{eq:validity}
    \end{align}
\end{enumerate}
        
\end{theorem}

In other words, Theorem~\ref{thm:pipeline_correctness}  distinguishes the behavior of Algorithm~\ref{alg:pose_estimation_cluttered} under these two operating regimes by analyzing the cardinality of the solution set $\DetectionSet$. In the Spatial Isolation regime, where the filter $M_{i^*}$ successfully excludes all external clutter, the filtered image $I \odot M_{i^*}$ conforms strictly to the generative model of the object. Consequently, the inverse map converges to a unique solution, returning a singleton set containing the optimal pose estimate. Conversely, in the Robust Detection regime, the intrusion of clutterâ€”even within the permissible bound---may introduce ambiguity. Under these conditions, our algorithm generalizes to a set-valued estimator, returning a discrete set of candidate poses. While the presence of clutter prevents guaranteeing uniqueness (cardinality of one), the calibration of the bound $\tau$ ensures validity: the set of candidates is guaranteed to contain the ground truth pose.

\begin{proof}
    We start by proving~\eqref{eq:validity}.
  It follows from~\eqref{asm:partition_1} that 
  there exists some $i^*$ such that $\Pose \in \PoseSpace_{i^*}$.
  By~\lemref{lem:spatial_filter_preservation} and Assumption~\ref{asm:bounded_clutter}, 
  the certified detector from~\algref{alg:object_detection} correctly detects the target object in $I_\Pose \odot \Mask_{i^*}$, it returns $d_{i^*} = \True$ for $\widetilde{\Image}_{i^*}$.
  Therefore, the algorithm adds $d_{i^*} = \True$ and $\PoseEst_{i^*}$ to $\DetectionSet$.
  The uniqueness of solutions in~\eqref{eq:uniqueness}, follows from Assumption~\ref{asm:bounded_clutter} and \lemref{lem:spatial_filter_exclusion} where all partitions other than $i^*$ will be rejected by the detector $\Detector$.
\end{proof}

\section{Evaluation}%
\label{sec:eval}

{\Revised In this section, we present a series of experiments to evaluate the proposed framework.
Experiment 1 evaluates the correctness and computational cost of the GGM detailed in~\figref{fig:ggm_process}
for two target objects of varying complexity.
Experiment 2 validates the certified pose estimation algorithm (\algref{alg:dec_enc_framework})
and the error bounds of~\thmref{thm:estimation_error_bounds} in uncluttered environments.
For this, we design a GGM and further use it to train a pose estimator for each of two targets:
a slow-vehicle sign and a set of runway markings.
Experiment 3 compares the certified estimator against established geometric pose estimators
on both targets.
Finally, Experiment 4 assesses the performance of the pipeline for detection and pose estimation in cluttered environments (\algref{alg:pose_estimation_cluttered}).}

\subsection{Experimental Setup}

For the experiments, we implemented the GGM computational graph in Python using PyTorch.
Synthetic images were used in Experiments 1, 2 and 3.
The real images in Experiment 4 were captured by a SilkyEvCam event-based camera with known intrinsic parameters
($\ImWidthPxl = 640$, $\ImHeightPxl = 480$, $\FocalLength =\SI{8}{mm}$, pixel size $=\SI{15}{\micro\meter} \times \SI{15}{\micro\meter}$).
The ground truth of both the camera and the target object poses were measured using a Vicon motion capture system.

{\Revised
Two target objects were considered in the experiments.
The first target is a set of runway markings, scaled down to $\SI{0.3}{m} \times \SI{0.1}{m}$.
As shown in~\figref{fig:runway_dimensions}, the runway markings consist of a series of white stripes that indicate its threshold for landing, and characters and digits that indicate the runway's identifier.
These standard runway markings~\cite{faa2020advisory} are used to provide visual aids for aircrafts during landing.
The runway is modeled as a composition of polygons and the GGM is designed following the same process.
The second target is a slow-vehicle sign of size $\qty{14}{\text{in}}\times\qty{16}{\text{in}}$.
As shown in~\figref{fig:slow_sign_setup}, the sign is placed on a planar surface, and images are captured using the event-based camera at various relative poses between the camera and the sign.
The camera's asynchronous per-pixel brightness-change events
are accumulated into frames over windows of
$100\,\si{\milli\second}$.
Ten accumulated frames are averaged, binarized using Otsu's
histogram-based threshold~\cite{otsu1979threshold}, and filtered
to discard connected components below a minimum area.
The framework then processes each resulting binary image
independently; the GGM does not model individual event times.
}

\begin{figure}[t]
	\centering
	\includegraphics[width=0.6\columnwidth]{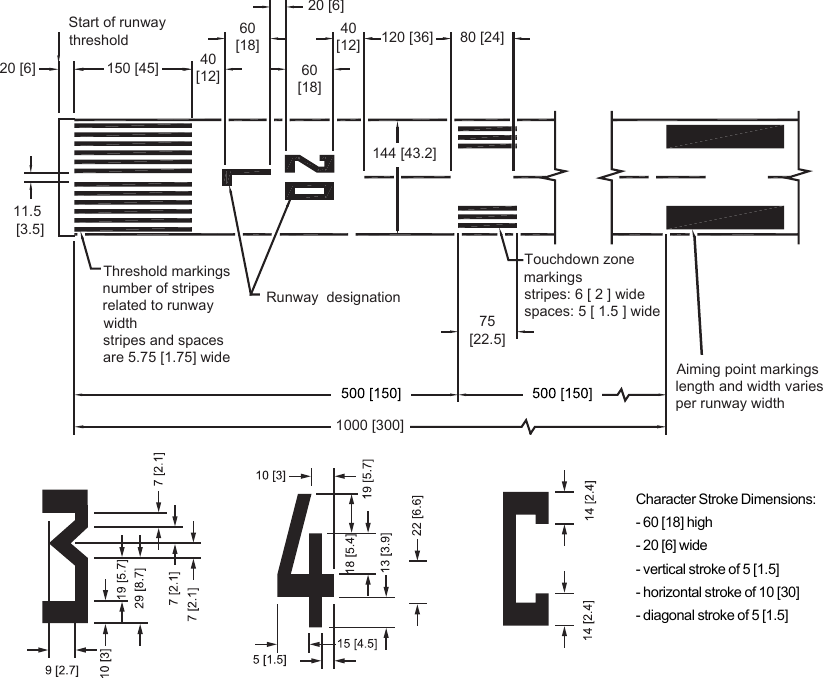}
	\caption{Standard dimensions of runway markings (in meters [feet])~\cite{faa2020advisory},
  comprising threshold markings and alphanumeric designations.
  }
	\label{fig:runway_dimensions}
\end{figure}

\begin{figure}[t]
\centering
\includegraphics[width=0.75\columnwidth]{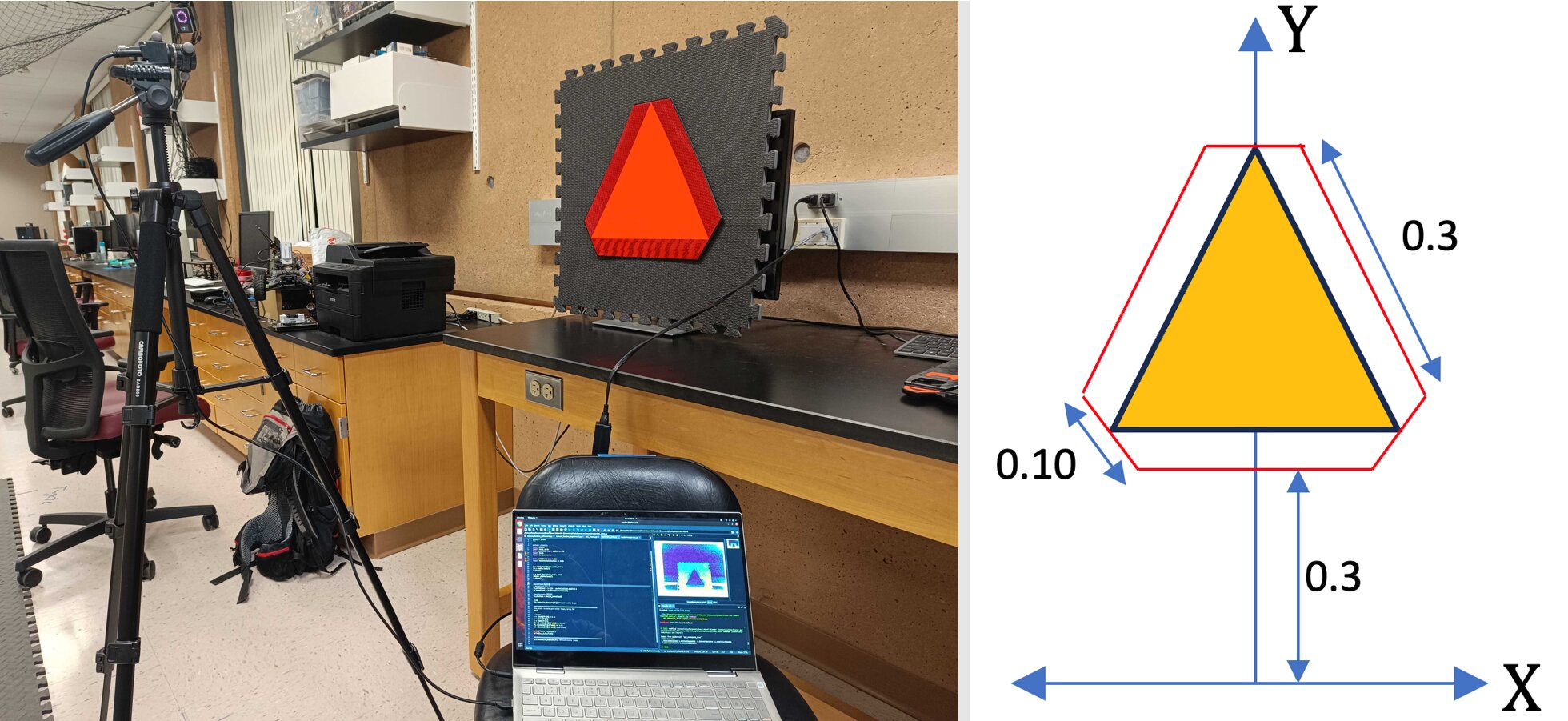}
\caption{Experimental setup for the slow-vehicle-ahead sign,
with relative pose $(x,y,z,\theta,\phi,\psi)$ between the camera and the sign.}
\label{fig:slow_sign_setup}
\end{figure}

For the GGM,
an XML file is used to specify the target object's geometry: vertices are specified as 2D coordinates (in the \TCF) using \texttt{<point>} elements, and polygons are defined as sets of vertices using \texttt{<polygon>} elements.
Note that 2D coordinates suffice for defining vertices since the target is a planar object on the XY plane.
The expression used to specify how to compose the polygons to form the target object is given in the \texttt{<composition>} element.
In turn, the XML is used to configure the parameters of the GGM computational graph in~\figref{fig:ggm_process}.
The result is a decoder that takes the camera pose as input and generates the image of the target object.

\subsection{Experiment 1: GGM Correctness and Computational Cost}

{\Revised
In this experiment, we construct a GGM for each of two target objects of differing complexity and assess its correctness and computational cost. Correctness is established for both targets by comparing the generated images against the analytic projection of the object's vertices, a computation independent of the generator. For the slow-vehicle sign, we additionally compare the generated images against ground-truth images captured by the event-based camera.
\figref{fig:ggm_generated_images}(a-b) shows examples of the images generated by each GGM at different camera poses.
\tabref{tab:ggm_sizes} reports the polygon and vertex counts of the two targets and the resulting number of nodes per layer of the corresponding GGMs as described earlier in~\figref{fig:ggm_process}.
The node count of $\Layer0$ to $\Layer2$ scales with the vertex count of the target,
while $\Layer3$ and $\Layer4$ depend on the pixel count alone and are therefore identical for both.
}

\begin{figure*}[t]
\centering%
\includegraphics[width=1.0\linewidth]{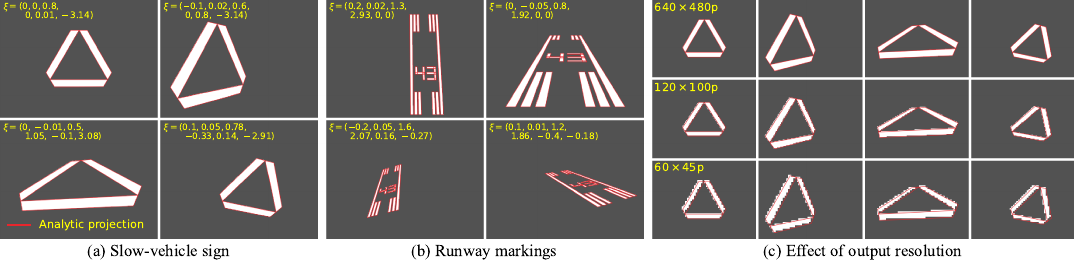}
\vspace{-7mm}
\caption{%
  (a,b) GGM outputs for the two targets at four poses,
	with analytic vertex projections; and (c) the slow-vehicle
	sign at three output resolutions.
	}%
    \vspace{-5mm}
\label{fig:ggm_generated_images}%
\end{figure*}

\begin{table}[t]
\Revised
\centering%
\setlength{\tabcolsep}{2pt}%
\caption{{\Revised Target complexity and GGM layer sizes for
$640\times480$ images.}}
\label{tab:ggm_sizes}
\resizebox{1.0\columnwidth}{!}{%
\begin{tabular}{p{2.6cm}m{1.5cm}m{1.0cm}rrrrrrr}
\toprule
{Target} & {Polygons} & {Vertices} & $\Layer0$ & $\Layer1$ & $\Layer2$ & $\Layer3$ & $\Layer4$ \\
\midrule
Slow-Vehicle Sign & 2  & 9  & 18	& \num{2764800}  & \num{6144000}  & \num{614400} & \num{307200}\\[0.1em]
Runway Markings   & 19 & 76 & 152	& \num{23347200} & \num{47308800} & \num{614400} & \num{307200}\\[0.1em]
\bottomrule
\end{tabular}%
}%
\end{table}

{\Revised
We assess correctness by comparing GGM-generated images against the analytic projection of the object's vertices.
\figref{fig:ggm_generated_images} shows examples of the generated images with their analytic projection superimposed.
The two computations are independent: the GGM tests every pixel against every edge and emits a binary image without computing a vertex coordinate, whereas the projection applies the pinhole model to the vertices alone.
The mean deviation between the GGM-generated images and their analytic projections is below one pixel for both targets and at every resolution due to the rounding differences in the rasterization process.
Their agreement therefore establishes that the generator places the target where the projection requires.
We also compare the images generated by the 640$\times$480 GGM of the slow-vehicle sign with the images captured by the actual camera at different poses as shown in~\figref{fig:camera_vs_ggm}.
We observe that the GGM-generated images closely resemble the camera images.
The maximum distance between each vertex in the GGM and its corresponding pixel in the generated image is approximately 3 pixels for the high-resolution image, and decreases to less than 1 pixel for the low-resolution image.
}

\begin{figure}[t]
	\centering
	\begin{tikzpicture}
		\node[anchor=south west,inner sep=0] (image) at (0,0) {\includegraphics[width=1.0\columnwidth]{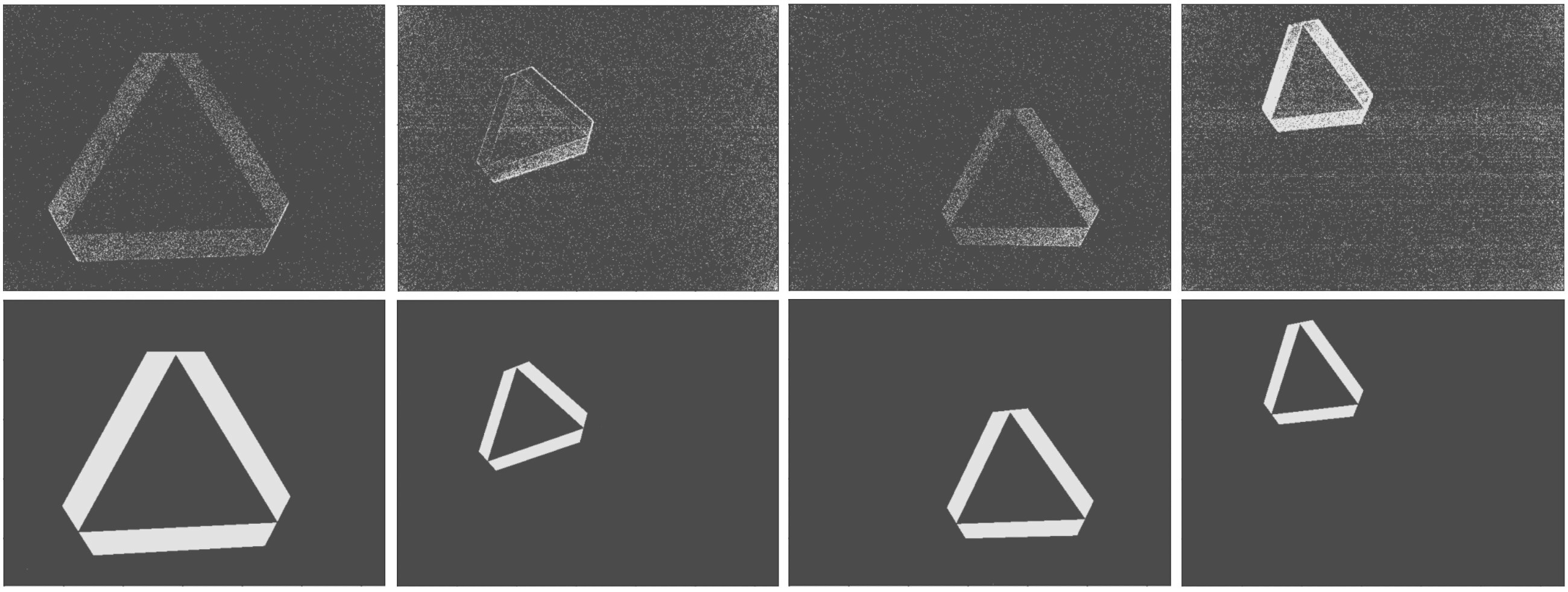}};
		\begin{scope}[x={(image.south east)},y={(image.north west)}]
		\node[yellow,font=\tiny,text width=81pt,anchor=north west,xscale=0.7, fill=darkgray] at (0.125-0.123,0.005) {$\!\!\Pose_1\;(-0.015,-0.435,+0.598,$\\$~\hfill-0.096,+0.086,-0.047)\!$};
		\node[yellow,font=\tiny,text width=81pt,anchor=north west,xscale=0.7, fill=darkgray] at (0.375-0.123,0.005) {$\!\!\Pose_2\;(-0.133,-0.433,+0.885,$\\$~\hfill-0.128,-0.196,-0.087)\!$};
		\node[yellow,font=\tiny,text width=81pt,anchor=north west,xscale=0.7, fill=darkgray] at (0.625-0.123,0.005) {$\!\!\Pose_3\;(-0.327,+0.015,+0.949,$\\$~\hfill-0.310,-0.153,-0.326)\!$};
		\node[yellow,font=\tiny,text width=81pt,anchor=north west,xscale=0.7, fill=darkgray] at (0.875-0.123,0.005) {$\!\!\Pose_4\;(-0.476,-0.383,+1.084,$\\$~\hfill+0.159,-0.267,-0.084)\!$};
		\end{scope}
	\end{tikzpicture}
  \caption{Camera images (top) and GGM outputs (bottom) at four poses.}
    \vspace{-7mm}
	\label{fig:camera_vs_ggm}
\end{figure}

{\Revised
The impact of image resolution is evaluated in~\figref{fig:ggm_generated_images}(c),
which shows the images generated by three different GGMs of the same target but with different output resolutions: high ($640\times480\si{p}$), medium ($120\times100\si{p}$), and low ($60\times45\si{p}$).
The measured inference time and memory footprint of the GGMs across those resolutions are summarized in~\tabref{tab:ggm_resolution}.
The time required to run a GGM grows sublinearly with the pixel count, while its memory footprint grows linearly.
Across objects, the inference time of the runway markings GGM is approximately
$8.5$ times longer than that of the slow-vehicle sign GGM,
which is consistent with the difference in their edge counts.
Note that this computational cost is incurred offline,
as the GGM is used to train and certify the encoder
and is not evaluated at deployment.
}

\begin{table}[t]
\Revised
\centering%
\setlength{\tabcolsep}{4pt}%
\caption{{\Revised GGM inference time and memory footprint
across output resolutions for both targets
(sample size $=5000$, $\sigma<0.5\si{\milli\second}$).}}
\label{tab:ggm_resolution}
\resizebox{1.0\columnwidth}{!}{%
\begin{tabular}{lrrrrrr}
\toprule
& & \multicolumn{2}{c}{Slow-vehicle sign} & \multicolumn{2}{c}{Runway markings} & \\
\cmidrule(lr){3-4}\cmidrule(lr){5-6}
{Resolution} & {Pixels} & {Time\! (\si{\milli\second})} & {Images/\si{\second}} & {Time (\si{\milli\second})} & {Images/\si{\second}} & {Mem\! (\si{\mega\byte})} \\
\midrule
$120  \!\times\! 100$   & \num{1.2e4} & 3.45 & 290 & 30.71 & 32.6 & 14.1   \\[0.1em]
$160  \!\times\! 120$   & \num{1.9e4} & 3.46 & 289 & 30.56 & 32.7 & 14.8   \\[0.1em]
$320  \!\times\! 240$   & \num{7.7e4} & 3.48 & 288 & 30.82 & 32.4 & 20.0  \\[0.1em]
$640  \!\times\! 480$   & \num{3.1e5} & 3.48 & 287 & 30.64 & 32.6 & 41.3  \\[0.1em]
$1280 \!\times\! 960$  & \num{1.2e6} & 6.62 & 151 & 58.72 & 17.0 & 125.5 \\[0.1em]
$1920 \!\times\! 1440$ & \num{2.8e6} & 13.22 & 76 & 114.43 & 8.7 & 277.0 \\[0.1em]
\bottomrule
\end{tabular}%
}%
	\vspace{0.3em}
	\begin{minipage}{\linewidth}
		\scriptsize
Memory usage depends only on image resolution and is
identical for both targets.
	\end{minipage}
\end{table}

\subsection{Experiment 2: Certifying the Error Bounds of the Encoder}

{\Revised
In this experiment, we validate the certified pose estimation framework described in~\secref{sec:pose_estimation_uncluttered} by applying the certified pose estimation procedure (\algref{alg:dec_enc_framework}) on two target objects: the slow-vehicle sign and the runway markings.
To this end, an encoder is trained for each target object using synthetic images generated by the corresponding GGM from the previous experiment.
Following the procedure outlined in \algref{alg:dec_enc_framework}, we then apply the certified pose estimation to each target object and compare the measured error against the certified bound.
}

{\Revised
Both encoders share the same underlying architecture,
comprising four convolutional layers of 16, 32, 64, and 128 channels
with a stride of 2,
followed by three fully connected layers of 512, 256, and 9 units
with rectified linear activations throughout.
Since \thmref{thm:estimation_error_bounds} requires a fixed Lipschitz bound on the encoder,
batch normalization layers are omitted from the architecture,
as their effective scaling depends on statistics accumulated during training
rather than on the layer weights alone.
Each encoder is trained using the Adam optimizer
with a learning rate of $10^{-4}$ and
a batch size of $32$ for $10^5$ iterations.
The encoder's Lipschitz constant is controlled during training
by a regularization term with coefficient $\lambda = 0.1$
and by clipping the encoder weights to $[-0.05, 0.05]$ after each update.
}

{\Revised
The training data is sampled over the operating regions defined by the pose spaces:
\begin{center}\resizebox{0.98\columnwidth}{!}{$\begin{aligned}
\PoseSpace_{\text{sign}} = \{\Pose \in \R^6 :\;
&(\PosX,\PosY,\PosZ) \in [-450,450] \!\times\! [-280,450] \!\times\! [350,1500],\\
&\left(\nicefrac{\Roll}{45}\right)^2 + \left(\nicefrac{\Pitch}{20}\right)^2 + \left(\nicefrac{\Yaw}{16}\right)^2 \leq 1 \},\\
\PoseSpace_{\text{runway}} = \{\Pose \in \R^6 :\;
&(\PosX,\PosY,\PosZ) \in [-600,600] \!\times\! [-600,600] \!\times\! [1200,4000],\\
&\left(\nicefrac{\Roll}{20}\right)^2 + \left(\nicefrac{\Pitch}{15}\right)^2 + \left(\nicefrac{\Yaw}{20}\right)^2 \leq 1 \},
\end{aligned}$}\end{center}
with the translation and rotation units in millimeters and degrees, respectively,
and the angles measured relative to a nominal viewing attitude.
Note that the translation bounds are those over which the respective targets remain fully within the camera frame.
Also, the rotation semi-axes are chosen to satisfy the $\delta$-identifiability of~\asmref{asm:injective_decoder} by excluding the symmetries of the target.
Each pose space is discretized with a grid step of $\BoundInput = 50\si{\milli\meter}$ in translation and $5\si{\degree}$ in rotation, and poses at which the target is not fully within the field of view are excluded. This yields $\num{1.58e6}$ poses for the sign and $\num{5.83e6}$ for the runway.
}%

{\Revised
We assess the $\delta$-identifiability
of~\asmref{asm:injective_decoder} over each grid, following the
procedure of~\secref{subsec:assessing_delta}.
Grouping the generated images by exact pixel equality yielded no
equivalence class of size greater than one on either target.
That is, no two fully visible sampled poses produce an identical
image.
Both decoders are therefore injective over their grids, giving
$\hat{\delta}=0$, and the bound reduces to
$\BoundInput(\Lipschitz_{\Dec}\Lipschitz_{\Enc}+1)
+\TrainingError$.

For example, the certified error bound for the slow-vehicle sign
is given by:
\begin{align*}
  \Error_{\Enc}
  &\leq
  \BoundInput
  \left(\Lipschitz_{\Dec}\Lipschitz_{\Enc}+1\right)
  +\TrainingError\\
  &=(50)(0.1651+1)+375.4
  =433.7\si{\milli\meter},
\end{align*}
for the translation component and, similarly,
$21.8\si{\degree}$ for the rotation component.

\tabref{tab:error_bounds} reports $\Lipschitz_{\Enc}$ for each
target and states the value $\Lipschitz_{\Dec}=1$ used in the
computations, thereby showing their effect on the certified
bounds.
The contribution associated specifically with the product
$\Lipschitz_{\Dec}\Lipschitz_{\Enc}$ is
$\BoundInput\Lipschitz_{\Dec}\Lipschitz_{\Enc}$.
This term is $8.26\si{\milli\meter}$ and
$0.83\si{\degree}$ for the slow-vehicle sign, and
$19.88\si{\milli\meter}$ and $1.99\si{\degree}$ for the runway
markings.
Including the remaining grid term $\BoundInput$, the complete
certification margin above $\TrainingError$ is
$58.26\si{\milli\meter}$ and $5.83\si{\degree}$ for the sign,
and $69.88\si{\milli\meter}$ and $6.99\si{\degree}$ for the
runway markings.
These margins correspond to $15.5\%$ and $3.6\%$ of the
translational maximum errors, and $36.5\%$ and $28.6\%$ of the
rotational maximum errors, respectively.

Target complexity affects the computational and certification
aspects of the framework differently.
As shown in~\tabref{tab:ggm_sizes}, the sizes of layers
$\Layer0$--$\Layer2$ increase with the number of target
vertices, affecting the offline construction and evaluation
cost of the GGM.
The bound in~\thmref{thm:estimation_error_bounds}, however,
contains no explicit polygon- or vertex-count term.
Target geometry can instead affect the target-specific
quantities $\delta$, $\Lipschitz_{\Dec}$,
$\Lipschitz_{\Enc}$, and $\TrainingError$.
These quantities are therefore recomputed for each target,
rather than inferred from its number of vertices.
In the present experiments, the two target descriptions contain
$9$ and $76$ vertices, while their corresponding
Lipschitz-dependent translational terms are
$8.26\si{\milli\meter}$ and $19.88\si{\milli\meter}$.
Thus, the measured certification term does not scale directly
with the vertex count.
\figref{fig:randomgrid} shows the error distributions over
$\num{5e5}$ poses sampled from each grid.
All lie within the certified bounds, and the largest sampled
error falls below $\TrainingError$, as a subset of the grid
should.
}%

\begin{table}[t]
\Revised
\centering%
\setlength{\tabcolsep}{4pt}%
\caption{{\Revised Certified error bounds for the slow-vehicle sign and runway markings ($\BoundInput_{\text{translation}} = 50\si{\milli\meter}$, $\BoundInput_{\text{rotation}} = 5\si{\degree}$, $\hat{\delta} = 0$, $\Lipschitz_{\Dec} = 1$).}}
\label{tab:error_bounds}
\resizebox{1.0\columnwidth}{!}{%
\begin{tabular}{lrrrrrr}
\toprule
& & {Encoder} & \multicolumn{2}{c}{Translation (\si{\milli\meter})} & \multicolumn{2}{c}{Rotation (\si{\degree})} \\
\cmidrule(lr){4-5}\cmidrule(lr){6-7}
{Target} & {Grid poses} & {$\Lipschitz_{\Enc}$} & {Max err. $\TrainingError$} & {Certified bound} & {Max err. $\TrainingError$} & {Certified bound} \\
\midrule
Slow-vehicle sign & \num{1.58e6} & 0.1651 & 375.42  & {433.68}  & 15.98 & {21.81} \\[0.1em]
Runway markings    & \num{5.83e6} & 0.3976 & 1963.32 & {2033.20} & 24.46 & {31.45} \\[0.1em]
\bottomrule
\end{tabular}%
}%
\end{table}

\begin{figure}[t]
\centering
\includegraphics[width=1.0\columnwidth]{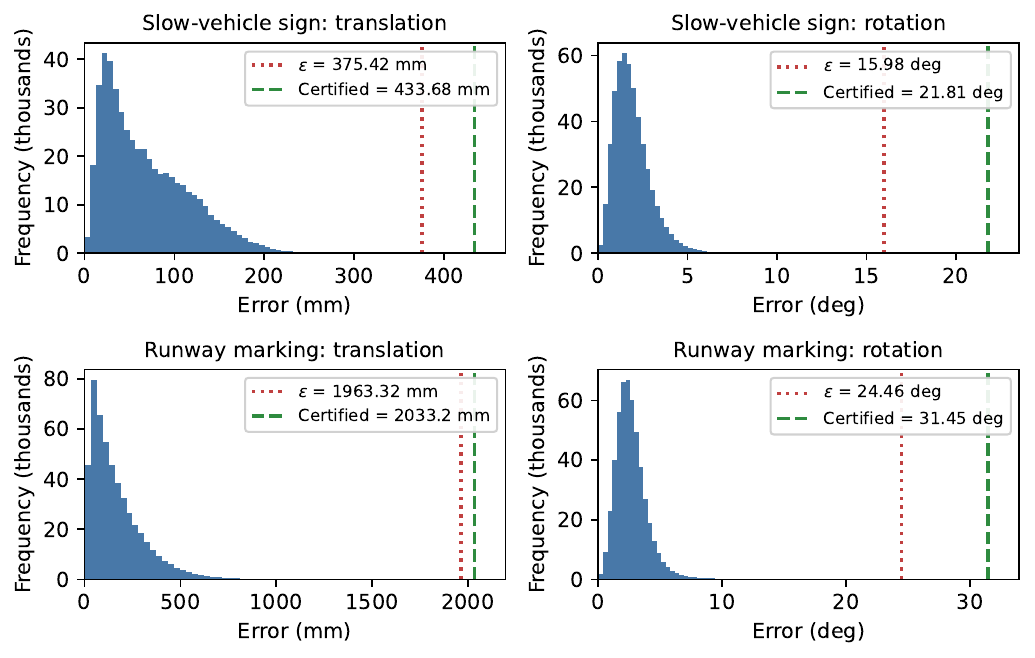}%
\vspace{-0.5em}%
\caption{{\Revised Test error distributions for the two targets, over \num{5e5} poses sampled from each certification grid. $\TrainingError$ is the largest error over the full grid.}
}%
\vspace{-5mm}
\label{fig:randomgrid}
\end{figure}

\subsection{Experiment 3: Comparison with Classical Pose Estimators}

{\Revised
The goal of this experiment is to evaluate the certified encoders against
well-established solvers for recovering pose from a known planar target.
Specifically, we consider five methods:
EPnP~\cite{lepetit2009epnp}, a general perspective-n-point solver;
P3P~\cite{gao2003p3p}, which recovers pose from three correspondences and uses a fourth to select among the resulting solutions;
IPPE~\cite{collins2014ippe}, which is specific to planar targets;
IPPE combined with RANSAC~\cite{fischler1981ransac}; and
decomposition of the homography induced by the target plane.
Each method receives a monochromatic image of $640 \times 480$ pixels
and returns an estimate of the camera pose relative to the target.
The test set is drawn from the pose spaces described earlier in Experiment 2,
where $\num{20000}$ poses are sampled at random
from the $\num{1576767}$ fully visible poses of the sign's grid,
and $\num{10000}$ from the $\num{5741246}$ of the runway's.
For each sampled pose, the image of the target object is generated using its GGM.
The pose estimates returned by each method are then compared to the ground-truth poses
to evaluate their accuracy.

The five methods were implemented in Python using OpenCV.
Natively, the five baseline methods expect as input
a set of correspondences between points on the target and their projections in the image,
together with the camera intrinsics,
and return the pose of the target in the camera frame.
In contrast, the encoder operates directly on the input image
and does not require the intermediate correspondence set.
Consequently,
we implement an intermediate layer that translates an input image into
the candidate correspondence set required by the baselines.

For the slow-vehicle sign, the layer extracts the convex hull of the foreground,
simplifies it to six vertices, refines them to subpixel accuracy,
and matches them to the known hexagon.
The obtained labelling is determined only up to a cyclic permutation and a reflection,
so all twelve labellings are solved,
and the candidate lying in the operating region with the lowest reprojection error is retained.
If no candidate lies in the operating region,
it is recorded as a failure rather than as a large error,
as such solutions are numerically degenerate.

Unlike the slow-vehicle sign,
the runway markings render as $19$ disjoint rectangles from which no single contour yields the correspondence set.
Their convex hull is the outer rectangle formed by the two border lines,
which is symmetric under a $180\si{\degree}$ rotation about each of the three axes.
A pose recovered from its four corners is therefore ambiguous under three distinct rotations,
whereas the digits and corner stripes (\ie the inner markings)
that resolve the ambiguity lie in the interior.
That is, the hull discards the very structure that makes the pose identifiable.

We therefore consider two strategies for constructing the runway correspondence set.
The first, \textsc{hull4}, recovers the four ordered corners of the hull
and inherits the ambiguity above.
The second, \textsc{components}, uses strictly more information.
First, the image is segmented into connected components,
each corresponding to one of the parts of the marking.
Each of the four \textsc{hull4} labellings is then taken in turn as a hypothesis,
yielding a coarse pose.
From that pose, the 19 model centroids are projected into the image and matched
to the observed components by nearest neighbour.
The pose is then re-solved on the matched set and scored by reprojection over all matched points.
Since the digits and corner stripes (\ie the inner markings) are asymmetric,
they align with their projections only under the correct hypothesis.
Finally, the hypothesis with the lowest reprojection error over all matched points
is selected and passed as the final pose estimate.
} 

{\Revised

\tabref{tab:baseline_comparison} reports the results of the comparison for both target objects.
A resolved pose refers to an estimate that falls within the operating region and is therefore included in the computation of the error statistics.
That is, the error columns are computed over the resolved poses alone.
Consequently, a method that resolves few poses may report a low error on the subset it handles, and the two must be read together as complementary indicators of performance.

In the case of the slow-vehicle sign,
the encoder returns an estimate at every one of the $\num{20000}$ sampled poses,
with a maximum rotational error of $11.54\si{\degree}$.
The baselines resolve between $70.3\%$ and $85.9\%$ of the same poses,
and on those they resolve their rotational error reaches between $66.48$ and $81.99\si{\degree}$,
an error of a magnitude that indicates a misidentified orientation rather than a degraded estimate.
The slow-vehicle sign is an equilateral triangle with truncated corners
and is therefore invariant under a $120\si{\degree}$ rotation,
which renders three of the twelve candidate labellings indistinguishable
and leaves the solver no basis for choosing among them.
The encoder does not face this choice,
as the operating region excludes the symmetric orientations by construction.
On the resolved poses, the baselines are generally more accurate in median translation,
reaching $3.7\si{\milli\meter}$ for IPPE with RANSAC against the encoder's $57.3\si{\milli\meter}$.
Homography fails to resolve the majority of poses
as it recovers translation only up to an unknown scale.

For the runway markings, the encoder resolves every pose,
with a maximum rotational error of $17.83\si{\degree}$,
outperforming that of every baseline.
The baselines resolve between $49.2\%$ and $93.9\%$ of the $\num{10000}$ sampled poses.
Their rotational error is markedly lower here than on the sign,
between $21.50$ and $31.88\si{\degree}$.
Homography decomposition resolves no pose under either strategy.
The additional information available to \textsc{components}
does not translate into a uniform improvement over \textsc{hull4}.
On one hand,
it raises the resolved fraction for IPPE from $68.5\%$ to $93.9\%$
and for IPPE with RANSAC from $72.2\%$ to $89.4\%$.
On the other hand,
it lowers the resolved fraction for EPnP from $76.3\%$ to $49.2\%$
and for P3P from $72.2\%$ to $54.7\%$.
As the target recedes, adjacent rectangles merge into single connected components,
and a component spanning several parts cannot be matched to any one model centroid:
a median of $12$ components is observed against the $19$ of the model.
The interior structure is therefore present in the image but not always recoverable from it.
} 

\begin{table}[t]
\Revised
\centering%
\setlength{\tabcolsep}{4pt}%
\caption{{\Revised Pose estimation error and certification results for the slow-vehicle sign and runway markings.}}
\label{tab:baseline_comparison}
\resizebox{1.0\columnwidth}{!}{%
\begin{tabular}{lrrrrrrr}
\toprule
& & \multicolumn{3}{c}{Translation error (\si{\milli\meter})} & \multicolumn{3}{c}{Rotation error (\si{\degree})} \\
\cmidrule(lr){3-5}\cmidrule(lr){6-8}
{Method} & {Resolved} & {Med} & {Max} & {Certified} & {Med} & {Max} & {Certified} \\
\midrule
\multicolumn{8}{l}{\emph{Slow-vehicle sign: 2 parts, 9 edges, convex}} \\[0.2em]
Proposed (certified) & \textbf{100.0\%} & 57.3  & 339.3  & \textbf{$+15.5\%$} & 1.72 & \textbf{11.54} & \textbf{$+36.5\%$} \\[0.1em]
EPnP                 & 81.8\% & 4.9   & 564.3  & -- & 1.39 & 81.99 & -- \\[0.1em]
P3P                  & 70.3\% & 10.5  & 264.8  & -- & 2.75 & 74.49 & -- \\[0.1em]
IPPE                 & 85.9\% & 3.8   & \textbf{79.4} & -- & 0.73 & 78.95 & -- \\[0.1em]
IPPE + RANSAC        & 83.7\% & \textbf{3.7} & 119.4 & -- & \textbf{0.71} & 66.48 & -- \\[0.1em]
Homography   & 11.4\% & 532.9 & 1136.2 & -- & 0.97 & 21.43 & -- \\[0.4em]
\midrule
\multicolumn{8}{l}{\emph{Runway markings: 19 parts, 76 edges, non-convex}} \\[0.2em]
Proposed (certified) & \textbf{100.0\%} & 127.7 & 1613.3 & \textbf{$+3.6\%$} & 2.51 & \textbf{17.83} & \textbf{$+28.6\%$} \\[0.3em]
\multicolumn{8}{l}{\textsc{hull4} strategy} \\[0.1em]
\hspace{1em}EPnP          & 76.3\% & 123.3 & 1735.3 & -- & 3.10 & 22.89 & -- \\[0.1em]
\hspace{1em}P3P           & 72.2\% & \textbf{26.0} & \textbf{581.6} & -- & 0.87 & 27.73 & -- \\[0.1em]
\hspace{1em}IPPE          & 68.5\% & 79.4  & 712.0  & -- & 1.06 & 21.50 & -- \\[0.1em]
\hspace{1em}IPPE + RANSAC & 72.2\% & 26.0  & 581.6  & -- & 0.87 & 27.73 & -- \\[0.1em]
\hspace{1em}Homography & 0.0\% & -- & -- & -- & -- & -- & -- \\[0.3em]
\multicolumn{8}{l}{\textsc{components} strategy} \\[0.1em]
\hspace{1em}EPnP          & 49.2\% & 275.6 & 2401.1 & -- & 6.10 & 30.01 & -- \\[0.1em]
\hspace{1em}P3P           & 54.7\% & 30.7  & 1418.8 & -- & 1.54 & 31.88 & -- \\[0.1em]
\hspace{1em}IPPE          & 93.9\% & 254.7 & 2115.0 & -- & 3.24 & 26.03 & -- \\[0.1em]
\hspace{1em}IPPE + RANSAC & 89.4\% & 22.9  & 1418.8 & -- & \textbf{0.87} & 23.05 & -- \\[0.1em]
\hspace{1em}Homography & 0.0\% & -- & -- & -- & -- & -- & -- \\[0.1em]
\bottomrule
\end{tabular}%
}%
	\vspace{0.3em}
	\begin{minipage}{\linewidth}
		\scriptsize
		Median and maximum errors are computed over resolved
		test poses. Unresolved poses are counted as failures in
		the resolved rate. The certified margin denotes the
		percentage by which the certified bound exceeds the
		maximum error over the certification grid.
	\end{minipage}
  \vspace{-15pt}
\end{table}

{\Revised
The results show that the encoder returns a pose estimate for $100\%$
of the sampled poses for both targets,
which is consistent with Theorem~1 that
certifies a bounded estimation error for every
pose within the certified operating region, including poses not
used to train the encoder.
In contrast, the four PnP-based
pipelines resolve between $70.3\%$ and $85.9\%$ of the sign
poses, while homography decomposition resolves $11.4\%$. For the
runway markings, the PnP-based pipelines resolve between
$49.2\%$ and $93.9\%$ of the poses, depending on the
correspondence strategy. The classical pipelines generally
achieve lower median translation errors on the poses they
resolve, whereas the encoder has the lowest maximum rotational
error on both targets.
} 

\vspace{-0.8em}
\subsection{Experiment 4: Pose Estimation in Cluttered Environments}

We evaluate the detection and pose-estimation pipeline in
cluttered environments using the runway target. The pose
region is defined by
$\Pose_y \in [\SI{0.8}{\meter}, \SI{1.0}{\meter}]$,
$\Pose_z \in [\SI{1.6}{\meter}, \SI{1.8}{\meter}]$, and
$\Pose_\theta \in [0.5, 0.7]$ rad, with $\Pose_x = 0$.
We partition this region into $\PartitionCt = 27$ hypercubes
with spacing of 0.1
and construct a certified detector
${\Detector}_i$ and spatial filter $\SpatialFilter_i$ for each
cell $i = 1, \ldots, 27$.
The camera recorded 1,320 frames at 25 fps.
\figref{fig:pipeline_keyframes} shows 12 representative frames
and their filtered outputs, each reporting the estimated pose
$(\Pose_\theta,\Pose_y,\Pose_z)$ and error relative to ground
truth. The results show effective clutter removal while
preserving the target structure required for pose estimation.

\begin{figure*}[ht]
	\setlength{\FigWidth}{0.155\linewidth}%
	\setlength{\FigYGap}{-0.2em}%
	\centering
	\setlength{\tabcolsep}{1pt}%

	\resizebox{1.0\textwidth}{!}{%
	\begin{tabular}{@{}cccccc@{}}

		\includegraphics[width=\FigWidth]{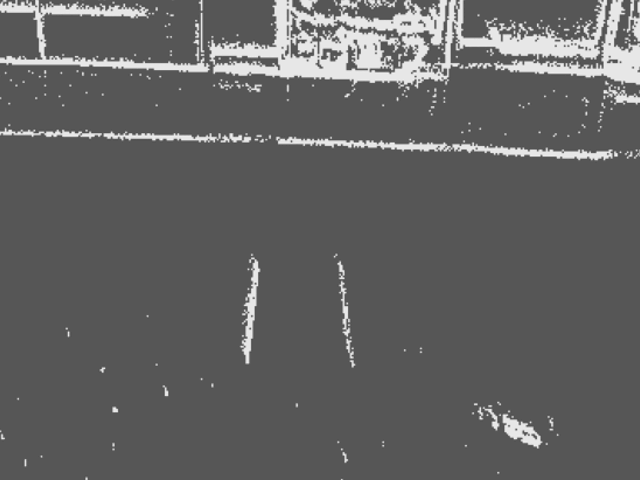} &
		\includegraphics[width=\FigWidth]{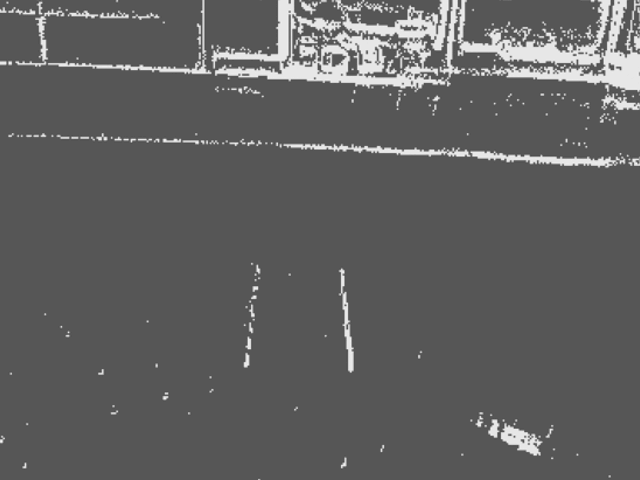} &
		\includegraphics[width=\FigWidth]{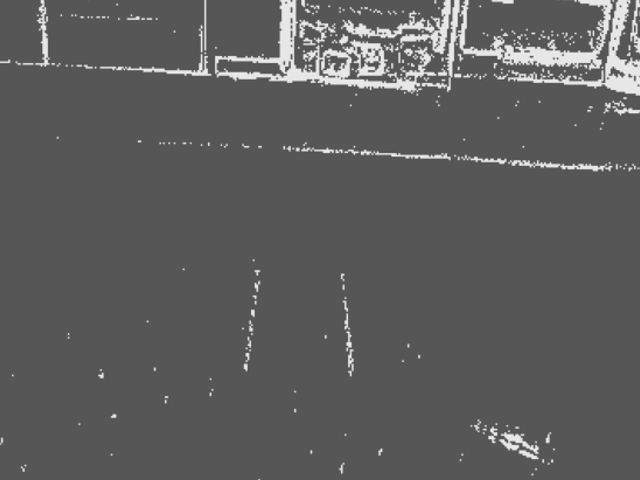} &
		\includegraphics[width=\FigWidth]{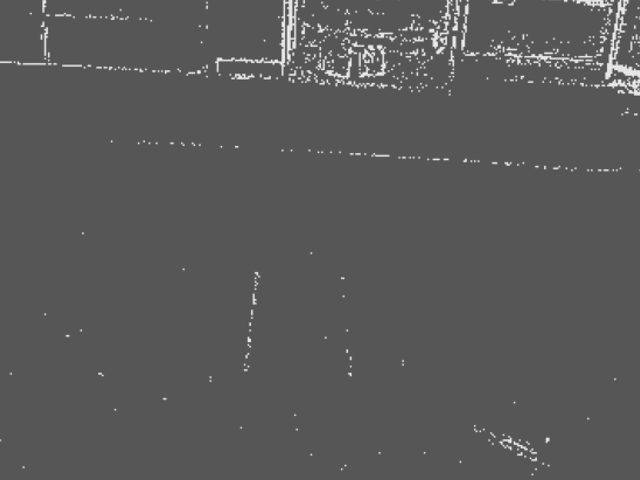} &
		\includegraphics[width=\FigWidth]{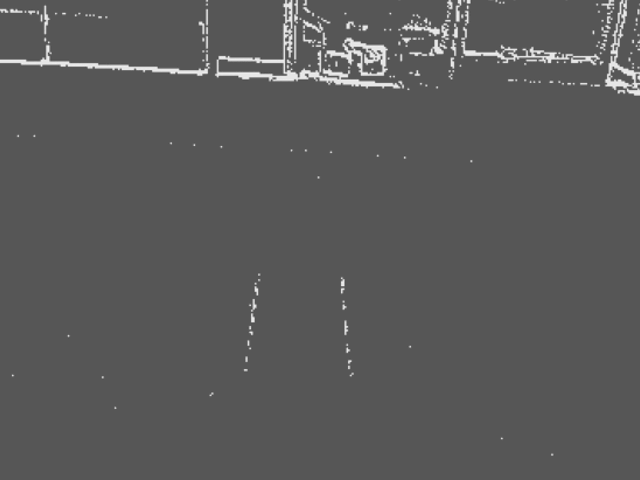} &
		\includegraphics[width=\FigWidth]{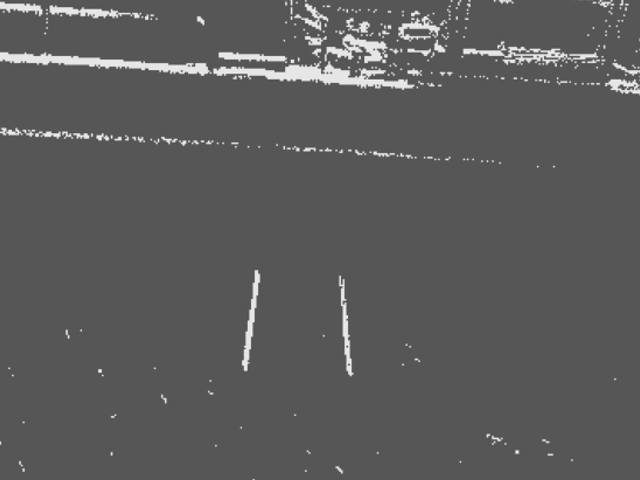}
		\\[\FigYGap]

		\begin{tikzpicture}[
			text width=1.0\FigWidth,
			outer sep=0pt,
			inner sep=0pt]
			\node[anchor=south west] (img) {
				\includegraphics[width=\FigWidth]
				{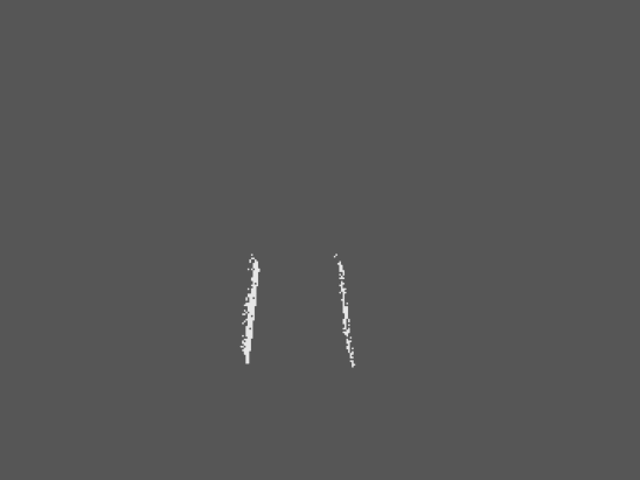}};
			\node[anchor=south west,pOverlay]
				at (img.south west) {t = 1};
			\node[anchor=north west,pOverlay]
				at (img.north west) {
				$\Pose$ = (0, 1.73, 0.55,\\ \qquad\, 0, 0.49, 0) \\
				[0.2em] Err = 0.27};
		\end{tikzpicture} &
		\begin{tikzpicture}[
			text width=1.0\FigWidth,
			outer sep=0pt,
			inner sep=0pt]
			\node[anchor=south west] (img) {
				\includegraphics[width=\FigWidth]
				{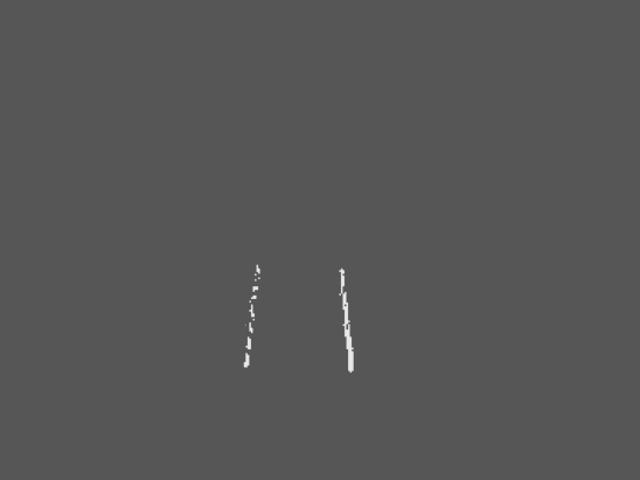}};
			\node[anchor=south west,pOverlay]
				at (img.south west) {t = 2};
			\node[anchor=north west,pOverlay]
				at (img.north west) {
				$\Pose$ = (0, 1.56, 0.78,\\ \qquad\, 0, 0.58, 0) \\
				[0.2em] Err = 0.26};
		\end{tikzpicture} &
		\begin{tikzpicture}[
			text width=1.0\FigWidth,
			outer sep=0pt,
			inner sep=0pt]
			\node[anchor=south west] (img) {
				\includegraphics[width=\FigWidth]
				{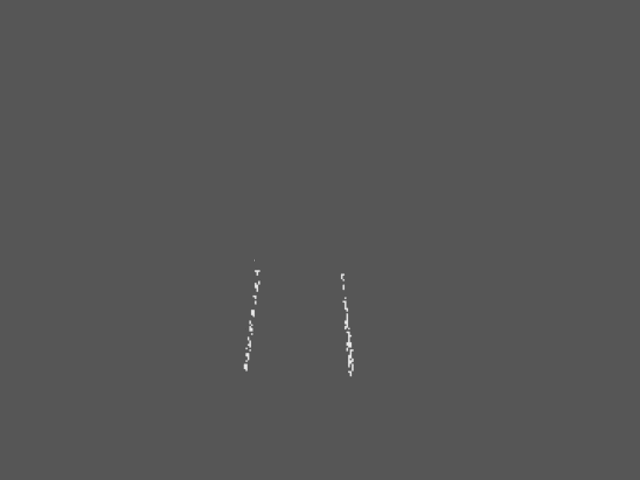}};
			\node[anchor=south west,pOverlay]
				at (img.south west) {t = 3};
			\node[anchor=north west,pOverlay]
				at (img.north west) {
				$\Pose$ = (0, 1.53, 0.90,\\ \qquad\, 0, 0.65, 0) \\
				[0.2em] Err = 0.24};
		\end{tikzpicture} &
		\begin{tikzpicture}[
			text width=1.0\FigWidth,
			outer sep=0pt,
			inner sep=0pt]
			\node[anchor=south west] (img) {
				\includegraphics[width=\FigWidth]
				{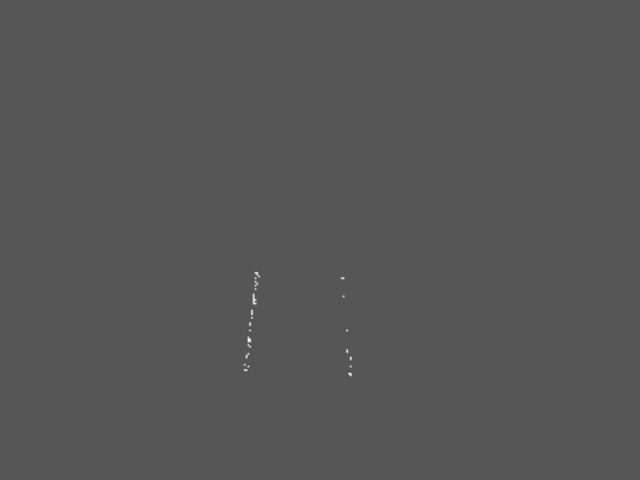}};
			\node[anchor=south west,pOverlay]
				at (img.south west) {t = 4};
			\node[anchor=north west,pOverlay]
				at (img.north west) {
				$\Pose$ = (0, 1.58, 0.79,\\ \qquad\, 0, 0.59, 0) \\
				[0.2em] Err = 0.21};
		\end{tikzpicture} &
		\begin{tikzpicture}[
			text width=1.0\FigWidth,
			outer sep=0pt,
			inner sep=0pt]
			\node[anchor=south west] (img) {
				\includegraphics[width=\FigWidth]
				{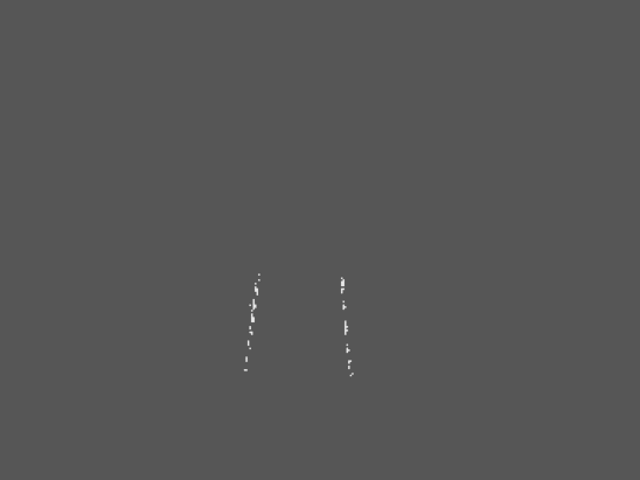}};
			\node[anchor=south west,pOverlay]
				at (img.south west) {t = 5};
			\node[anchor=north west,pOverlay]
				at (img.north west) {
				$\Pose$ = (0, 1.41, 1.01,\\ \qquad\, 0, 0.75, 0) \\
				[0.2em] Err = 0.16};
		\end{tikzpicture} &
		\begin{tikzpicture}[
			text width=1.0\FigWidth,
			outer sep=0pt,
			inner sep=0pt]
			\node[anchor=south west] (img) {
				\includegraphics[width=\FigWidth]
				{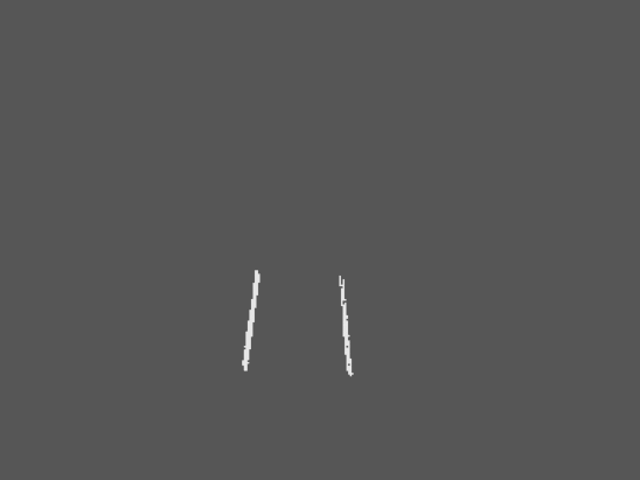}};
			\node[anchor=south west,pOverlay]
				at (img.south west) {t = 6};
			\node[anchor=north west,pOverlay]
				at (img.north west) {
				$\Pose$ = (0, 1.56, 0.78,\\ \qquad\, 0, 0.66, 0) \\
				[0.2em] Err = 0.26};
		\end{tikzpicture}
		\\[\FigYGap]

		\includegraphics[width=\FigWidth]{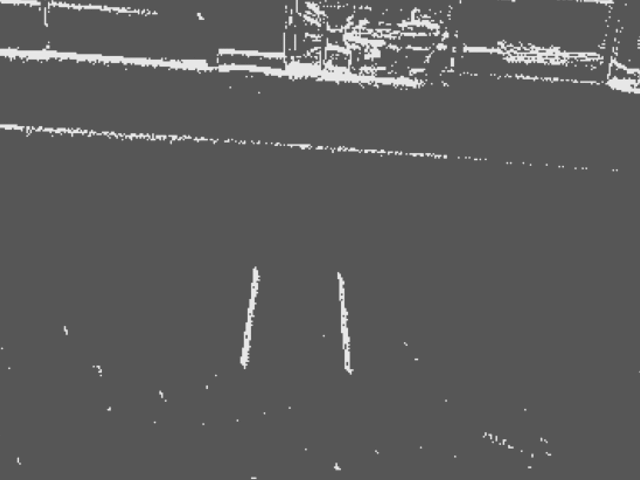} &
		\includegraphics[width=\FigWidth]{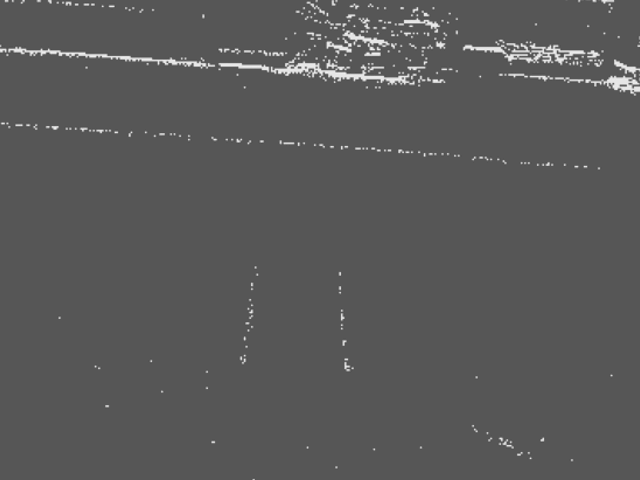} &
		\includegraphics[width=\FigWidth]{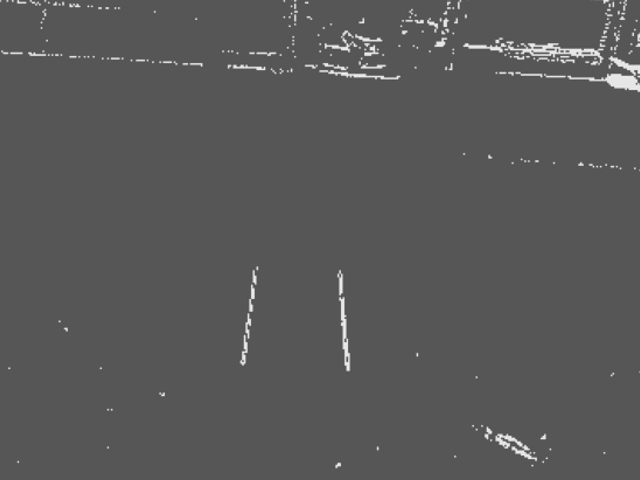} &
		\includegraphics[width=\FigWidth]{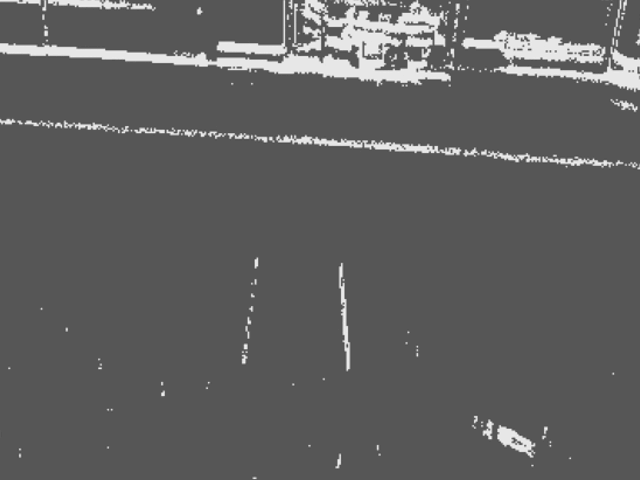} &
		\includegraphics[width=\FigWidth]{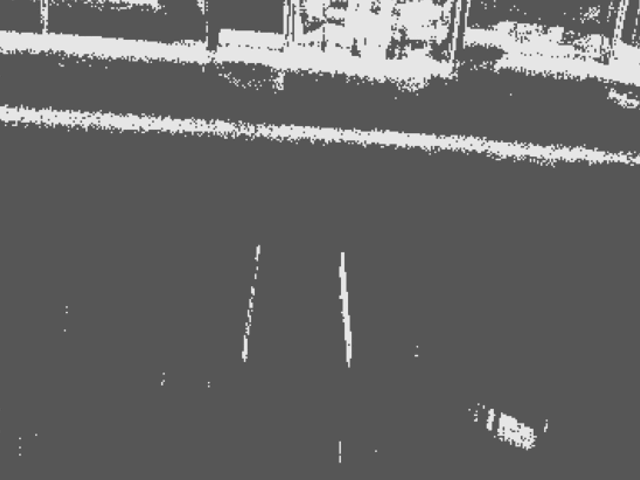} &
		\includegraphics[width=\FigWidth]{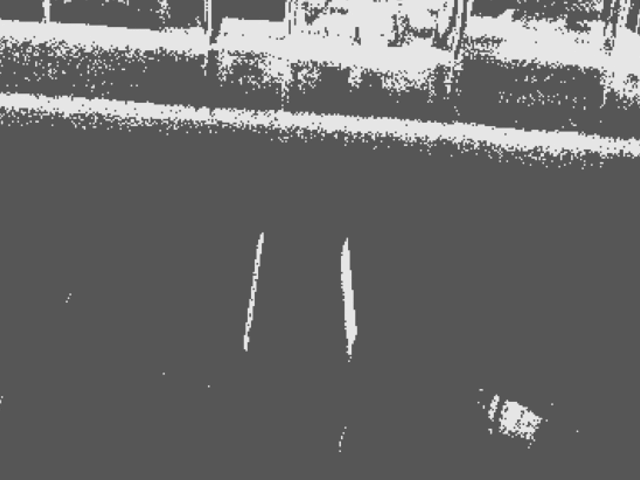}
		\\[\FigYGap]

		\begin{tikzpicture}[
			text width=1.0\FigWidth,
			outer sep=0pt,
			inner sep=0pt]
			\node[anchor=south west] (img) {
				\includegraphics[width=\FigWidth]
				{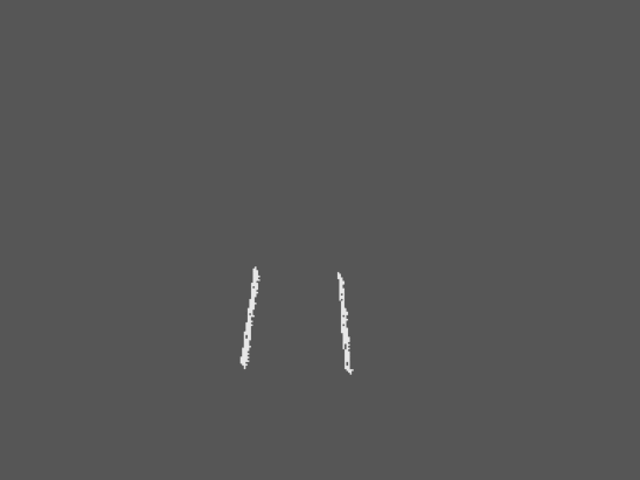}};
			\node[anchor=south west,pOverlay]
				at (img.south west) {t = 7};
			\node[anchor=north west,pOverlay]
				at (img.north west) {
				$\Pose$ = (0, 1.54, 0.88,\\ \qquad\, 0, 0.72, 0) \\
				[0.2em] Err = 0.13};
		\end{tikzpicture} &
		\begin{tikzpicture}[
			text width=1.0\FigWidth,
			outer sep=0pt,
			inner sep=0pt]
			\node[anchor=south west] (img) {
				\includegraphics[width=\FigWidth]
				{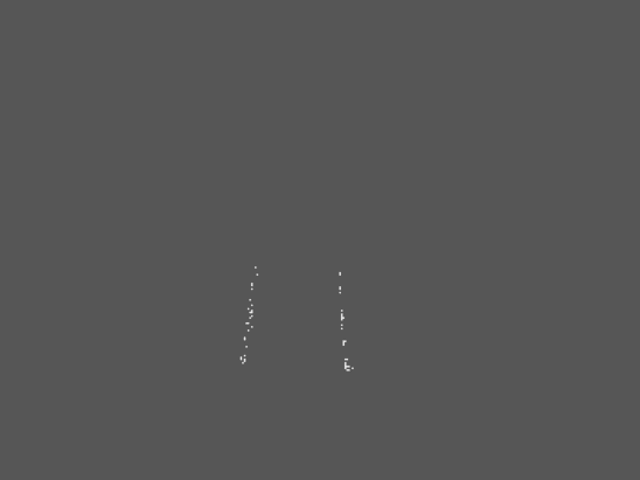}};
			\node[anchor=south west,pOverlay]
				at (img.south west) {t = 8};
			\node[anchor=north west,pOverlay]
				at (img.north west) {
				$\Pose$ = (0, 1.44, 0.99,\\ \qquad\, 0, 0.72, 0) \\
				[0.2em] Err = 0.27};
		\end{tikzpicture} &
		\begin{tikzpicture}[
			text width=1.0\FigWidth,
			outer sep=0pt,
			inner sep=0pt]
			\node[anchor=south west] (img) {
				\includegraphics[width=\FigWidth]
				{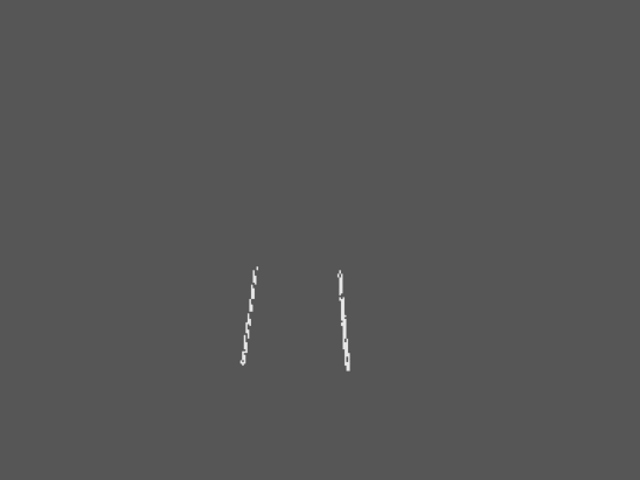}};
			\node[anchor=south west,pOverlay]
				at (img.south west) {t = 9};
			\node[anchor=north west,pOverlay]
				at (img.north west) {
				$\Pose$ = (0, 1.65, 0.69,\\ \qquad\, 0, 0.51, 0) \\
				[0.2em] Err = 0.29};
		\end{tikzpicture} &
		\begin{tikzpicture}[
			text width=1.0\FigWidth,
			outer sep=0pt,
			inner sep=0pt]
			\node[anchor=south west] (img) {
				\includegraphics[width=\FigWidth]
				{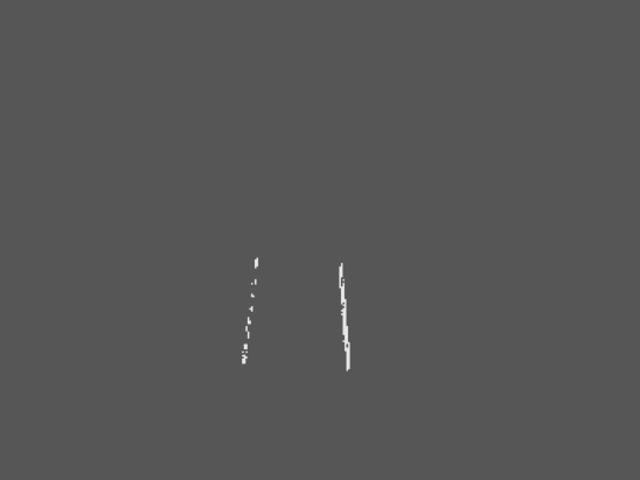}};
			\node[anchor=south west,pOverlay]
				at (img.south west) {t = 10};
			\node[anchor=north west,pOverlay]
				at (img.north west) {
				$\Pose$ = (0, 1.50, 0.89,\\ \qquad\, 0, 0.66, 0) \\
				[0.2em] Err = 0.15};
		\end{tikzpicture} &
		\begin{tikzpicture}[
			text width=1.0\FigWidth,
			outer sep=0pt,
			inner sep=0pt]
			\node[anchor=south west] (img) {
				\includegraphics[width=\FigWidth]
				{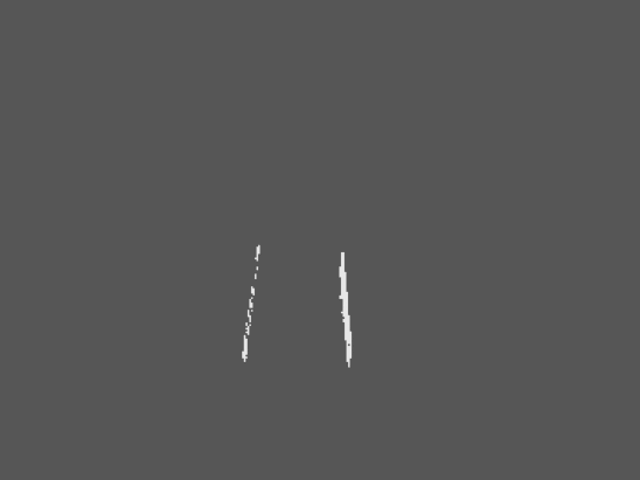}};
			\node[anchor=south west,pOverlay]
				at (img.south west) {t = 11};
			\node[anchor=north west,pOverlay]
				at (img.north west) {
				$\Pose$ = (0, 1.53, 0.90,\\ \qquad\, 0, 0.65, 0) \\
				[0.2em] Err = 0.24};
		\end{tikzpicture} &
		\begin{tikzpicture}[
			text width=1.0\FigWidth,
			outer sep=0pt,
			inner sep=0pt]
			\node[anchor=south west] (img) {
				\includegraphics[width=\FigWidth]
				{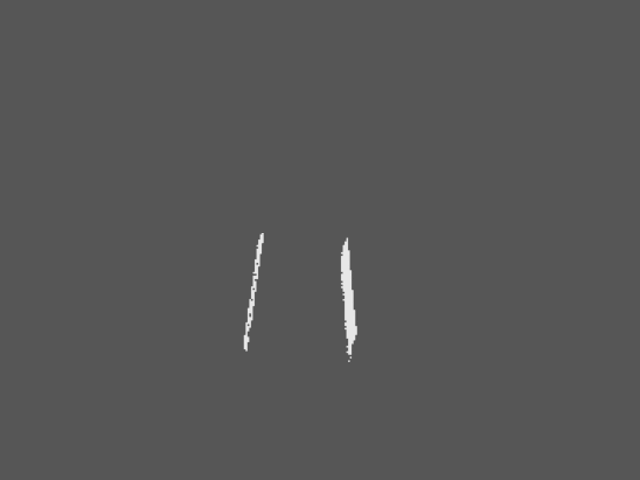}};
			\node[anchor=south west,pOverlay]
				at (img.south west) {t = 12};
			\node[anchor=north west,pOverlay]
				at (img.north west) {
				$\Pose$ = (0, 1.58, 0.79,\\ \qquad\, 0, 0.59, 0) \\
				[0.2em] Err = 0.21};
		\end{tikzpicture}

	\end{tabular}%
	}
\vspace{-0.8em}
  \caption{%
	Raw camera images $\Image$ (top) and filtered outputs
	$\tilde{\Image}$ (bottom) at twelve representative
	timestamps. Each filtered image reports the estimated pose
	$\Pose$ and corresponding error, demonstrating clutter
	removal and pose estimation within the certified bounds.
}
	\label{fig:pipeline_keyframes}
  \vspace{-1em}
\end{figure*}

{\Revised
To assess the computational performance of the encoder stage,
we measure the inference latency of the two encoders from
Experiment~2 under the same pose-space partitioning.
Spatial-filter construction and certified-detector design are
performed offline and therefore do not contribute to encoder
inference latency.
During online operation, the $\PartitionCt=27$ partition cells
associated with each frame can be processed independently by
the encoder.
\tabref{tab:latency} reports the encoder latency for processing
all $\PartitionCt$ cells, either sequentially or as a single
batch.
Batched evaluation requires $4.18\si{\milli\second}$ for the
slow-vehicle sign and $4.24\si{\milli\second}$ for the runway
markings, compared with $17.64$ and
$18.06\si{\milli\second}$, respectively, for sequential
evaluation.
Thus, batching the $27$ independent evaluations increases
latency by only a factor of approximately six relative to a
single evaluation, rather than by a factor of $27$.
The resulting rates are $239\si{\hertz}$ for batched and
$57\si{\hertz}$ for sequential evaluation, both exceeding the
camera frame rate of $25\si{fps}$.
}%

\begin{table}[t]
\Revised
\centering%
\setlength{\tabcolsep}{4pt}%
\caption{{\Revised
Encoder inference latency under uncluttered, batched-cluttered,
and sequential-cluttered conditions.}}
\label{tab:latency}
\resizebox{1.0\columnwidth}{!}{%
\begin{tabular}{lrrrrrrrrr}
\toprule
& \multicolumn{3}{c}{Uncluttered} & \multicolumn{3}{c}{Cluttered, batched} & \multicolumn{3}{c}{Cluttered, sequential} \\
\cmidrule(lr){2-4}\cmidrule(lr){5-7}\cmidrule(lr){8-10}
{Target} & {Mean} & {99th} & {Rate} & {Mean} & {99th} & {Rate} & {Mean} & {99th} & {Rate} \\[-0.1em]
& (\si{\milli\second}) & (\si{\milli\second}) & (\si{\hertz}) & (\si{\milli\second}) & (\si{\milli\second}) & (\si{\hertz}) & (\si{\milli\second}) & (\si{\milli\second}) & (\si{\hertz}) \\
\midrule
Slow-vehicle sign & 0.69 & 0.78 & 1453 & 4.18 & 5.40 & 239 & 17.64 & 18.40 & 57 \\[0.1em]
Runway markings   & 0.71 & 0.79 & 1404 & 4.24 & 5.71 & 236 & 18.06 & 19.08 & 55 \\[0.1em]
\bottomrule
\end{tabular}%
}%
\vspace{0.3em}
\begin{minipage}{\linewidth}
	\scriptsize
Measurements use an NVIDIA RTX 2080 Ti, sample size $=1000$,
and $\PartitionCt=27$.
\end{minipage}
\vspace{-15pt}
\end{table}

\section{Conclusion}%
\label{sec:conclusion}

This paper introduced a framework for certified vision-based pose estimation. GGMs encode camera and target geometry through analytically derived rather than learned parameters, which makes image formation amenable to reachability analysis. A decoder-encoder design with grid-based sampling yields deterministic bounds on the pose estimation error, and a multi-stage pipeline extends these bounds to cluttered environments through spatial filtering and certified object detection. On synthetic images, all estimation errors remained within the certified bounds. The certified estimator returned an estimate at every tested pose, with a lower maximum rotational error than every classical estimator evaluated.
Batched encoder inference over all partitions runs at over
$230\,\si{\hertz}$ on a desktop GPU.

{\Revised
The guarantee of~\thmref{thm:estimation_error_bounds} is a
certified pose set rather than a covariance.
This set can be supplied as a bounded measurement to a temporal
estimator: it can be intersected with a motion-model prediction
in set-membership filtering or used as a consistency constraint
in an extended Kalman or particle filter.
Because the estimate is absolute with respect to the target
coordinate frame, observations of the known target can also
re-anchor a SLAM trajectory whenever the target is visible.
The present framework certifies each image independently.
Analyzing the evolution of uncertainty through a particular
temporal estimator, including multi-frame tightening and drift
between target observations, is left for future work.
}

\vspace{-0.5em}%
\label{page:before_biblio}%
\renewcommand{\IEEEbibitemsep}{-0.5pt}
\bibliographystyle{IEEEtran} 
\bibliography{20_CerGe_Biblio.bib}%

\end{document}